\documentclass[11pt]{article}
 
\usepackage{fullpage}
\usepackage[round]{natbib}
\usepackage{authblk}
\usepackage{mymath}

\title{Risk Comparisons in Linear Regression:\\ Implicit Regularization Dominates Explicit Regularization}

\author[1]{Jingfeng Wu}
\author[1,3]{Peter L. Bartlett\thanks{Alphabetical order.}}
\author[2,3]{Sham M. Kakade$^*$}
\author[1]{Jason D. Lee$^*$}
\author[1]{Bin Yu$^*$}
\affil[1]{University of California, Berkeley\\ {\tt \{uuujf,peter,jasondlee,binyu\}@berkeley.edu}}
\affil[2]{Harvard University\\ {\tt sham@seas.harvard.edu}}
\affil[3]{Google DeepMind}

\date{\today}

\begin{document}
\maketitle

\begin{abstract}
Existing theory suggests that for linear regression problems categorized by capacity and source conditions, \emph{gradient descent} (GD) is always minimax optimal, while both \emph{ridge regression} and online \emph{stochastic gradient descent} (SGD) are polynomially suboptimal for certain categories of such problems. Moving beyond minimax theory, this work provides \emph{instance-wise} comparisons of the finite-sample risks for these algorithms on any well-specified linear regression problem.

Our analysis yields three key findings. First, GD \emph{dominates} ridge regression: with comparable regularization, the excess risk of GD is \emph{always} within a constant factor of that of ridge, but ridge can be \emph{polynomially} worse even when tuned optimally. Second, GD is \emph{incomparable} with SGD. While it is known that for certain problems GD can be polynomially better than SGD, the reverse is also true: we construct problems, inspired by \emph{benign overfitting} theory, where optimally stopped GD is polynomially worse. Finally, GD dominates SGD for a significant subclass of problems---those with fast and continuously decaying covariance spectra---which includes all problems satisfying the standard capacity condition.\footnote{Accepted for presentation at the Conference on Learning Theory (COLT) 2026.}
\end{abstract}

\section{Introduction}
Modern machine learning models trained by gradient-based methods exhibit excellent generalization. 
This occurs even in the absence of explicit regularization and in the overparameterized regime where the number of parameters exceeds the number of samples. 
\emph{Gradient descent} (GD), the backbone of optimization methods in machine learning, is believed to provide \emph{implicit regularization} that prevents the trained model from overfitting \citep[see, e.g.,][]{bartlett2021deep}. 

From a geometric perspective, the implicit regularization of GD is tightly connected to an explicit norm regularization.
For instance, for overparameterized linear regression, GD converges to the empirical risk minimizer with \emph{minimum $\ell_2$-norm}. This connection, although elementary, enables the surprising phenomenon of \emph{benign overfitting} \citep{bartlett2020benign}.
Moreover, for logistic regression with linearly separable data, GD converges in direction to the \emph{maximum $\ell_2$-margin parameter vector} \citep{soudry2018implicit,ji2018risk}.
As a final example, for \emph{every} convex smooth problem, the GD path and the $\ell_2$-regularization path differ in norm by a multiplicative factor within $0.585$ and $3.415$, and in direction by an angle no more than $\pi/4$ \citep{wu2025benefits}. 

In this work, we seek to understand the implicit regularization of GD from a statistical perspective.
Using well-specified linear regression as a theoretical test bed, we compare the \emph{instance-wise} finite-sample excess risks of GD, \emph{ridge regression}, and online \emph{stochastic gradient descent} (SGD). 
Here, ridge regression incorporates an explicit $\ell_2$-regularization, 
while GD and SGD achieve regularization implicitly through \emph{early stopping} \citep{buhlmann2003boosting,yao2007early} and \emph{stochastic averaging} \citep{polyak1992acceleration}, respectively.
We make the following contributions (see \Cref{tab:comparison,tab:minimax} for an overview).

\begin{table}[t]
    \setcellgapes{2pt}
    \makegapedcells
    \begin{minipage}{.5\linewidth}
      \caption{Instance-wise risk comparisons.}
      \label{tab:comparison}
      \centering
        \begin{tabular}{c|c}
\toprule
\makecell{\makecell{all well-specified\\ linear regression\\ problems ($\Lbb$)}}   & \makecell{... with fast, continuously\\ decaying spectra\\ ($\Sbb$, subset of $\Lbb$)}  \\
\midrule
 \makecell{GD $\prec_{\Lbb}$ ridge\\ GD $\neq_{\Lbb}$ SGD}  %
   & {GD $\prec_{\Sbb}$ SGD} \\ 
\bottomrule
\end{tabular}
    \end{minipage}%
\hfill
    \begin{minipage}{.5\linewidth}
      \centering
        \caption{Optimality for $(a,r)$-power law class.}
        \label{tab:minimax}
        \begin{tabular}{c|c}
\toprule
algorithm & minimax optimal regime \\
\midrule
 ridge &  $0\le r\le 1$ \\
\midrule
 SGD &  $0< (a-1)/(2a) \le r$ \\
 \midrule
 GD  &  $ 0\le r$  \\
\bottomrule
\end{tabular}
\end{minipage} 
\vspace{1ex}

{\raggedright
\small
\textbf{Table summaries.} 
Table~\ref{tab:comparison} summarizes our main contributions on instance-wise comparisons. Table~\ref{tab:minimax} summarizes prior results on minimax optimality for the power-law problem class. While these prior results are scattered across the literature, we collect them here and show they are corollaries of recent, tighter analyses (see \Cref{sec:power-law} for details).

\textbf{Instance-wise risk comparisons (Table~\ref{tab:comparison}):}
For a problem class $\Pbb$, we say algorithm $\Acal$ \emph{dominates} $\Bcal$, or $\Acal \prec_{\Pbb} \Bcal$, if its risk is never more than a constant factor larger (though sometimes polynomially smaller) than $\Bcal$'s on every problem in the class $\Pbb$. They are \emph{incomparable}, or $\Acal \neq_{\Pbb} \Bcal$, if neither dominates the other on $\Pbb$.
Our results show that for all well-specified linear regression problems ($\Lbb$, see \Cref{eq:well-specified-linear-regression}), GD dominates ridge but is incomparable with SGD. For the subset of problems with fast, continuously decaying spectra ($\Sbb$, see \Cref{eq:continuous-spectrum}), GD dominates SGD. 
Note that Table 2 implies SGD $\neq_{\Lbb}$ ridge; see \citep{zou2021benefits} (or \Cref{sec:sgd-vs-ridge}) for a class where SGD dominates ridge.

\textbf{Minimax optimality (Table~\ref{tab:minimax}):}
For the $(a,r)$-power law class ($\Pbb_{a,r}$ in \Cref{eq:power-law}, and its restricted variant $\Pbb'_{a,r}$ defined in \Cref{eq:power-law-strict} satisfies $\Pbb_{a,r}'\subset \Sbb\subset \Lbb$), GD is minimax optimal for all source conditions $r \ge 0$. In contrast, ridge regression and SGD are only optimal in limited regimes ($0 \le r \le 1$ and $r \ge (a-1)/(2a)$, respectively) and can be polynomially suboptimal otherwise.
\par}

\end{table}

\paragraph{Contributions.}
We first show that \textbf{GD dominates ridge regression} in the following sense.
For \emph{every} well-specified linear regression problem, the excess risk achieved by GD is no more than a \emph{constant} times that of ridge regression, when the stopping time for GD is set inversely proportional to the ridge regularization. 
However, for a natural subset of these problems, the excess risk of GD is \emph{polynomially} smaller than that of ridge regression in their dependence on sample size, even when the ridge regularization is tuned optimally.
The one-sided dominance demonstrates that implicit regularization is surprisingly effective.
We obtain this result by proving a new ridge-type upper bound for GD, and comparing it with an existing lower bound for ridge regression \citep{tsigler2023benign}.

Second, we show that \textbf{GD and SGD are incomparable}. 
While it is known that SGD can be polynomially worse than GD \citep{pillaud2018statistical}, to our surprise, the reverse is also true. We show this by constructing a sequence of well-specified linear regression problems, for which the excess risk of GD is polynomially worse than that of SGD. 
Our construction leverages a key insight from the theory of \emph{benign overfitting} \citep{bartlett2020benign}, revealing an unexpected separation between batch and online learning.
Additionally, we derive a novel lower bound for GD that might be of broader interest.

Third, we show that \textbf{GD dominates SGD in a significant subset} of well-specified linear regression problems, whose covariance spectra decay \emph{fast} and \emph{continuously}. This additional condition does not constrain the true parameter and is satisfied by all \emph{power-law} spectra (also known as the capacity condition). Similarly to before, for problems in this subset, the excess risk of GD is always no worse than that of SGD by a constant factor, but can be polynomially better. 
We establish this by deriving a new SGD-type upper bound for GD, and comparing that with a known lower bound for SGD \citep{wu2022power}.

Our results complement the classical, worst-case analysis for learning classes of linear regression problems categorized by \emph{capacity} and \emph{source} conditions \citep{caponnetto2007optimal}. For problem classes of this kind, GD is known to be \emph{always} minimax optimal; in contrast, both ridge regression and SGD are known to be \emph{polynomially} suboptimal for certain classes (see \Cref{tab:minimax,tab:power-law}). 
These results were scattered in the literature (see \Cref{sec:power-law} for a detailed review), but are now simple consequences of recent tight bounds for ridge regression \citep{bartlett2020benign,tsigler2023benign} and SGD \citep{zou2023benign,wu2022last,wu2022power}, and the two novel upper bounds for GD provided in this work.

\paragraph{Notation.}
For two positive-valued functions
$f$ and $g$, we write $f\lesssim g$ or $f\gtrsim g$
if there exists $c>0$ such that for every $x$, $f(x) \le cg(x)$ or $f(x) \ge cg(x)$, respectively.
We write $f\eqsim g$ if $f\lesssim g \lesssim f$.
We use the standard big-O notation, with $\bigOT$ and $\widetilde\Omega$ to hide polylogarithmic factors within the $\Ocal$ and $\Omega$ notation, respectively. 
For two vectors $\uB$ and $\vB$ in a Hilbert space, we denote their inner product by $\la\uB, \vB\ra$ or, equivalently, $\uB^\top \vB$, and the vector norm by $\|\vB\|:=\sqrt{\vB^\top \vB}$.
For a \emph{positive semi-definite} (PSD) matrix $\MB$, we write $\|\MB\|$ as its operator norm, i.e., its largest eigenvalue.
For a vector $\vB$ and a PSD matrix $\MB$ of appropriate shape, we write \(\|\vB\|^2_{\MB} := \vB^\top \MB \vB.\)
For two matrices $\AB$ and $\BB$ of the same shape, we write $\la \AB, \BB\ra = \tr(\AB^\top \BB)$.

\section{Preliminaries}
\paragraph{Linear regression.}
Let $\Hbb$ be a separable Hilbert space. Its dimension, denoted by $d$, is either finite ($d<\infty$) or countably infinite ($d=\infty$).
Let $\xB \in \Hbb$ and $y \in \Rbb$ be a pair of covariates and response, associated with a population probability measure $\mu(\xB, y)$.
In linear regression, we seek to minimize the population risk,
\[
\risk(\wB) := \Ebb(\xB^\top \wB -y)^2,\quad \wB\in \Hbb,
\]
where the expectation is over $\mu(\xB, y)$.
Denote the second moment of the covariates as
\[
\SigmaB := \Ebb [\xB\xB^\top] \in \Hbb^{\otimes 2}.
\]
Throughout the paper, assume $\tr(\SigmaB) < \infty$.
Denote the optimal parameter as 
\[
\wB^* \in \arg\min \risk(\cdot).
\]
If the optimal parameter is not unique, let $\wB^*$ be the one with minimum $\ell_2$-norm.
Then the \emph{excess risk} is
\begin{equation*}%
    \excessRisk(\wB):= \risk(\wB) - \risk(\wB^*) = \|\wB- \wB^*\|^2_{\SigmaB},\quad \wB \in \Hbb.
\end{equation*}
For convenience, we may refer to a linear regression problem by its population probability measure $\mu(\xB, y)$. 
When necessary, we will write $\excessRisk_{\mu}$ to emphasize that the excess risk is measured with respect to a specific measure $\mu$.

\paragraph{Algorithms.}
Let $n\ge 1$ be the sample size. 
Let $(\xB_i, y_i)_{i=1}^n$ be $n$ independent copies of $(\xB, y)$.
We also write 
\[
\XB:= \begin{bmatrix}
    \xB_1^\top \\
    \vdots \\ 
    \xB_n^\top 
\end{bmatrix} \in %
\Hbb^n, \quad 
\yB := \begin{bmatrix}
    y_1 \\ 
    \vdots \\
    y_n
\end{bmatrix}\in\Rbb^n .
\]
We consider the following three estimators constructed from $(\xB_i, y_i)_{i=1}^n$.
\begin{itemize}[leftmargin=*]
    \item \textbf{Ridge regression} produces the $\ell_2$-regularized empirical risk minimizer,
    \begin{equation}\label{eq:ridge}\tag{ridge}
\begin{aligned}
    \hat\wB &:= \arg\min_{\wB} \frac{1}{n}\sum_{i=1}^n\|\xB_i^\top \wB - y_i\|^2 + \lambda \|\wB\|^2  \\
    &= \big(\XB^\top \XB+ n\lambda\IB\big)^{-1}\XB^\top \yB,
\end{aligned}
    \end{equation}
where $\lambda \ge 0$ is a hyperparameter.
\item \textbf{Gradient descent} outputs $\hat \wB:= \wB_t$ according to the following recursive update,
\begin{equation}\label{eq:gd}\tag{GD}
 \wB_0=0, \quad  \wB_{s} = \wB_{s-1}- \frac{\eta}{n}\XB^\top (\XB\wB_{s-1} - \yB),\quad s=1,\dots,t,
\end{equation}
where $\eta>0$ is a fixed stepsize and $t$, the \emph{stopping time}, is considered to be the main hyperparameter. 
As our results mainly concern the statistical properties of GD, one can consider \emph{gradient flow} (by taking $\eta \to 0_+$ and rescaling the stopping time accordingly) to simplify the mental picture.
\item \textbf{Stochastic gradient descent} outputs $\hat\wB := \wB_n$ generated by the following update,
\begin{equation}\label{eq:sgd}\tag{SGD}
 \wB_{s} =
 \begin{cases}
      \wB_{s-1} - \eta_s \xB_s (\xB_s^\top \wB_{s-1} - y_s) &  1\le s\le n, \\
 0 & s=0,
 \end{cases}\quad 
\text{where} \ \eta_s = \frac{\eta}{2^\ell} \ \text{for} \  \ell= \bigg\lfloor\frac{s}{n / \log n} \bigg\rfloor,
\end{equation}
where the initial stepsize $\eta>0$ is a hyperparameter.
We focus on the last iterate of SGD with an exponentially decaying stepsize scheduler $(\eta_s)_{s=1}^{n}$.
This variant of SGD is closer to practice and is known to behave nearly optimally in various settings \citep{ge2019step,wu2022last}. However, our discussions apply to other SGD variants such as the average of the tail iterates of SGD with a constant stepsize \citep{zou2023benign}.
\end{itemize}

Our results rely on the following set of assumptions.
\begin{assumption}[conditions for upper bounds]\label{assum:upper-bound}
For deriving the upper bounds, we assume:
\begin{assumpenum}
  \item \label{assum:upper-bound:item:x} the entries of $\SigmaB^{-\frac{1}{2}}\xB$ are independent and $\sigma_x^2$-subgaussian;
  \item \label{assum:upper-bound:item:noise} the conditional noise variance is bounded from above,
  \[\Ebb[(\yB-\xB^\top\wB^*)^2\, | \, \xB ] \le \sigma^2. \]
\end{assumpenum}
\end{assumption}

\Cref{assum:upper-bound:item:x} is widely used in the literature for benign overfitting \citep{bartlett2020benign}. 
The requirement of the independence of the entries in \Cref{assum:upper-bound:item:x} is somewhat restrictive.
Note that it can be relaxed to some extent for ridge regression \citep{tsigler2023benign}, and can be replaced by weak moment conditions for SGD \citep{zou2023benign,wu2022last}.
We adopt \Cref{assum:upper-bound:item:x} as a clean, sufficient condition that enables all needed prior results for ridge regression and SGD.
We leave it as future work to extend our comparison results to more general cases.

\begin{assumption}[conditions for lower bounds]
For deriving the lower bounds, we assume:
\begin{assumpenum}\label{assum:lower-bound}
    \item \label{assum:lower-bound:item:x} the entries of $\SigmaB^{-\frac{1}{2}}\xB$ are independent and $\sigma_x^2$-subgaussian;
    \item \label{assum:lower-bound:item:noise} the conditional noise is zero mean, and its variance is bounded from below,
    \[
    \Ebb[\yB\, | \, \xB] = \xB^\top \wB^*,\quad 
    \Ebb[(\yB-\xB^\top\wB^*)^2\, | \, \xB ] \ge \sigma^2;
    \]
    \item \label{assum:lower-bound:item:x-symmetry} the distribution of each component of $\SigmaB^{-1/2}\xB$ is symmetric, i.e., $\la\eB_i, \SigmaB^{-1/2}\xB\ra \stackrel{P}{\sim} -\la \eB_i, \SigmaB^{-1/2}\xB\ra$ for all $i$.
\end{assumpenum}
\end{assumption}

\Cref{assum:lower-bound:item:noise} requires the noise to be \emph{well-specified}, which enables a tight bias and variance decomposition.
\Cref{assum:lower-bound:item:x-symmetry} is a technical condition, which is designed specifically to cancel the off-diagonal terms in the bias error analysis. 
This condition can be replaced by the following Bayesian variant \citep{tsigler2023benign}.

\begin{taggedassumption}{{\ref{assum:lower-bound:item:x-symmetry}}'}[Bayesian symmetry]
    \label{assum:lower-bound:item:w-symmetry} 
    Assume the optimal parameter $\wB^*$ admits a prior distribution such that 
    \[\Ebb [ \uB_i^\top \wB^* \wB^{*\top} \uB_j]  = \begin{dcases}
        \big(\uB_i^\top \bar \wB^*\big)^2 & i=j, \\ 
        0 & i\ne j,
    \end{dcases}\] 
    where $\bar\wB^*\in\Hbb$ is a fixed vector, $(\uB_i)_{i\ge 1}$ are the eigenvectors of $\SigmaB$, and the expectation is over the prior of $\wB^*$.
\end{taggedassumption}

It is clear that Gaussian linear regression problems satisfy \Cref{assum:upper-bound,assum:lower-bound} with $\sigma^2_x=1$:
\begin{equation*}
\xB\sim\Ncal(0, \SigmaB),\quad 
\yB|\xB \sim \Ncal(\xB^\top\wB^*, \sigma^2),
\end{equation*}
where $(\SigmaB, \wB^*, \sigma^2)$ specifies each Gaussian linear regression problem.

In the remainder of this section, we review existing tight risk bounds for ridge regression \citep{tsigler2023benign} and SGD \citep{wu2022power}, which are pivotal to our risk comparison results.
To present these bounds, the following notation is handy.

\paragraph{Additional notation.}
Let the eigendecomposition of the covariance $\SigmaB\in\Hbb^{\otimes 2}$ be 
\[
\SigmaB = \sum_{i\ge 1} \lambda_i\uB_i\uB_i^\top,\quad \lambda_1\ge \lambda_2\ge \dots,
\]
where $(\lambda_i, \uB_i)_{i\ge 1}$ are the eigenvalues, in non-increasing order, and their corresponding eigenvectors. 
For an index $k$, allowed to be zero or infinity, we define 
\begin{equation*}
    \SigmaB_{0:k} := \sum_{i\le k} \lambda_i \uB_i\uB_i^\top,\quad 
    \SigmaB_{k:\infty} := \sum_{i>k}\lambda_i\uB_i\uB_i^\top,
\end{equation*}
both of which are PSD matrices in $\Hbb^{\otimes 2}$.
For a PSD matrix $\MB\in \Hbb^{\otimes 2}$, we define $\MB^{-1}$ as its pseudoinverse.

\subsection{Risk bounds for ridge regression}
The following proposition summarizes existing tight bounds for ridge regression.

\begin{proposition}[ridge bounds]\label[proposition]{thm:ridge}
Let $\hat \wB$ be given by \Cref{eq:ridge} with $\lambda\ge 0$. 
For every $\sigma_x^2$ there exist $c_0, c_1, c_2, c_3\ge 1$ for which the following holds.
Define 
\[
k^*:= \min\bigg\{k: \lambda+ \frac{\sum_{i>k}\lambda_i}{n}\ge c_2 \lambda_{k+1}\bigg\},\quad 
\tilde\lambda := \lambda+ \frac{\sum_{i>k^*}\lambda_i}{n},\quad 
D:= k^* + \frac{1}{\tilde\lambda^2}\sum_{i>k^*}\lambda_i^2.
\]
\begin{itemize}[leftmargin=*]
\item If \Cref{assum:upper-bound} holds and $k^*\le n/c_3$, then with probability at least $1-\exp(-n/c_0)$,
    \begin{align*}
    \Ebb\big[ \excessRisk(\hat \wB) \, | \, \XB \big] \le c_1 \bigg( \tilde\lambda^2 \|\wB^*\|^2_{\SigmaB^{-1}_{0:k^*}} + \|\wB^*\|^2_{\SigmaB_{k^*:\infty}} + \sigma^2\frac{D}{n} \bigg).
\end{align*}
\item If
Assumptions \ref{assum:lower-bound:item:x}, \ref{assum:lower-bound:item:noise}, and \ref{assum:lower-bound:item:w-symmetry} 
hold, then with probability at least $1-\exp(-n/c_0)$, 
\begin{align*}
   \Ebb \big[ \excessRisk(\hat \wB)  \, | \, \XB\big] \ge 
   \frac{1}{c_1}\bigg( \tilde\lambda^2 \|\bar \wB^*\|^2_{\SigmaB^{-1}_{0:k^*}} + \|\bar \wB^*\|^2_{\SigmaB_{k^*:\infty}} + \sigma^2\min\bigg\{\frac{D}{n},\, 1\bigg\} \bigg).
\end{align*}
\item Moreover, under \Cref{assum:lower-bound:item:x,assum:lower-bound:item:noise,assum:lower-bound:item:x-symmetry}, the above lower bound holds in expectation, replacing $\bar\wB^*$ by $\wB^*$.
\end{itemize}
\end{proposition}

In \Cref{thm:ridge}, the upper bound and the high probability lower bound are due to \citet{tsigler2023benign}, in which the variance error bounds (the terms involving $\sigma^2$) are ultimately due to \citet{bartlett2020benign}.
The expectation lower bound is by \citet[Theorem B.2]{zou2021benefits}, adapted from the lower bound by \citet{tsigler2023benign}.
Note that the expectation lower bound uses \Cref{assum:lower-bound:item:x-symmetry} instead of \Cref{assum:lower-bound:item:w-symmetry}, avoiding the Bayesian perspective for the high probability lower bound.

\subsection{Risk bounds for SGD}
The next proposition summarizes the existing tight bounds for SGD.

\begin{proposition}[SGD bounds]\label[proposition]{thm:sgd}
Let $\hat \wB := \wB_n$ be given by \Cref{eq:sgd} with sample size $n\ge 100$ and initial stepsize $\eta\le 1/(4 \tr(\SigmaB))$. 
For every $\sigma_x^2$ and $c>0$, there exists $c_1\ge 1$ for which the following holds.
Define 
\[
N:= \frac{n}{\log n}, \quad 
k^*:= \min\bigg\{k: \frac{1}{\eta N} \ge c\lambda_{k+1}\bigg\},\quad 
D:= k^* + (\eta N)^2\sum_{i>k^*}\lambda_i^2.
\]
\begin{itemize}[leftmargin=*]
\item Under \Cref{assum:upper-bound}, it holds that
\begin{align*}
    \Ebb \excessRisk(\hat\wB)  \le c_1 \Bigg( \bigg\|\prod_{t=1}^{n}\big(\IB-\eta_t\SigmaB\big)\wB^*\bigg\|^2_{\SigmaB} + \big(\sigma^2+\|\wB^*\|^2_{\SigmaB}\big)\frac{D}{N} \Bigg).
\end{align*}
\item Under \Cref{assum:lower-bound:item:noise}, it holds that 
\begin{align*}
    \Ebb \excessRisk(\hat\wB)  \ge \frac{1}{c_1} \Bigg( \bigg\|\prod_{t=1}^{n}\big(\IB-\eta_t\SigmaB\big)\wB^*\bigg\|^2_{\SigmaB} + \sigma^2\frac{D}{N} \Bigg).
\end{align*}
\end{itemize}
\end{proposition}

\Cref{thm:sgd} is a consequence of \citep[Corollary 3.4]{wu2022power}.
Specifically, our \Cref{assum:upper-bound:item:x} implies their Assumption 1A with $\alpha=16\sigma_x^4$ \citep[Lemma A.1]{zou2021benefits} and our \Cref{assum:upper-bound:item:noise} implies their Assumption 2. Then the upper bound in \Cref{thm:sgd} follows from \citep[upper bound in Corollary 3.4]{wu2022power}.
Although \citet{wu2022power} only stated their lower bound in Corollary 3.4 for Gaussian noise (see their Assumption 2') and only for $c = 1$, it is clear from their proof that their lower bound applies to any well-specified noise characterized by \Cref{assum:lower-bound:item:noise} and any positive constant $c>0$ when allowing $c_1$ to depend on $c$ \citep[see][proof of Theorem D.1]{wu2022last}.
Moreover, their Assumption 1B always holds with $\beta= 0$. 
Then the lower bound in \Cref{thm:sgd} follows from \citet[lower bound in Corollary 3.4]{wu2022power} by setting $\beta=0$.

We remark that Corollary 3.4 in \citet{wu2022power} is a refinement of results in \citet{wu2022last}, which ultimately build upon the operator methods developed by \citet{zou2023benign} and earlier literature \citep[see references in][]{zou2023benign}.
We point out that the bounds in \Cref{thm:sgd} hold under weaker moment conditions and can be made slightly tighter. 
However, the stated \Cref{thm:sgd} is sufficient for our risk comparison purpose. 

It is worth noting that the upper and lower bounds in \Cref{thm:sgd} match up to constant factors provided that the signal-to-noise ratio is bounded from above, i.e., $\|\wB^*\|^2_{\SigmaB}\lesssim \sigma^2$.

\subsection{SGD versus ridge regression }\label{sec:sgd-vs-ridge}
The upper bound for ridge regression in \Cref{thm:ridge} holds with high probability, while a variant of its lower bound (under \Cref{assum:lower-bound:item:x-symmetry}) in \Cref{thm:ridge} and the SGD bounds in \Cref{thm:sgd} both hold in expectation.
Ignoring the obvious gap between high probability and expectation, these bounds are highly comparable when matching $\tilde\lambda$ for ridge regression with $1/(\eta N)$ for SGD.
Specifically, both $\tilde\lambda$ and $1/(\eta N)$ play the role of \emph{effective regularization}, controlling the bias-variance tradeoff.
When these quantities increase (via increasing $\lambda$ or decreasing $\eta$, respectively), the variance error (the terms involving $\sigma^2$) tends to decrease while the bias error (the terms involving $\wB^*$) tends to increase, and vice versa. 

For a more detailed comparison, two observations are crucial. First, the head component of the bias error of SGD decays at an \emph{exponential} rate, while that of ridge regression only decays at a \emph{polynomial} rate. 
Thus, SGD can be more effective in reducing the bias error for well-conditioned problems. 
Second, the effective regularization for SGD ($1/(\eta N)$) and ridge regression ($\tilde\lambda$) are both subject to certain constraints.
For SGD, because of the \emph{limited optimization power} from the one-pass nature, its effective regularization is limited by $1/(\eta N) = \Omega(\log(n) / n)$.
For ridge regression with $\lambda\ge 0$, its effective regularization $\tilde\lambda = \lambda +  \sum_{i>k^*}\lambda_i / n$ is limited from below if the covariance spectrum decays slowly, but can be arbitrary otherwise.
Therefore, both SGD and ridge regression have their own advantages and limitations. Indeed, we will provide natural examples in \Cref{sec:power-law} for which ridge regression is polynomially better than SGD, and vice versa (see \Cref{tab:power-law}). 

Previously, \citet{zou2021benefits} showed that for a class of \emph{well-conditioned} Gaussian linear regression problems, the sample complexity of SGD is at most polylogarithmically worse, but can be polynomially better, than that of ridge regression. 
When restricted to only well-conditioned problems, SGD is no longer limited by its optimization power. Thus \citet{zou2021benefits} are able to show that SGD nearly dominates ridge regression.
As mentioned earlier, when considering all well-specified linear regression problems, SGD and ridge regression are incomparable (see \Cref{tab:power-law}). 

In the next section, we remove the well-conditioning assumption required by \citet{zou2021benefits} and show that for \emph{all} well-specified linear regression problems, the excess risk of GD is at most a constant times worse, but can be polynomially better, than that of ridge regression.
One reason this is possible is that GD has unlimited optimization power, unlike SGD.

\section{GD dominates ridge regression}\label{sec:gd-vs-ridge}
In this section, we present our first main result,
a novel upper bound for GD, and compare it with \Cref{thm:ridge} to show that GD is no worse than ridge regression. 
Later in \Cref{sec:power-law}, we consider a class of well-specified linear regression problems categorized by capacity and source conditions (see $\Pbb_{a,r}$ defined in \Cref{eq:power-law}), and show that for a certain class of this kind, GD achieves a polynomially better excess risk than ridge regression (\Cref{thm:power-law:ridge,thm:power-law:gd}; see also \Cref{tab:power-law}).
These results together show that GD dominates ridge regression.

\paragraph{A ridge-type bound for GD.}
The following theorem provides a new upper bound for GD comparable to \Cref{thm:ridge}.
Its proof is deferred to \Cref{append:sec:gd:ridge}.

\begin{theorem}[an upper bound for GD]\label{thm:gd:ridge}
Let $\hat \wB := \wB_t$ be given by \Cref{eq:gd} with stopping time $t\ge 0$ and stepsize $\eta \le n/\|\XB\XB^\top\|$.
Suppose that \Cref{assum:upper-bound} holds, then for every $\sigma_x^2$ there exist $c_0, c_1, c_2, c_3 \ge 1$ for which the following holds.
Define 
\begin{align*}
k^*:= \min\bigg\{k: \frac{1}{\eta t} + \frac{ \sum_{i>k}\lambda_i}{n}\ge c_2 \lambda_{k+1}\bigg\},\quad 
\tilde\lambda := \frac{1}{\eta t} + \frac{\sum_{i>k^*}\lambda_i}{n},\quad 
   D:= k^* + \frac{1}{\tilde\lambda^2}\sum_{i>k^*}\lambda_i^2.
\end{align*}
If $k^* \le n/c_3$, 
then with probability at least $1-\exp(-n/c_0)$, 
\begin{align*}
    \Ebb[ \excessRisk(\hat\wB) \, | \, \XB] \le c_1 \bigg(\tilde\lambda^2 \|\wB^*\|^2_{\SigmaB^{-1}_{0:k^*}} + \|\wB^*\|^2_{\SigmaB_{k^*:\infty}} + \sigma^2 \frac{D}{n} \bigg).
\end{align*}
Moreover, the quantities $c_0$, $c_2$, and $c_3$ can be made the same as those in \Cref{thm:ridge}.
\end{theorem}

The upper bound for GD in \Cref{thm:gd:ridge} is comparable to the upper and lower bounds for ridge regression in \Cref{thm:ridge}.
In particular, the critical index $k^*$, effective regularization $\tilde\lambda$, and effective dimension $D$ are the same when matching $1/(\eta t)$ in \Cref{thm:gd:ridge} with $\lambda$ in \Cref{thm:ridge}.
Then, according to \Cref{thm:gd:ridge,thm:ridge}, GD attains the same upper bound as that of ridge regression for all $(\SigmaB, \wB^*, \sigma^2)$.

However, to rigorously show that GD is no worse than ridge regression, we must compare the upper bound for GD with the lower bound for ridge regression, for which a technical issue needs to be addressed.
Note that the variance error in the lower bound for ridge regression in \Cref{thm:ridge} scales with $\min\{D/n, 1\}$, while the variance error in the upper bound for GD in \Cref{thm:gd:ridge} scales with $D/n$.
The latter can be greater than the former; however, neither ridge regression nor GD generalizes in such cases, provided that the noise variance is nontrivial. 
We make the above discussion rigorous in the following theorem.

\paragraph{GD is no worse than ridge regression.}
Let $b>0$ be a positive constant controlling the signal-to-noise ratio.
Denote the set of \emph{well-specified linear regression problems} and its \emph{Bayesian variant} as
\begin{equation}\label{eq:well-specified-linear-regression}
\begin{aligned}
\Lbb_{b} &:= \big\{\mu(\xB, y) \ \text{satisfying Assumptions \ref{assum:upper-bound}, \ref{assum:lower-bound:item:noise}, \ref{assum:lower-bound:item:x-symmetry}, and} \  \|\wB^*\|^2_{\SigmaB}\le b \sigma^2 \big\}, \\
    \Lbb'_{b} &:= \big\{\mu(\xB, y) \ \text{satisfying Assumptions \ref{assum:upper-bound}, \ref{assum:lower-bound:item:noise}, \ref{assum:lower-bound:item:w-symmetry}, and} \  \|\bar\wB^*\|^2_{\SigmaB}\le b \sigma^2 \big\},
\end{aligned}
\end{equation}
respectively. 
We show that GD is no worse than ridge for every problem in $\Lbb'_{b}$ in the next theorem.

\begin{theorem}[GD is no worse than ridge]\label{thm:gd-ridge}
Let $\hat\wB^{\ridge}_{\lambda}$ be given by \Cref{eq:ridge} with regularization $\lambda\ge 0$, and $\hat \wB^{\gd}_{t}$ be given by \Cref{eq:gd} with any fixed stepsize $\eta\le n/\|\XB\XB^\top\|$ and stopping time $t\ge 0$. 
Then for every $\mu\in \Lbb_b'$, sample size $n\ge 1$, and regularization $\lambda \ge 0$, there exists a stopping time $t\ge 0$ such that
\begin{align*}
\text{with probability at least $1-\exp(-n/c_0)$},\quad    \Ebb \big[ \excessRisk_{\mu}(\hat \wB^{\gd}_{t})  \, | \, \XB\big] \le c \cdot \Ebb \big[ \excessRisk_{\mu}(\hat \wB^{\ridge}_{\lambda}) \, | \, \XB\big],
\end{align*}
where $c_0\ge 1$ only depends on $\sigma_x^2$, and $c\ge 1$ only depends on $\sigma_x^2$ and $b$.
\end{theorem}
\begin{proof}[Proof of \Cref{thm:gd-ridge}]
Assumptions \ref{assum:upper-bound}, \ref{assum:lower-bound:item:noise}, and \ref{assum:lower-bound:item:w-symmetry} enable \Cref{thm:gd:ridge} and the lower bound in \Cref{thm:ridge}. 
In \Cref{thm:gd:ridge}, we take an additional expectation over $\wB^*$ and replace $\wB^*$ with $\bar\wB^*$ in the bound.
We consider two cases.

If $D/n > 1/c_3$ in \Cref{thm:ridge}, then the lower bound in \Cref{thm:ridge} is further lower bounded by $\sigma^2 / (c_1 c_3)$. In this case, we just set $t=0$ so $\hat\wB_t^{\gd}=0$, which incurs a Bayesian averaged excess risk of $\|\bar \wB^*\|^2_{\SigmaB} \le b \sigma^2$. So the claim holds in this case.

If $D/n \le 1/c_3$ in \Cref{thm:ridge}, then we have $k^*\le n/c_3$ and $\min\{D/n, \, 1\} = D/n$ in \Cref{thm:ridge}. 
In this case, we set $t = 1/(\eta \lambda)$ in \Cref{thm:gd:ridge}, which leads to $k^*\le n/c_3$ in \Cref{thm:gd:ridge} and enables the upper bound in \Cref{thm:gd:ridge}. Therefore, the claim also holds in this case.
\end{proof}

\Cref{thm:gd-ridge} shows that the excess risk attained by GD is no more than a constant times that of ridge regression for all well-specified linear regression problems specified by \Cref{eq:well-specified-linear-regression}. 
In \Cref{thm:gd-ridge}, the Bayesian perspective, through \Cref{assum:lower-bound:item:w-symmetry} in the definition of $\Lbb'_b$, is merely technical rather than fundamental. 
The Bayesian perspective enables comparing two high probability bounds. 
If one accepts comparing the high probability upper bound in \Cref{thm:gd:ridge} with the expectation lower bound in \Cref{thm:ridge}, then \Cref{assum:lower-bound:item:w-symmetry} can be replaced by \Cref{assum:lower-bound:item:x-symmetry}, and hence the claim in \Cref{thm:gd-ridge} holds for $\Lbb_b$.

\paragraph{Comparison with \citet{ali2019continuous}.}
Previously, \citet[Theorem 2]{ali2019continuous} showed that GD achieves an excess risk no more than $1.69$ times that of ridge regression by matching $1/(\eta t)$ with $\lambda$. However, they only obtained this assuming the optimal parameter $\wB^*$ satisfies an \emph{isotropic} prior, i.e., $\Ebb \wB^{*\otimes 2} \propto \IB$ (or equivalently, \Cref{assum:lower-bound:item:w-symmetry} with $\bar\wB^*\propto\oneB$). 
Their proof is only three lines of linear algebra \citep[see][the proof of Theorem 2]{ali2019continuous}, as the isotropic prior allows commuting matrices, greatly simplifying the analysis.
But, before our work, the comparison is less clear for a general prior: \citet[Remark 8]{ali2019continuous} wrote that ``it is not clear to us whether the result is true for prediction risk in general.''

We also point out that isotropic priors are special, under which ridge regression with optimally tuned regularization is well-known to be Bayes optimal. Thus, for the set of problems considered by \citet{ali2019continuous}, GD is equivalent to ridge regression in terms of excess risk rates, when both are tuned optimally. 

In this regard, our \Cref{thm:gd-ridge} significantly generalizes the results by \citet{ali2019continuous}, showing that even when the prior is anisotropic, the excess risk of GD is still no worse than that of ridge regression by a constant multiplier (which can be greater than $1.69$).
Additionally, when considering anisotropic priors, there exist examples such that GD is polynomially better than ridge regression (see \Cref{sec:power-law} or \Cref{tab:power-law}).
The key ingredient to our improved results is the new upper bound for GD in \Cref{thm:gd:ridge}, the proof of which is much more involved and requires additional subgaussian assumptions (which are not needed by \citet{ali2019continuous}).

\section{GD versus SGD}
Having seen how GD compares to ridge regression,
it is natural to compare GD with SGD, which we do in this section. 
The results turn out to be rather intriguing: when considering all well-specified linear regression problems, GD and SGD are incomparable; however, for a subset of those problems with fast, continuously decaying spectra, GD does dominate SGD. 

\subsection{GD is incomparable with SGD}\label{sec:gd-vs-sgd:incomparable}
We first show that GD is incomparable with SGD when considering all well-specified linear regression problems.
Specifically, we show that there exist problems of this kind for which GD is polynomially worse than SGD, and vice versa.

\paragraph{A lower bound for GD.} We provide a lower bound for GD in the following theorem, the proof of which is deferred to \Cref{append:sec:gd:lower-bound}.

\begin{theorem}[a lower bound for GD]\label{thm:gd:lower-bound}
Let $\hat \wB:=\wB_t$ be given by \Cref{eq:gd} with stepsize $\eta$ and stopping time $t$. 
Suppose that \Cref{assum:lower-bound} holds. 
Then for every $\sigma_x^2$ there exist $c_0, c_1, c_2 \ge 1$ for which the following holds.
Define $D$ as in \Cref{thm:gd:ridge} and
\begin{align*}
    \ell^* := \min\bigg\{k: \frac{\sum_{i>k}\lambda_i}{n} \ge c_2 \lambda_{k+1}\bigg\}.
\end{align*}
Then for every $n\ge c_0$, $0<\eta \le n/\|\XB\XB^\top\|$, and $t\ge 0$, we have
\begin{align*}
    \Ebb \excessRisk(\hat\wB) 
    \ge \frac{1}{c_1}\Bigg( \bigg(\frac{\sum_{i>\ell^*}\lambda_i}{n}\bigg)^2 \|\wB^*\|^2_{\SigmaB_{0:\ell^*}^{-1}} + \|\wB^*\|^2_{\SigmaB_{\ell^*:\infty}}
    + \sigma^2 \min\bigg\{ \frac{D}{n},\, 1\bigg\} \Bigg).
\end{align*}
\end{theorem}

In \Cref{thm:gd:lower-bound}, the critical index $\ell^*$ plays a key role in the theory of \emph{benign overfitting} \citep{bartlett2020benign}. 
Provided with the results in \Cref{thm:ridge}, we interpret \Cref{thm:gd:lower-bound} as follows: the excess risk of GD is at least the sum of the variance error of ridge regression, with $\lambda= 1/(\eta t)$, and the bias error of \emph{ordinary least squares} (OLS, i.e., ridge regression with $\lambda\to 0_+$).

When the spectrum decays slowly, \Cref{thm:gd:lower-bound} suggests that GD decreases the head component of the bias error at best at a polynomial rate; however, \Cref{thm:sgd} suggests that SGD can decrease that error at an exponential rate. 
This is the key intuition for showing that GD can be polynomially worse than SGD. 
Finally, we point out that the slowly decaying spectrum is also a key condition for benign overfitting \citep{bartlett2020benign}.

\paragraph{GD can be polynomially worse than SGD.}
With the above discussions, we present the next theorem showing that GD can be polynomially worse than SGD, with proof deferred to \Cref{append:sec:gd-sgd}.

\begin{theorem}[a hard example for GD]\label{thm:gd-vs-sgd:heavy-tail}
Let $n \ge 100$ be the sample size.
Consider a sequence of $d$-dimensional linear regression problems satisfying \Cref{assum:upper-bound,assum:lower-bound}
with $\sigma^2 \le 1$, $d \ge n^2$, and
    \begin{align*}
        \wB^* = \begin{bmatrix}
            n^{0.45} \\
            0 \\
            \vdots \\
            0
        \end{bmatrix}, \quad 
        \SigmaB = \begin{bmatrix}
            n^{-0.9} & & & \\
            & 1/d & & \\
            & & \ddots & \\
            & & & 1/d
        \end{bmatrix}.
    \end{align*}
Then $\|\wB^*\|_{\SigmaB}^2 \le 1$ and $\tr(\SigmaB)\le 2$.
Moreover, 
\begin{itemize}[leftmargin=*]
\item for $\hat\wB^{\gd}:=\wB_t$ given by \Cref{eq:gd} with any stepsize $\eta \le n/\|\XB\XB^\top\|$ and any stopping time $t\ge 0$, 
\begin{align*}
    \Ebb \excessRisk(\hat\wB^{\gd})  = \Omega(n^{-0.2});
\end{align*}
\item for $\hat\wB^{\sgd}:=\wB_n$ given by \Cref{eq:sgd} with initial stepsize $\eta = 1/(4\tr(\SigmaB))$, 
\begin{align*}
    \Ebb \excessRisk(\hat\wB^{\sgd}) = \Ocal\bigg(\frac{\log n}{n}\bigg).
\end{align*}
\end{itemize}
\end{theorem}

In \Cref{thm:gd-vs-sgd:heavy-tail}, we construct a sequence of well-specified linear regression problems for which GD is polynomially worse than SGD. 
Besides, \Cref{thm:gd-vs-sgd:heavy-tail} has the following interesting implications. 
Notice that \Cref{thm:gd-vs-sgd:heavy-tail} allows setting $\sigma^2 = 0$ and $t\to\infty$, that is, the same lower bound for GD applies to OLS in the noiseless case. 
In this case, \Cref{thm:gd-vs-sgd:heavy-tail} suggests that for high-dimensional \emph{noiseless} linear regression, OLS can be polynomially worse than SGD.
This is very different from the finite-dimensional noiseless cases, where OLS is optimal as it achieves a zero risk almost surely. 

\begin{remark}[suboptimality of OLS for noiseless linear regression]
The suboptimality of OLS for noiseless linear regression is a known phenomenon. For example, \citet{tsigler2023benign} showed that ridge regression with \emph{negative} regularization can be better than OLS in certain high-dimensional low-noise cases (see their paper for earlier references on this).
However, to the best of our knowledge, \Cref{thm:gd-vs-sgd:heavy-tail} is the first result showing that SGD can also be better than OLS (or GD) for noiseless linear regression. 
We leave it as future work to investigate the optimal algorithms for noiseless or low-noise linear regression.   
\end{remark}

\paragraph{GD can be polynomially better than SGD. }
As discussed in \Cref{sec:sgd-vs-ridge}, SGD and ridge regression are incomparable; that is, both of them can be polynomially better than the other for certain well-specified linear regression problems. 
Moreover, we show in \Cref{sec:gd-vs-ridge} that GD is always no worse than ridge regression. These together imply that GD can be polynomially better than SGD. 
Later in \Cref{sec:power-law}, we will provide concrete examples under capacity and source conditions for this (see \Cref{tab:power-law}). 

Summarizing our discussions in \Cref{sec:gd-vs-sgd:incomparable}, we show that GD and SGD are incomparable for all well-specified linear regression problems.
We next study when GD would dominate SGD. 

\subsection{GD dominates SGD in a subclass}
Although GD does not dominate SGD for all well-specified linear regression problems, we show in this part that GD does dominate SGD under an additional spectrum condition. 

\paragraph{An SGD-type bound for GD.} 
We have established a ridge-type upper bound for GD in \Cref{thm:gd:ridge}.
However, that bound could be loose for comparison with SGD. 
Indeed, SGD can decrease the head component of the bias error at an exponential rate (\Cref{thm:sgd}), while that error may only decrease at a polynomial rate in \Cref{thm:gd:ridge}.
Intuitively, when the stopping time is sufficiently early, GD should also be able to exponentially decrease the head component of its bias error (until hitting the barrier given by \Cref{thm:gd:lower-bound}). 
We make this intuition rigorous by deriving an SGD-type upper bound for GD with stopping time $t\lesssim n$ in the next theorem. Its proof is deferred to \Cref{append:sec:gd:exp}.

\begin{theorem}[an upper bound for GD]\label{thm:gd:exp}
Let $b>0$ be any positive constant.
Let $\hat \wB := \wB_t$ be given by \Cref{eq:gd} with stepsize $\eta \le 1/(2\tr(\SigmaB)) $ and stopping time $t\le b n$.
Under \Cref{assum:upper-bound}, there exist $c_2, c_3 \ge 1$ that only depend on $\sigma_x^2$, and $c_0, c_1\ge 1$ that only depend on $\sigma_x^2$ and $b$, for which the following holds. 
Define
\[
k^*:=\min\bigg\{k: \frac{1}{\eta t}  \ge c_2 \lambda_{k+1}\bigg\},\quad 
D:= k^* + (\eta t)^2 \sum_{i>k^*}\lambda_i^2,\quad 
D_1:= k^* + \eta t \sum_{i>k^*}\lambda_i.
\]
If $k^*\le n/c_3$, then with probability at least $1-\exp(-k^*/c_0)$,
\begin{align*}
\Ebb[\excessRisk(\wB_t)\, | \, \XB]
  \le c_1\Bigg( \underbrace{\frac{1}{\eta^2 t^2} \big\|(\IB-\eta\SigmaB)^{t/2}\wB^*\big\|^2_{\SigmaB_{0:k^*}^{-1}} + \|\wB^*\|^2_{\SigmaB_{k^*:\infty}}}_{\effBias} 
  +  \effVar
  + \underbrace{\sigma^2 \frac{D}{n} }_{\variance} \Bigg),
\end{align*}
where 
\begin{align*}
    \effVar \le  \frac{1}{\eta^2 t^2}\|\wB^*\|^2_{\SigmaB_{0:k^*}^{-1}}\bigg(\frac{D}{n} + \frac{D^2_1}{n^2} \bigg)
    \le c_2^2 \|\wB^*\|^2_{\SigmaB_{0:k^*}}\bigg(\frac{D}{n} + \frac{D^2_1}{n^2} \bigg) . 
\end{align*}
\end{theorem}

In \Cref{thm:gd:exp}, we decompose the excess risk of GD into the sum of a variance error, an effective bias error, and an effective variance error. 
Compared to \Cref{thm:sgd}, the effective bias error matches the bias error in \Cref{thm:sgd}, the variance error matches that in \Cref{thm:sgd}, but the extra effective variance error does not match any term in \Cref{thm:sgd}.
Specifically, the effective variance error scales with both the effective dimension $D$ and the \emph{order-$1$ effective dimension} $D_1$, instead of just the effective dimension $D$.
Note that the appearance of $D_1$ in the bound is not a technical artifact, but is somewhat unavoidable due to the lower bound in \Cref{thm:gd:lower-bound}.
We next analyze the order-$1$ effective dimension.

\paragraph{Order-1 effective dimension.}
Comparing $D_1$ with $D$, we always have $D_1 \ge D$. 
For certain cases, e.g., in \Cref{thm:gd-vs-sgd:heavy-tail}, we could have $D_1 \gg D$.
In the following, we provide a sufficient condition to ensure $D_1 \eqsim D$, which essentially requires the spectrum to decay \emph{fast} and \emph{continuously}.

\begin{assumption}[fast continuously decaying spectrum]\label{assum:spectrum}
Assume that for a constant $ \sigma_{\lambda} > 0$, $(\lambda_i)_{i\ge 1}$ satisfies
\[
\text{for every $0<\mu<\lambda_1$},\quad 
\frac{1}{\mu} \sum_{\lambda_i \le \mu }\lambda_i \le \sigma_{\lambda} \cdot \#\{i: \lambda_i > \mu \}.
\]
\end{assumption}

By the definitions of $D_1$ and $D$ in \Cref{thm:gd:exp}, \Cref{assum:spectrum} ensures that $D\le D_1 \le (1+\sigma_\lambda) k^* \le (1+\sigma_\lambda) D$, i.e., the effective dimension and the order-1 effective dimension are within constant factors of each other.
Next, we provide concrete examples and counterexamples for \Cref{assum:spectrum}, the proof of which is a simple calculus and is therefore omitted. 

\begin{example}\label{example:spectrum}
We have the following examples or counterexamples for \Cref{assum:spectrum}:
\begin{itemize}[leftmargin=*]
    \item the exponential spectrum, $\lambda_i \eqsim a^{-i}$ for $a>1$, satisfies \Cref{assum:spectrum} for some $\sigma_\lambda \eqsim 1$;
    \item the polynomial spectrum, $\lambda_i \eqsim i^{-a}$ for $a>1$, satisfies \Cref{assum:spectrum} for some $\sigma_\lambda \eqsim 1 / (a-1)$;
    \item the polylogarithmic spectrum, $\lambda_i \eqsim i^{-1}\log^{-a}(i)$ for $a>1$, violates \Cref{assum:spectrum};
    \item the spike spectrum in \Cref{thm:gd-vs-sgd:heavy-tail} violates \Cref{assum:spectrum}.
\end{itemize}
\end{example}

We remind the reader that the first two and the last two spectra in \Cref{example:spectrum} prevent and enable benign overfitting, respectively \citep{bartlett2020benign}. 
Indeed, the hard examples for GD to be polynomially worse than SGD in \Cref{thm:gd-vs-sgd:heavy-tail} are strongly tied to benign overfitting. 
By essentially ruling out benign overfitting, \Cref{assum:spectrum} guarantees that GD is no worse than SGD, which we show next.

\paragraph{GD dominates SGD under \Cref{assum:spectrum}.}
Denote the set of well-specified linear regression problems with fast, continuously decaying spectra by
\begin{equation}\label{eq:continuous-spectrum}
    \Sbb_{b}:=
    \big\{ \mu(\xB, y)\ \text{satisfying Assumptions \ref{assum:upper-bound}, \ref{assum:lower-bound:item:noise}, and \ref{assum:spectrum} with}\ 
   \|\wB^*\|^2_{\SigmaB}\le b \sigma^2
    \big\},
\end{equation}
where $b>0$ is a constant controlling the signal-to-noise ratio.
Then our next theorem shows that GD is no worse than SGD for problems in $\Sbb_b$.

\begin{theorem}[GD is no worse than SGD in a subset]\label{thm:gd-vs-sgd:spectrum}
Let $\hat\wB^{\sgd}_{\eta}$ be given by \Cref{eq:sgd} with initial stepsize $\eta \ge 0$, and $\hat \wB^{\gd}_{t}$ be given by \Cref{eq:gd} with any fixed stepsize less than $n/\|\XB\XB^\top\|$ and stopping time $t\ge 0$. 
Then for every $\mu\in \Sbb_b$, sample size $n\ge 100$, and SGD stepsize $\eta \le 1/(4\tr(\SigmaB))$, there exists a stopping time $t\ge 0$ such that
\begin{align*}
\text{with probability at least $1-\exp(-k^*/c_0)$},\quad    \Ebb \big[ \excessRisk_{\mu}(\hat \wB^{\gd}_{t})  \, | \, \XB\big]   \le c  \cdot \Ebb  \excessRisk_{\mu}(\hat \wB^{\sgd}_{\eta}),
\end{align*}
where $k^*$ is as defined in \Cref{thm:gd:exp}, 
$c_0\ge 1$ only depends on $\sigma_x^2$, and 
$c \ge 1$ only depends on $\sigma_x^2$, $\sigma_\lambda$, and $b$.
\end{theorem}
\begin{proof}[Proof of \Cref{thm:gd-vs-sgd:spectrum}]
\Cref{assum:lower-bound:item:noise} enables the lower bound in \Cref{thm:sgd}, \Cref{assum:upper-bound} enables \Cref{thm:gd:exp}, and \Cref{assum:spectrum} guarantees that $D\le D_1 \le (1+\sigma_{\lambda}) D$.
The claim follows from applying \Cref{thm:gd:exp} with $t = 4N$ and \Cref{thm:sgd} with $c=4c_2$ for $c_2$ in \Cref{thm:gd:exp}.
\end{proof}

Finally, recall that power-law spectra satisfy \Cref{assum:spectrum} by \Cref{example:spectrum}. Therefore, $\Sbb_b$ includes all well-specified linear regression problems under the (strict) capacity condition \citep{caponnetto2007optimal}. In the next section, we will show that GD can be polynomially better than SGD under the (strict) capacity condition. This, together with \Cref{thm:gd-vs-sgd:spectrum}, suggests that GD dominates SGD for well-specified linear regression problems with fast, continuously decaying spectra.
Finally, we point out that $\Sbb_b$ only imposes a spectrum condition on $\Lbb_b$, whereas it imposes no extra constraint on the true parameter.

\section{Exact rates under capacity and source conditions}\label{sec:power-law}
Our final set of results is to apply the bounds presented so far to 
compute the exact rates of ridge regression, SGD, and GD for learning a class of well-specified linear regression problems under capacity and source conditions \citep{caponnetto2007optimal,dieuleveut2016nonparametric}. 
The results are summarized in \Cref{tab:power-law}.

\begin{figure}[t]
\centering
\begin{minipage}[b]{.65\linewidth}
\centering
\setcellgapes{1pt}
\makegapedcells
\begin{tabular}[b]{c|c|c|c}
\toprule
 & $0\le r< \dfrac{a-1}{2a}$ & {$\dfrac{a-1}{2a}\le  r \le 1$} & $1<r$ \\
\midrule
ridge & \multicolumn{2}{c|}{ $\bigO\big(n^{-\frac{2ar}{1+2ar}}\big)$} & {$\Omega\big(n^{-\frac{2a}{1+2a}}\big)$}  \\
\midrule
 SGD & {$\widetilde\Omega\big(n^{-2r}\big)$} & \multicolumn{2}{c}{$\bigOT\big(n^{-\frac{2ar}{1+2ar}}\big)$} \\
\midrule
GD & \multicolumn{3}{c}{$\bigO\big(n^{-\frac{2ar}{1+2ar}}\big)$}  \\
\midrule
minimax & \multicolumn{3}{c}{$\Omega\big(n^{-\frac{2ar}{1+2ar}}\big)$}  \\
\bottomrule
\end{tabular}
\end{minipage}\hfill
\begin{minipage}[b]{.35\linewidth}
    \centering
    \includegraphics[width=\linewidth]{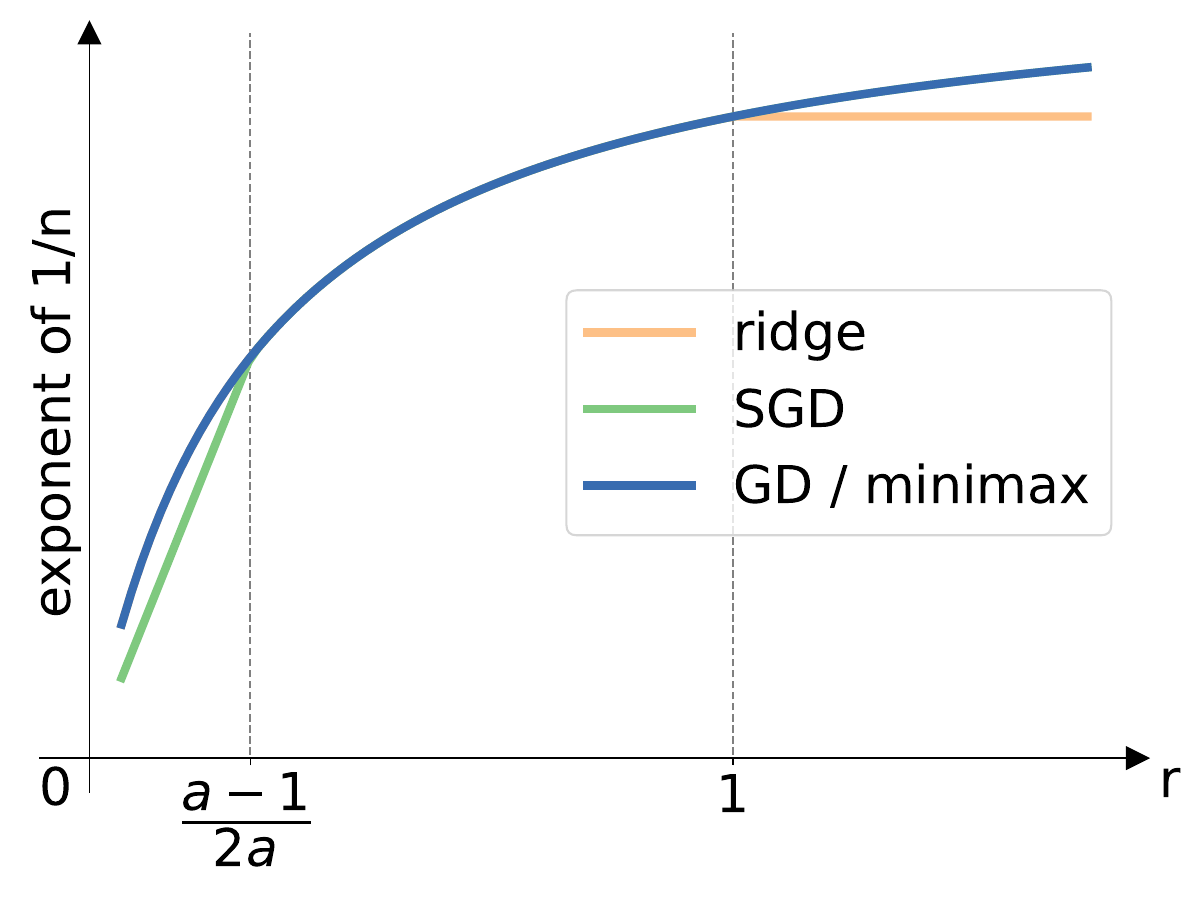}
\end{minipage}
\caption{A summary of results in \Cref{sec:power-law}.
\textbf{Left}: rates for learning $(a,r)$-power law class (see $\Pbb_{a,r}$ defined in \Cref{eq:power-law}).
\textbf{Right}: exponents of $1/n$ in the rates when fixing $a$ and varying $r$.
GD is minimax optimal for all $a>1$ and $r\ge 0$ (\Cref{thm:power-law:gd}).
Ridge regression is minimax optimal for $0\le r\le 1$ but is polynomially suboptimal otherwise (\Cref{thm:power-law:ridge}). 
SGD is (nearly) minimax optimal for $r \ge (a-1)/(2a)$ but is polynomially suboptimal otherwise (\Cref{thm:power-law:sgd}). 
However, the best of ridge regression and SGD is (nearly) optimal for all $a>1$ and $r\ge 0$.}
\label{tab:power-law}
\end{figure}

\paragraph{Power law class.}
For $a>1$ and $r\ge 0$, we define the \emph{$(a,r)$-power law class} as
\begin{equation}\label{eq:power-law}
    \Pbb_{a,r}:=
    \big\{ \mu(\xB, y)\ \text{satisfying \Cref{assum:upper-bound,assum:lower-bound} with}\ 
   \sigma^2\eqsim 1,\, \lambda_i\lesssim i^{-a},\, \|\SigmaB^{-r}\wB^*\|^2_{\SigmaB}\lesssim 1
    \big\},
\end{equation}
and the \emph{strict $(a,r)$-power law class} as
\begin{equation}\label{eq:power-law-strict}
    \Pbb_{a,r}':=
    \big\{ \mu(\xB, y)\ \text{satisfying \Cref{assum:upper-bound,assum:lower-bound} with}\ 
   \sigma^2\eqsim 1,\, \lambda_i\eqsim i^{-a},\, \|\SigmaB^{-r}\wB^*\|^2_{\SigmaB}\lesssim 1
    \big\}.
\end{equation}
Here, the conditions on $\lambda_i$ and $\wB^*$ are known as the \emph{capacity} and \emph{source} conditions, respectively \citep{caponnetto2007optimal,dieuleveut2016nonparametric}. 
Note that $\Pbb_{a,r}$ uses a one-sided capacity condition \citep[see, e.g.,][(A3) and (A4) in Section 2.7]{dieuleveut2016nonparametric}, while $\Pbb'_{a,r}$ uses a two-sided capacity condition \citep[Definition 1]{caponnetto2007optimal}. 
In the remainder of this section, we will mainly focus on $\Pbb_{a, r}$ since $\Pbb_{a, r}' \subset \Pbb_{a, r}$.
Our results in this section will show that the \emph{worst-case} rates of ridge, SGD, and GD are consistent for $\Pbb_{a,r}$ and $\Pbb'_{a,r}$.
However, with respect to \emph{instance-wise} risk comparison, $\Pbb_{a,r}$ and $\Pbb_{a,r}'$ have different implications, which will be discussed at the end of this section.

We call \Cref{eq:power-law,eq:power-law-strict} (strict) power law classes because they include all \emph{power law} problems \citep[see, e.g.,][(a3) and (a4) in Section 2.7]{dieuleveut2016nonparametric} that satisfy \Cref{assum:upper-bound,assum:lower-bound} with
\begin{equation*}%
\sigma^2\eqsim 1,\quad 
\lambda_i \eqsim i^{-a},\quad \lambda_i(\uB_i^\top \wB^*)^2 \lesssim i^{-b},\quad \text{for some}\ b> 1+2a r, 
\end{equation*}
where $ (\lambda_i, \uB_i)_{i\ge 1}$ are the eigenvalues and corresponding eigenvectors of $\SigmaB$.
The latter conditions have also been adopted to define the power law class in the literature \citep[see, e.g.,][]{zhang2024optimality}.
When compared to these papers, it is convenient to convert the notation by $b=1+2a r$.

\paragraph{Minimax lower bound.}
The following \Cref{thm:power-law:lower-bound} provides a minimax lower bound on the rate for learning the $(a,r)$-power law class.
Variants of this minimax lower bound are well known in the literature in various settings \citep[see, e.g.,][Theorems 2 and 3]{caponnetto2007optimal}.
The version we present here is due to \citet[Theorem 2]{zhang2024optimality}.

\begin{proposition}[a minimax lower bound]\label{thm:power-law:lower-bound}
For every $a>1$ and $r \ge 0$, we have
\begin{align*}
\inf_{f} \sup_{\mu\in\Pbb_{a,r}} \Ebb \excessRisk_{\mu}(f(\XB,\yB)) =\Omega\Big( n^{-\frac{2a r}{1+ 2a r }} \Big),
\end{align*}
where the infimum is over all measurable maps,
\(
    f: (\Hbb\otimes\Rbb)^{\otimes n} \to \Hbb,
\)
and $(\XB, \yB)$ are $n$ samples drawn from $\mu$ independently.
\end{proposition}

\paragraph{Ridge regression is partially optimal.}
In the next corollary, we compute the exact rates for ridge regression using \Cref{thm:ridge}, the proof of which is included in \Cref{append:sec:power-law:ridge}.

\begin{corollary}[ridge regression rates]\label[corollary]{thm:power-law:ridge}
Let $\hat \wB$ be given by \Cref{eq:ridge} with regularization $\lambda \ge 0$.
\begin{itemize}[leftmargin=*]
    \item For $0\le r\le 1$, 
    setting
    \(
    \lambda =  n^{-\frac{a}{1+2a r}}
    \)
guarantees that for all $\mu \in \Pbb_{a,r}$,
    \begin{align*}
    \text{with probability at least}\  1-\exp(-n/c_0),\quad 
       \Ebb[ \excessRisk_{\mu} (\hat\wB)\, |\, \XB] = \bigO\Big( n^{-\frac{2a r}{1+2a r}} \Big),
    \end{align*}
    where $c_0$ is as defined in \Cref{thm:ridge}.
    \item For $r>1$, there exists  $\mu\in \Pbb_{a,r}$ such that 
   \begin{align*}
   \text{for all $\lambda$},\quad 
        \Ebb[ \excessRisk_{\mu}(\hat \wB) ] = \Omega\Big( n^{-\frac{2a}{1+2a}}  \Big) = \omega\Big(n^{-\frac{2a r}{1+2a r}}\Big).
    \end{align*}
\end{itemize}
\end{corollary}

\Cref{thm:power-law:ridge} shows that ridge regression is minimax optimal for $r\le 1$ but is polynomially suboptimal for $r>1$.
The first part of \Cref{thm:power-law:ridge} appears in \citep{caponnetto2007optimal} under a slightly different set of assumptions.
The suboptimality of ridge regression for $r> 1$ is referred to in the literature as the \emph{saturation effect} \citep[see, e.g.,][]{bauer2007regularization,dicker2017kernel}, but was not formally proved until the work by \citet{li2023saturation} (under a slightly different set of assumptions).
Building upon the ridge regression bounds derived by \citet{tsigler2023benign}, restated as \Cref{thm:ridge}, these are just simple calculations by bringing in the capacity and source conditions.

\paragraph{SGD is partially optimal.}
In the next corollary, we compute the exact rates for SGD using  \Cref{thm:sgd}, with proof deferred to \Cref{append:sec:power-law:sgd}.
\begin{corollary}[SGD rates]\label[corollary]{thm:power-law:sgd}
Let $\hat \wB$ be given by \Cref{eq:sgd} with initial stepsize $\eta>0$. Let $N=n/\log (n)$.
\begin{itemize}[leftmargin=*]
\item For $1+2ar \ge a$, setting 
\(
\eta = N^{-\frac{1+2ar - a}{1+2ar}}  / {(4\tr(\SigmaB))} \le 1/(4\tr(\SigmaB))
\)
guarantees that for all $\mu\in\Pbb_{a,r}$, 
\begin{align*}
    \Ebb \excessRisk_{\mu}(\hat \wB) = \bigO \Big(N^{-\frac{2ar}{1+2ar}} \Big) = \bigOT\Big(n^{-\frac{2ar}{1+2ar}}\Big).
\end{align*}
\item For $1+2ar < a$ and for each $n\ge 100$, there exists $\mu \in \Pbb_{a,r}$ such that
\begin{align*}
\text{for all} \ 0< \eta\le \frac{1}{4\tr(\SigmaB)},\quad 
    \Ebb\excessRisk_{\mu}(\hat\wB) =\Omega\big( N^{-2r} \big) =  \widetilde\Omega\big(n^{-2r}\big) .
\end{align*}
\end{itemize}
\end{corollary}

We note that the logarithmic factors in \Cref{thm:power-law:sgd} are purely artifacts of the stepsize scheduler in \cref{eq:sgd}, and can all be removed by considering the averaging of the tail iterates of SGD with a constant stepsize \citep{zou2023benign}.
Ignoring the logarithmic factors, \Cref{thm:power-law:sgd} suggests that SGD is minimax optimal for $r\ge (a-1)/(2a)$ but is polynomially suboptimal for $r< (a-1)/(2a)$.

Previously, a partial version of the first part of \Cref{thm:power-law:sgd}, restricted to $(a-1)/(2a)\le r\le 1$, appeared in \citet{dieuleveut2016nonparametric}. Note that \citet{dieuleveut2016nonparametric} did not obtain the tight rates for $r>1$ because they considered the iterate-averaged variant of SGD, which is known to be worse than the tail-averaging or last-iterate variants of SGD \citep{zou2023benign,wu2022last}.
The second part of \Cref{thm:power-law:sgd} is believed to be true in the literature by inspecting the best-known upper bounds \citep[see, e.g.,][]{pillaud2018statistical}. However, to the best of our knowledge, it has not been formally proven until our work, although the proof of which is a simple corollary of earlier results by \citet{zou2023benign} and \citet{wu2022last}.

\paragraph{GD is always optimal.}
In the following corollary, we compute the exact rates for GD using \Cref{thm:gd:ridge,thm:gd:exp}, whose proof is included in \Cref{append:sec:power-law:gd}.
\begin{corollary}[GD rates]\label[corollary]{thm:power-law:gd}
Let $\hat\wB$ be given by \Cref{eq:gd} with stepsize $\eta> 0$ and stopping time $t \ge 0$ such that 
\begin{align*}
\eta t = n^{\frac{a}{1+2ar}},\quad 
\eta \le \frac{1}{2\tr(\SigmaB)},\quad 
\text{and}\ \ \eta \ge n^{-\frac{1+2ar - a}{1+2ar}} \ \ \text{if}\ \  r>1.
\end{align*}
Let $c_0$ be the maximum of those in \Cref{thm:gd:ridge,thm:gd:exp}.
Then for all $a>1$, $r\ge 0$, and $\mu\in \Pbb_{a,r}$,
\begin{align*}
\text{with probability at least}\ 1-\exp\big(-k^*/c_0\big),\quad 
    \Ebb[\excessRisk_{\mu}(\hat \wB)\, | \, \XB] = \Ocal\Big(n^{-\frac{2ra }{1+2ra}}\Big).
\end{align*}
\end{corollary}

\Cref{thm:power-law:gd,thm:power-law:lower-bound} show that  GD is minimax optimal for all $a>1$ and $r\ge 0$.

Previously, \citet{lin2017optimal} showed that GD is minimax optimal for $1+2ar \ge a$, for which SGD is also (nearly) optimal according to \Cref{thm:power-law:sgd}.
Later, \citet{pillaud2018statistical} showed that GD is optimal for all $a>1$ and $r\ge 1$, but only in the finite-dimensional setting \citep[see (A3) in][for the $\kappa_0$ in which to be well-defined, the space needs to be finite-dimensional]{pillaud2018statistical}.
Recently, \citet{lin2025improved} showed that GD is minimax optimal for $1+2ar<a+1$. Thus the results by \citet{lin2017optimal} and \citet{lin2025improved} together give the results in \Cref{thm:power-law:gd} (the assumptions are slightly different).
Their analysis is ad hoc for the power law class. 
In comparison, \Cref{thm:power-law:gd} is a simple calculation by plugging the capacity and source conditions into our general bounds in \Cref{thm:gd:ridge,thm:gd:exp}.

\paragraph{GD can be polynomially better than both ridge regression and SGD.}
Comparing \Cref{thm:power-law:ridge,thm:power-law:sgd,thm:power-law:gd}, GD is minimax optimal for all power law classes, while both ridge regression and SGD are polynomially suboptimal for certain power law classes. Therefore, GD can be polynomially better than both ridge regression and SGD. Interestingly, the best of ridge regression and SGD is (nearly) optimal for all power law classes. These discussions are summarized in \Cref{tab:power-law}.

\paragraph{One-sided versus two-sided capacity conditions.}
To conclude this section, we discuss a subtle difference between $\Pbb_{a,r}$ in \Cref{eq:power-law}, which uses a one-sided capacity condition $\lambda_i \lesssim i^{-a}$, and $\Pbb_{a,r}'$ in \Cref{eq:power-law-strict}, which uses a two-sided capacity condition $\lambda_i \eqsim i^{-a}$.
Interestingly, the former is often used in the SGD literature \citep[see, e.g.,][(A3) and (A4) in Section 2.7]{dieuleveut2016nonparametric}, while the latter is often used in the ridge regression literature \citep[see, e.g.,][]{caponnetto2007optimal,blanchard2018optimal}.

Although $\Pbb_{a,r}'\subset\Pbb_{a,r}$, these two problem sets have the same implication in terms of minimax rates and the worst-case rates of ridge, SGD, and GD.
This is because, from the proof of \Cref{thm:power-law:lower-bound,thm:power-law:ridge,thm:power-law:sgd,thm:power-law:gd}, the hardest instances in $\Pbb_{a,r}$ occur at the boundaries, i.e., when $\lambda_i\eqsim i^{-a}$, which also belongs to $\Pbb_{a,r}'$.

However, these two conditions have different implications in terms of instance-wise risk comparisons between SGD and GD. 
Specifically, GD dominates SGD in $\Pbb_{a,r}'$ for all $a>1$ and $r\ge 0$ due to \Cref{thm:gd-vs-sgd:spectrum,example:spectrum}. 
In contrast, GD and SGD are incomparable in $\Pbb_{a,r}$. This is because the one-sided capacity condition, $\lambda_i \lesssim i^{-a}$, need not satisfy \Cref{assum:spectrum}, and therefore \Cref{thm:gd-vs-sgd:spectrum} does not apply.
Indeed, variants of \Cref{thm:gd-vs-sgd:heavy-tail} show that $\Pbb_{a,r}$ can contain instances where SGD is polynomially better than GD.
For instance, the example in \Cref{thm:gd-vs-sgd:heavy-tail} can be modified to satisfy $\lambda_i \lesssim i^{-a}$ by replacing the $1/d$ tail eigenvalues by $d^{-a}$, so it belongs to $\Pbb_{a,0}$ (this example can be further modified to allow $r>0$).
Then, for $d= n^{1.8/(2a-1)}$, a similar calculation suggests that SGD with an initial stepsize $\eta \eqsim n^{-0.1}\log^2(n)$ attains an $\bigOT(1/n)$ risk, while GD suffers from an $\Omega(n^{-(4a-3.8)/(2a-1)})$ risk. So GD can be polynomially worse than SGD for $1<a<1.4$, which together with \Cref{thm:power-law:sgd,thm:power-law:gd} implies that GD and SGD are incomparable in this problem set.


\section{Additional related works}\label{sec:related}

Our results build upon existing finite-sample bounds for ridge regression \citep{bartlett2020benign,tsigler2023benign} and SGD \citep{zou2023benign,wu2022last,wu2022power}. We refer the reader to these papers for earlier references on these topics.
Note that these bounds are tight in rates but can be loose in constant factors. 
Hence, our risk comparisons only focus on the rates. 
We point out that it is often possible to obtain bounds with sharp constant factors via random matrix theory \citep[see, e.g.,][for treatments of ridge regression in the dimension-free setting]{cheng2024dimension,misiakiewicz2024non}.
We believe it is possible to further extend our comparison results by including the impact of constant factors, which we leave as future work.

In the following part, we discuss other related works.

\paragraph{Statistical dominance.}
The best example of \emph{statistical dominance} is perhaps that the James–Stein estimator dominates OLS for estimating the mean of Gaussian distributions in three or higher dimensions \citep{james1992estimation}.
Over the years, there have been attempts to compare estimators in the context of linear regression \citep{dhillon2013risk,dicker2017kernel,ali2019continuous,zou2021benefits}.
We have discussed the results by \citet{ali2019continuous} at the end of \Cref{sec:gd-vs-ridge}. We discuss others next.

The pioneering work by \citet{dhillon2013risk} showed that, in \emph{fixed design}, \emph{principal component regression} (PCR) attains an excess risk no worse than $4$ times that of ridge regression, but can be much better. 
This is the first dominance result in linear regression to the best of our knowledge.
The fixed design setup greatly simplifies the analysis (see \Cref{append:sec:fixed-design} for a concise analysis of how GD dominates ridge regression in fixed design). 
However, we focus on the more challenging random design setting, which is closer to machine learning practice and also exhibits many intriguing phenomena not existing in fixed design settings, including benign overfitting. We refer the reader to \citet{tibshirani2024prediction} for an in-depth review of the differences between fixed and random designs. 

\citet{dicker2017kernel} compared PCR with ridge regression under capacity and source conditions. They showed that for a certain category of such problems, the excess risk upper bound they derived for PCR is polynomially better than that for ridge regression. 
Note that \citet{dicker2017kernel} only compared the best-known upper bounds available to them for these algorithms. 
In comparison, we compare GD and ridge regression in a much broader class of problems; moreover, we compare their actual excess risks rather than just upper bounds. 
We leave it as future work to include PCR in the set of comparison results. 

Recently, \citet{zou2021benefits} compared the sample complexity for SGD and ridge regression. They showed that for a class of Gaussian linear regression problems that are \emph{easy to optimize}, the sample complexity for SGD is no more than a polylogarithmic factor than that of ridge regression, but can be polynomially better (see \Cref{sec:sgd-vs-ridge} for more details). 
Compared to \citet{zou2021benefits}, we show that GD dominates ridge regression for all well-specified linear regression problems, including those hard to optimize. 
Additionally, we provide a detailed set of comparison results between GD and SGD, in which GD is generally incomparable with SGD but dominates SGD in a subset of problems with fast and continuously decaying spectra.

\paragraph{Early stopping.}
Early stopping is known to implicitly regularize GD through norm control \citep{bartlett1996valid}.
For linear regression, this has been formally demonstrated since early works by \citet{buhlmann2003boosting} in the fixed design and \citet{yao2007early} in the random design.
In particular, \citet{yao2007early} observed a crucial difference between early-stopped GD and ridge regression: under capacity and source conditions, GD may not suffer from the saturation phenomenon that ridge regression suffers from (see our discussions after \Cref{thm:power-law:ridge} in \Cref{sec:power-law}). 
They made this observation based on the best-known upper bounds available at the time.
In comparison, we show that GD dominates ridge regression for all well-specified linear regression problems, which can be viewed as a formal justification of their observation. 

As discussed after \Cref{thm:power-law:gd} in \Cref{sec:power-law}, the rates for GD under capacity and source conditions have been investigated before \citep[see, e.g.,][]{lin2017optimal,pillaud2018statistical,lin2025improved}.
Compared to the prior results, our new bounds for GD in \Cref{thm:gd:ridge,thm:gd:exp} can be applied beyond the capacity and source conditions, and recover their rates in the comparable regimes. 

Recently, \citet{zou2022risk} derived an upper bound for GD that adapts to the spectrum of the covariance. 
Their analysis motivates ours; in particular, a crucial matrix inequality (see \Cref{lemma:GD-projector} in \Cref{append:sec:gd:ridge}) in the proof of our \Cref{thm:gd:ridge} is from them. 
However, \citet{zou2022risk} only obtained a GD bound assuming an isotropic prior ($\Ebb \wB^{*\otimes 2} \propto \IB$) over the optimal parameter, while our \Cref{thm:gd:ridge} applies to any $\wB^*$. 
Additionally, our work differs from theirs by providing a novel lower bound in \Cref{thm:gd:lower-bound} and a novel SGD-type upper bound in \Cref{thm:gd:exp} for GD.

\section{Concluding remarks}
In this work, we compare excess risks achieved by GD, SGD, and ridge regression for linear regression. We show that GD dominates ridge regression for all well-specified linear regression problems, but is incomparable with SGD. 
For a significant subset of these problems where the covariance spectrum decays fast and continuously, however, GD does dominate SGD. 
When applied to problems under capacity and source conditions, we recover known results: GD is always minimax optimal, but both ridge regression and SGD are only partially minimax optimal.
Our results highlight the effectiveness of implicit regularization and an intriguing separation between batch and online learning. 

We have discussed several open problems throughout the paper. Besides, the following future directions are worth mentioning.

\paragraph{Beyond linear regression.}
Our GD analysis, as well as the prior results for ridge regression and SGD used in this work, is specific to linear regression in a fundamental way. 
In particular, these tight analyses rely on the analytic formulas for the GD, ridge regression, and SGD estimators, which are generally not available beyond linear regression problems.
It is unclear to what extent our results generalize to other classes of statistical learning problems.
As concrete questions, how would early-stopped GD compare to the $\ell_2$-regularized empirical risk minimizer for Gaussian logistic regression? How would mirror descent compare to LASSO for sparse linear regression?

\paragraph{Negative ridge and oscillatory GD.}
A limitation in our comparison results is that we only consider ridge regression with nonnegative regularization ($\lambda\ge 0$) and GD with a small, stable stepsize ($\eta \le n/\|\XB\XB^\top\|$). 
We believe this is a mild limitation, as a negative $\ell_2$-regularization is rarely meaningful in practice for norm control. 
However, for certain linear regression problems, the optimal $\lambda$ for ridge regression could indeed be negative \citep[see, e.g.,][]{tsigler2023benign}.
It is unclear if GD still dominates ridge regression when allowing $\lambda<0$. 
We conjecture that this is true when extending the stepsize for GD from small ones to \emph{moderate} ones, $n/\|\XB\XB^\top\|<\eta < 2n/\|\XB\XB^\top\|$, for which GD oscillates in iterates but still monotonically decreases the empirical risk. This is left for future investigation.

\paragraph{Principal component regression.}
Another interesting question is how \emph{principal component regression} (PCR) compares to GD, SGD, and ridge regression for random-design linear regression. 
While PCR is easy to analyze in the fixed design setting~\citep{dhillon2013risk}, its sharp bound remains unknown in the more interesting random design setting. 
From an intuitive perspective, the early-stopped GD estimator is entry-wise ``exponentially close'' to the PCR estimator. But this does not necessarily imply their risks are similar, especially in the high-dimensional regime.
Moreover, it is important to notice that, unlike GD (and ridge regression), PCR depends on the spectrum of the empirical covariance in a \emph{non-continuous} way. 
This might limit the ability of PCR to balance bias and variance errors in the high-dimensional regime.
We leave it as future work to pin down the role of PCR in the context of our work.

\paragraph{Statistical effect of data reuse.}
Our work suggests that GD, a purely \emph{batch} method, and SGD, a purely \emph{online} method, both have their own benefits and limitations. 
This leads to a natural question: how to reuse data optimally? 
To make the question more concrete, let us consider \emph{multi-epoch} SGD, which performs multiple epochs of one-pass stochastic gradient descent steps. 
Clearly, multi-epoch SGD is no worse than either SGD or GD (in terms of statistical performance): multi-epoch SGD is identical to SGD when stopping at the first epoch, and converges to gradient flow as the stepsize tends to zero (rescaling the number of epochs properly). Hence, our theory suggests that multi-epoch SGD dominates each of SGD, GD, and ridge regression for linear regression. 
Note, however, that this does not hold for \emph{multi-pass} SGD that samples data with replacement, which is known to be no better than GD for linear regression \citep[Theorem 4.1]{zou2022risk}.
So the way to reuse data matters.
We ask the following questions: would multi-epoch SGD dominate the \emph{best} of GD and SGD? Furthermore, is there an even better way to reuse data?

\section*{Acknowledgments}
We thank Sivaraman Balakrishnan, Matus Telgarsky, Ryan Tibshirani, and Yuan Yao for helpful discussions in various stages of this work. 
We gratefully acknowledge the NSF's support of FODSI through grant DMS-2023505 and of the NSF and the Simons Foundation for the Collaboration on the Theoretical Foundations of Deep Learning through awards DMS-2031883 and \#814639 and of the ONR through MURI award N000142112431.
JDL acknowledges support of Open Philanthropy, NSF IIS 2107304,  NSF CCF 2212262, ONR Young Investigator Award, NSF CAREER Award 2144994, and NSF CCF 2019844.

\bibliography{ref}
\appendix

\section{Fixed design}\label[appendix]{append:sec:fixed-design}

In this part, we provide a simple analysis for GD and ridge regression for well-specified linear regression with a fixed design matrix. 

Let $\XB\in \Rbb^{n \times d}$ be a fixed design matrix, where $d$ is allowed to be infinite. 
Let the responses be
\[\yB \sim \Ncal(\XB\wB^*, \sigma^2 \IB),\]  
where $\wB^* \in \Rbb^d$ is the true parameter and $\sigma^2>0$ is the noise variance. 
Let the covariance matrix be 
\[\SigmaB := \frac{1}{n} \XB^\top\XB.\]
Let $(\lambda_i)_{i\ge 1}$ be the eigenvalues of $\SigmaB$ in non-increasing order.

\begin{theorem}[ridge regression in fixed design]\label{thm:fixed-design:ridge}
For $\hat\wB$ given by \Cref{eq:ridge} with $\lambda\ge 0$,
we have 
\begin{align*}
    \Ebb  \|\hat \wB - \wB^* \|^2_{\SigmaB} \eqsim \lambda^2 \|\wB^*\|^2_{\SigmaB^{-1}_{0:k}} + \|\wB^*\|^2_{\SigmaB_{k:\infty}}
    + \sigma^2  \frac{k + (1/\lambda)^2 \sum_{i>k}\lambda^2_{i}}{n},
\end{align*}
where $k$ is such that 
\begin{align*}
    \lambda_1 \ge \cdots \ge \lambda_k \ge \lambda \ge \lambda_{k+1} \ge \cdots.
\end{align*}
\end{theorem}
\begin{proof}[Proof of \Cref{thm:fixed-design:ridge}]
Let $\epsilonB := \yB - \XB\wB^* \sim \Ncal(0, \sigma^2\IB)$ be the noises.
By definition of $\hat \wB$ in \Cref{eq:ridge}, we have   
\begin{align*}
    \hat\wB - \wB^* 
    &= (\XB^\top \XB + n\lambda \IB)^{-1} \XB^\top \yB - \wB^* \\
    &= (\XB^\top \XB + n\lambda \IB)^{-1} \big( \XB^\top \XB\wB^* + \XB^\top \epsilonB\big) - \wB^* \\
    &= -n\lambda (\XB^\top \XB + n\lambda \IB)^{-1} \wB^* + (\XB^\top \XB + n\lambda \IB)^{-1}\XB^\top \epsilonB \\
    &= -\lambda (\SigmaB + \lambda \IB)^{-1} \wB^* + \frac{1}{n}(\SigmaB + \lambda \IB)^{-1}\XB^\top \epsilonB.
\end{align*}
Thus, we have 
\begin{align*}
    \Ebb \|\hat \wB - \wB^* \|^2_{\SigmaB}
    &= \lambda^2 \| (\SigmaB + \lambda \IB)^{-1} \wB^* \|^2_{\SigmaB} + \sigma^2 \bigg\la \SigmaB,\ \frac{1}{n^2}(\SigmaB + \lambda \IB)^{-2}\XB^\top\XB \bigg\ra \\
    &= \big\la\wB^{*\otimes 2},\ \lambda^2 \SigmaB(\SigmaB + \lambda \IB)^{-2} \big\ra +  \frac{\sigma^2}{n} \tr\big(\SigmaB^2 (\SigmaB + \lambda \IB)^{-2}\big) \\
    &\eqsim \big\la\wB^{*\otimes 2},\ \lambda^2 \SigmaB_{0:k}^{-1} + \SigmaB_{k:\infty} \big\ra +  \frac{\sigma^2}{n} \bigg(k + (1/\lambda)^2 \sum_{i>k}\lambda_i^2\bigg).
\end{align*}
This completes the proof.
\end{proof}

\begin{theorem}[GD in fixed design]\label{thm:fixed-design:gd}
For $\hat \wB$ given by \Cref{eq:gd} with stepsize $\eta \le 1/\lambda_1$ and stopping time $t$, we have
\begin{align*}
\Ebb \|\hat \wB - \wB^* \|^2_{\SigmaB}
&\eqsim \|(\IB-\eta \SigmaB)^t \wB^*\|^2_{\SigmaB} + \sigma^2  \frac{k + (\eta t)^2 \sum_{i>k}\lambda^2_{i}}{n} \\
&\lesssim \bigg(\frac{1}{\eta t}\bigg)^2 \|\wB^*\|^2_{\SigmaB^{-1}_{0:k}} + \|\wB^*\|^2_{\SigmaB_{k:\infty}} + \sigma^2  \frac{k + (\eta t)^2 \sum_{i>k}\lambda^2_{i}}{n} ,
\end{align*}
where $k$ is such that 
\begin{align*}
    \lambda_1 \ge \cdots \ge \lambda_k \ge \frac{1}{\eta t} \ge \lambda_{k+1} \ge \cdots.
\end{align*}
\end{theorem}
\begin{proof}[Proof of \Cref{thm:fixed-design:gd}]
Let $\epsilonB := \yB - \XB\wB^* \sim \Ncal(0, \sigma^2\IB)$ be the noises.
By \Cref{eq:gd}, we have 
\begin{align*}
\wB_{t+1} -\wB^*
    &= \wB_t - \wB^* - \frac{\eta}{n}\XB^\top (\XB \wB_t - \yB) \\
    &= \wB_t - \wB^* - \frac{\eta}{n}\XB^\top (\XB \wB_t - \XB \wB^* - \epsilonB) \\
    &= (\IB - \eta \SigmaB) (\wB_t- \wB^*) + \frac{\eta}{n}\XB^\top \epsilonB,
\end{align*}
which implies 
\begin{align*}
    \wB_t- \wB^*
    &= (\IB - \eta \SigmaB)^t (\wB_0 - \wB^*) + \frac{\eta}{n}\sum_{k=0}^{t-1} (\IB -\eta \SigmaB)^{t-1-k} \XB^\top \epsilonB  \\
    &= - (\IB - \eta \SigmaB)^t \wB^* + \frac{1}{n}\big( \IB- (\IB -\eta \SigmaB)^{t}\big) \SigmaB^{-1} \XB^\top \epsilonB .
\end{align*}
Thus, we have 
\begin{align*}
    \Ebb \|\wB_t - \wB^*\|_{\SigmaB}^2 
    &= \big\|(\IB - \eta \SigmaB)^t \wB^*\big\|_{\SigmaB}^2  + \sigma^2 \bigg\la \SigmaB,\ \frac{1}{n^2}\big( \IB- (\IB -\eta \SigmaB)^{t}\big)^2 \SigmaB^{-2} \XB^\top\XB\bigg\ra  \\
    &= \big\la \wB^{*\otimes 2},\ \SigmaB (\IB - \eta \SigmaB)^{2t}\big\ra + \frac{\sigma^2}{n} \tr\Big( \big( \IB- (\IB -\eta \SigmaB)^{t}\big)^2\Big) \\
    &\eqsim \big\la \wB^{*\otimes 2},\ \SigmaB (\IB - \eta \SigmaB)^{2t}\big\ra + \frac{\sigma^2}{n} \tr \big((\IB_{0:k} + \eta t \SigmaB_{k:\infty})^{2}\big) \\
    &\eqsim \big\la \wB^{*\otimes 2},\ \SigmaB (\IB - \eta \SigmaB)^{2t}\big\ra + \frac{\sigma^2}{n} \bigg(k + (\eta t)^2 \sum_{i>k}\lambda_i^{2}\bigg) \\
    &\lesssim  \bigg\la \wB^{*\otimes 2},\ \frac{1}{(\eta t)^2}\SigmaB_{0:k}^{-1} + \SigmaB_{k:\infty}\bigg\ra + \frac{\sigma^2}{n} \bigg(k + (\eta t)^2 \sum_{i>k}\lambda_i^{2}\bigg).
\end{align*}
This completes the proof.
\end{proof}

\Cref{thm:fixed-design:ridge,thm:fixed-design:gd} suggest that GD, in fixed design, achieves a risk no worse than that of ridge by at most a constant factor.
We see this by matching $1/(\eta t)$ with $\lambda$.
It is also clear that GD can be much better than ridge in cases where $\wB^*$ mainly lies in the large eigenvalue directions.

\section{\texorpdfstring{Proof of \Cref{thm:gd:ridge}}{Ridge-type bound for GD}}\label[appendix]{append:sec:gd:ridge}

\paragraph{Notation.}
We first introduce a set of notation following \citet{bartlett2020benign} and \citet{tsigler2023benign}.
Let the Gram matrix and the empirical covariance matrix be 
\[
\AB := \XB\XB^\top \in \Rbb^{n\times n},\quad \hat \SigmaB := \frac{1}{n}\XB^\top\XB \in \Hbb\otimes \Hbb,
\]
respectively.
Without loss of generality, assume $\SigmaB$ is diagonal. 
For an index $k$, we split the matrices and vectors in the following way,
\begin{align*}
    \XB = \begin{bmatrix}
        \XB_{\le k} & \XB_{>k}\\
    \end{bmatrix},\quad 
    \SigmaB = \begin{bmatrix}[1.2]
      \SigmaB_{\le k} & \\
        & \SigmaB_{>k} \\
    \end{bmatrix},\quad 
    \wB^* = \begin{bmatrix}[1.2]
        \wB^*_{\le k} \\
        \wB^*_{>k}
    \end{bmatrix},
\end{align*}
We also write $\XB = \ZB\SigmaB^{1/2}$, where $\ZB$ has independent and $\sigma_x^2$-subgaussian entries according to \Cref{assum:upper-bound:item:x}.
We decompose $\ZB$ following the same convention for decomposing $\XB$.
The following shrinkage matrix, introduced by \citet{zou2022risk}, plays a central role in our GD analysis: 
\[
\tilde\AB := \big(\IB - (\IB-\sfrac{\eta}{n} \AB)^t \big)^{-1} \AB.
\]
In comparison, \citet{tsigler2023benign} considered ridge regression that corresponds to the following shrinkage matrix (note that their $\lambda$ translates to $n \lambda$ in our notation \Cref{eq:ridge}):
\[
\tilde\AB_{\ridge} := \AB + n\lambda \IB.
\]
For convenience, we define
\[
\tilde \AB_{k} := \tilde\AB - \XB_{\le k}\XB_{\le k}^\top,\quad 
\AB_k := \XB_{>k}\XB_{>k}^\top + \frac{n}{\eta t} \IB.
\]
Finally, denote the label noise as 
\begin{align*}
    \epsilonB := \yB - \XB\wB^*.
\end{align*}

\paragraph{An analytical formula for GD.}
By \Cref{eq:gd}, we have
\begin{align*}
    \wB_{t+1}-\wB^* &= \wB_t - \wB^* - \frac{\eta}{n}\XB^\top \big( \XB \wB_t - \yB\big) \\
    &= \wB_t - \wB^* - \frac{\eta}{n}\XB^\top \big( \XB \wB_t - \XB\wB^* - \epsilonB \big) \\
    &= \big(\IB - \eta \hat\SigmaB \big)(\wB_t - \wB^*) + \frac{\eta}{n}\XB^\top\epsilonB.
\end{align*}
Unrolling the recursion, we get
\begin{align}
    \wB_t - \wB^* &= \big(\IB - \eta \hat\SigmaB \big)^t(\wB_0 - \wB^*) + \frac{\eta}{n} \sum_{s=0}^{t-1}\big(\IB - \eta \hat\SigmaB \big)^{t-1-s}\XB^\top\epsilonB \notag \\
    &= - \big(\IB - \eta \hat\SigmaB\big)^t \wB^* + \frac{1}{n}\big(\IB- (\IB - \eta \hat\SigmaB )^{t}\big) \hat\SigmaB^{-1}\XB^\top\epsilonB.\label{eq:gd:decompose}
\end{align}

\paragraph{Bias-variance decomposition.}
By the formula of GD \Cref{eq:gd:decompose}, we can bound its (expected) excess risk by
\begin{align*}
&\ \Ebb_{\epsilonB}  \excessRisk(\wB_t) = \Ebb_{\epsilonB} \|\wB_t - \wB^*\|^2_{\SigmaB} \\ 
    &\le 2 \big\|(\IB-\eta \hat\SigmaB)^t \wB^* \big\|^2_{\SigmaB} + \frac{2}{n^2} \Ebb_{\epsilonB} \big\|\big(\IB- (\IB - \eta \hat\SigmaB )^{t}\big) \hat\SigmaB^{-1}\XB^\top \epsilonB\big\|^2_{\SigmaB} \\
    &\le 2 \big\|(\IB-\eta \hat\SigmaB)^t\wB^* \big\|^2_{\SigmaB} + \frac{2}{n^2} \big\la \XB\hat\SigmaB^{-1}\big(\IB- (\IB - \eta \hat\SigmaB )^{t}\big) \SigmaB\big(\IB- (\IB - \eta \hat\SigmaB )^{t}\big) \hat\SigmaB^{-1}\XB^\top,\, \sigma^2\IB \big\ra \\
    &= 2 \big\la (\IB-\eta \hat\SigmaB)^t\SigmaB(\IB-\eta \hat\SigmaB)^t,\, \wB^{*\otimes 2} \big\ra + 2\sigma^2 \bigg\la  \frac{1}{n}\big(\IB- (\IB - \eta \hat\SigmaB )^{t}\big)\hat\SigmaB^{-1}\big(\IB- (\IB - \eta \hat\SigmaB )^{t}\big)  ,\, \SigmaB \bigg\ra, 
\end{align*}
where the first inequality is by Cauchy–Schwarz inequality and the second inequality is by \Cref{assum:upper-bound:item:noise}.
Define the bias and variance matrices as 
\begin{equation}\label{eq:bias-variance}
    \BB:= (\IB-\eta \hat\SigmaB)^t\SigmaB(\IB-\eta \hat\SigmaB)^t,\quad 
    \CB:=  \frac{1}{n}\big(\IB- (\IB - \eta \hat\SigmaB )^{t}\big)\hat\SigmaB^{-1}\big(\IB- (\IB - \eta \hat\SigmaB )^{t}\big) ,
\end{equation}
respectively.
We then have 
\begin{align*}
    \Ebb_{\epsilonB}  \excessRisk(\wB_t) \le 2\la \BB,\, \wB^{*\otimes 2}\ra + 2\sigma^2 \la \CB,\, \SigmaB\ra.
\end{align*}
The following \Cref{lemma:basic-algebra} provides a convenient reformulation of the bias and variance matrices using the GD shrinkage matrix $\tilde\AB$. 
\begin{lemma}[basic algebra]\label[lemma]{lemma:basic-algebra}
The bias and variance matrices defined in \Cref{eq:bias-variance} are equivalent to
\begin{equation*}
\BB = \big(\IB-\XB^\top \tilde\AB^{-1} \XB\big) \SigmaB \big(\IB-\XB^\top \tilde\AB^{-1} \XB\big),\quad 
    \CB = \XB^\top\tilde\AB^{-2}\XB,
\end{equation*}
respectively,
where the shrinkage matrix is 
\[
\tilde\AB := \big(\IB - (\IB-\sfrac{\eta}{n} \AB)^t \big)^{-1} \AB,\quad \text{where} \ \AB:= \XB\XB^\top.
\]
\end{lemma}
\begin{proof}[Proof of \Cref{lemma:basic-algebra}]
Using the algebraic facts that 
\begin{align*}
    \big( \XB^\top \XB \big)^{-1}\XB^\top = \XB^\top \big( \XB \XB^\top \big)^{-1},\quad  \big( \XB^\top \XB \big) \XB^\top= \XB^\top \big( \XB \XB^\top \big),
\end{align*}
and the definition $\AB:= \XB\XB^\top$,
we have
\begin{align*}
    \big(\IB - \eta\hat \SigmaB \big)^t 
    &= \big(\IB - \sfrac{\eta}{n}\XB^\top\XB \big)^{t-1} - \sfrac{\eta}{n}\big(\IB - \sfrac{\eta}{n}\XB^\top\XB \big)^{t-1} \XB^\top \XB \\
    &= \big(\IB - \sfrac{\eta}{n}\XB^\top\XB \big)^{t-1} - \sfrac{\eta}{n} \XB^\top \big(\IB - \sfrac{\eta}{n}\AB \big)^{t-1} \XB \\
    &= \IB - \sfrac{\eta}{n} \XB^\top \sum_{s=0}^{t-1} \big(\IB - \sfrac{\eta}{n}\AB \big)^{s} \XB \\
    &= \IB -\XB^\top\big(\IB - (\IB-\sfrac{\eta}{n} \AB)^t \big) \AB^{-1}\XB \\
    &= \IB - \XB^\top\tilde\AB^{-1}\XB,
\end{align*}
where the third equality is by unrolling the recursion.
Plugging this into the definitions of bias and variance matrices in \Cref{eq:bias-variance}, we get 
\begin{align*}
    \BB := (\IB-\eta \hat\SigmaB)^t\SigmaB(\IB-\eta \hat\SigmaB)^t = \big(\IB-\XB^\top \tilde\AB^{-1} \XB\big) \SigmaB \big(\IB-\XB^\top \tilde\AB^{-1} \XB\big),
\end{align*}
and 
\begin{align*}
    \CB
    &:= \frac{1}{n}\big(\IB- (\IB - \eta \hat\SigmaB )^{t}\big)\hat\SigmaB^{-1}\big(\IB- (\IB - \eta \hat\SigmaB )^{t}\big) \\
    &= \frac{1}{n}\XB^\top\tilde\AB^{-1}\XB\hat{\SigmaB}^{-1}\XB^\top\tilde\AB^{-1}\XB \\
    &= \XB^\top\tilde\AB^{-1}\XB(\XB^\top\XB)^{-1}\XB^\top\tilde\AB^{-1}\XB \\  
    &= \XB^\top\tilde\AB^{-2}\XB.
\end{align*}
This completes our proof.
\end{proof}

\subsection{Basic lemmas}\label[appendix]{append:sec:gd:lemmas}
In this part, we prepare ourselves with some basic yet useful lemmas for our proof.

The next lemma, first noted by \citet{zou2022risk}, suggests that the shrinkage matrix $\tilde\AB$ for GD is comparable to that for ridge regression ($\tilde\AB_{\ridge}$).

\begin{lemma}[Lemma 5.4 in \citep{zou2022risk}]\label[lemma]{lemma:GD-projector}
For every $\eta \le n / \|\AB\|$, we have
\begin{align*}
 \tilde\AB &\succeq \max\bigg\{ \AB,\, \frac{n}{\eta t} \IB \bigg\} \succeq \frac{1}{2}\bigg(\AB + \frac{n}{\eta t} \IB\bigg),\\
 \tilde\AB &\preceq \AB + \frac{2n}{\eta t} \IB. 
\end{align*}
\end{lemma}
\begin{proof}[Proof of \Cref{lemma:GD-projector}]
Let $\gamma := \eta/n$ to simplify our notation.
For every $z>0$ and $0<\gamma <1/z$, we have 
\[1-(1-\gamma z)^t \le \min\{1,\, \gamma t z\}.\]
Thus we have
\begin{align*}
  &  \IB - (\IB-\gamma \AB)^t \preceq \min\{ \IB,\, \gamma t \AB \} \\
 \Rightarrow \quad  & \tilde\AB = \big(\IB - (\IB-\gamma \AB)^t \big)^{-1} \AB \succeq \max\bigg\{ \AB,\, \frac{1}{\gamma t} \IB \bigg\} \succeq \frac{1}{2}\bigg(\AB + \frac{1}{\gamma t} \IB\bigg),
\end{align*}
which verifies the lower bound in the claim.
For the upper bound, consider 
\[
f(z) := \frac{z(1-\gamma z)^t}{1-(1-\gamma z)^t},\quad 0< z < 1/\gamma.
\]
For $t\ge \ln (2) / (\gamma z)$, we have $  (1-\gamma z)^t \le \min\{1/(\gamma z t),\, 1/2\}$, which implies
\begin{align*}
    f(z)\le \frac{z 1/(z \gamma t)}{1 - 1/2} = \frac{2}{\gamma t}.
\end{align*}
For $t < \ln (2) / (\gamma z)$, we have $1 - (1-\gamma z)^t \ge \gamma z t / 2$ and $(1-\gamma z)^t \le 1$, which implies
\begin{align*}
    f(z) \le \frac{z}{ \gamma z t /2} \le \frac{2}{\gamma t}.
\end{align*}
So we have $f(z)\le 2/(\gamma t)$ for all $0< \gamma z<1$.
This suggests that 
\begin{align*}
    \tilde\AB - \AB = \big(\IB - (\IB-\gamma \AB)^t \big)^{-1}(\IB-\gamma \AB)^t\AB \preceq \frac{2}{\gamma t}\IB.
\end{align*}
This completes our proof.
\end{proof}

The next elementary lemma plays a key role in our GD upper bound proof.

\begin{lemma}[basic matrix bounds]\label[lemma]{lemma:GD-projector-decomp}
For every $\eta \le n/\|\AB\|$, we have 
\begin{align*}
\tilde\AB_k &\preceq   2\AB_k, \\ 
\tilde\AB^2 &\succeq \frac{1}{4}\big( \XB_{\le k}\XB_{\le k}^\top + \AB_k \big)^2.
\end{align*}
\end{lemma}
\begin{proof}[Proof of \Cref{lemma:GD-projector-decomp}]
The first claim is by
\begin{align*}
    \tilde\AB_k := \tilde\AB -  \XB_{\le k}\XB_{\le k}^\top \preceq \AB + \frac{2n}{\eta t} \IB - \XB_{\le k}\XB_{\le k}^\top = \XB_{>k}\XB_{>k}^\top + \frac{2n}{\eta t} \IB \preceq 2 \AB_k,
\end{align*}
where the first inequality is by \Cref{lemma:GD-projector}.

For the second claim, notice that $\tilde\AB$ commutes with $\AB$, then \Cref{lemma:GD-projector} implies that
\begin{align*}
\tilde\AB\succeq \frac{1}{2}\bigg(\AB + \frac{n}{\eta t} \IB\bigg)   \quad \Rightarrow\quad \tilde\AB^2\succeq
\frac{1}{4}\bigg(\AB + \frac{n}{\eta t} \IB\bigg)^2.
\end{align*}
We complete the proof by noticing that $\AB = \XB_{\le k} \XB_{\le k}^\top + \XB_{>k}\XB_{>k}^\top $. 
\end{proof}

The following basic concentration results all appear in prior works, e.g., \citep{bartlett2020benign}. We include them here with a brief proof for completeness. 
\begin{lemma}[basic concentration bounds]\label[lemma]{lemma:basic-concentration}
Under \Cref{assum:upper-bound:item:x}, the joint of the following events holds with probability at least $1-\exp(-n/c_0)$:
\begin{gather*}
\frac{n}{2} \IB_k\preceq\SigmaB^{-\frac{1}{2}}_{\le k} \XB_{\le k}^\top \XB_{\le k}\SigmaB^{-\frac{1}{2}}_{\le k} \preceq 2n  \IB_k \  \text{for $k$ such that $k\le n/c_3$}, \\
\XB_{>k}\SigmaB_{>k}\XB_{>k}^\top
    \preceq c_1 \bigg(n \lambda_{k+1}^2 + \sum_{i>k}\lambda_i^2\bigg)\IB_n, \\
    \big\| \XB_{>k} \wB^*_{>k}\big\|^2 \le c_1 n \|\wB^*_{>k}\|^2_{\SigmaB_{>k}}, \\
    \tr\big(\XB_{>k} \SigmaB_{>k}\XB_{>k}^\top\big) \le c_1 n \sum_{i>k}\lambda_i^2,
\end{gather*}
where $c_0, c_1, c_3 >1$ only depend on $\sigma_x^2$. 
\end{lemma}
\begin{proof}[Proof of \Cref{lemma:basic-concentration}]
Let $(a_i)_{i\ge 1}$ be a sequence of fixed non-negative scalars in non-increasing order.
Let $(\zB_i)_{i\ge 1}$ be a sequence of independent random vectors in $\Rbb^r$ with independent $\sigma_x^2$-subgaussian entries with unit variance. 
Then by Bernstein's inequality and a union bound on $r$-dimensional unit sphere \citep[see, e.g.,][Proof of Lemma 9 in Appendix C]{bartlett2020benign}, we have 
\begin{equation}\label{eq:bernstein-union-bound}
    \Pr\bigg(\bigg\|\sum_{i}a_i \zB_i\zB_i^\top - \sum_{i}a_i\IB_n \bigg\| > t \bigg)  < \exp\bigg( -\frac{1}{c} \min\bigg\{\frac{t^2}{\sum_i a_i^2},\, \frac{t}{a_1} \bigg\} + 10r\bigg),
\end{equation}
where $c>1$ only depends on $\sigma_x^2$.

\emph{The first claim.}~
Notice that 
\begin{align*}
    \SigmaB^{-\frac{1}{2}}_{\le k} \XB_{\le k}^\top \XB_{\le k}\SigmaB^{-\frac{1}{2}}_{\le k}  = \sum_{i=1}^n \zB_i\zB_i,
\end{align*}
where $\zB_i \in \Rbb^k$ for $i=1,\dots,n$ are random vectors with independent $\sigma_x^2$-subgaussian entries with unit variance by \Cref{assum:upper-bound:item:x}. 
Applying \Cref{eq:bernstein-union-bound} with $r=k$, $a_1=\dots=a_n=1$, and $t = n / 2 $, we get 
\begin{equation*}
    \Pr\bigg(\bigg\|\XB_{\le k}^\top \SigmaB^{-1}_{\le k}\XB_{\le k} - n\IB_k \bigg\| > \frac{n}{2} \bigg)  < \exp\bigg( -\frac{n}{4 c} + 10 k\bigg).
\end{equation*}
Provided that $n\ge 80c k$, this gives the first claim.

\emph{The second claim.}~
Notice that 
\begin{align*}
    \XB_{>k}\SigmaB_{>k}\XB_{>k}^\top = \sum_{i>k} \lambda_i^2 \zB_i\zB_i^\top,
\end{align*}
where $\zB_i\in\Rbb^n$ for $i>k$ are random vectors with independent $\sigma_x^2$-subgaussian entries with unit variance by \Cref{assum:upper-bound:item:x}. 
Applying \Cref{eq:bernstein-union-bound} with $n$ as dimension, $(\lambda_i^2)_{i\ge k}$ as the weights, and 
\[
t = \max\Bigg\{11c n \lambda_{k+1}^2,\, \sqrt{11c n\sum_{i>k}\lambda_i^4}\Bigg\}
\le 11c n \lambda_{k+1}^2 + \sum_{i>k}\lambda_i^2,
\]
we have
\begin{align*}
  &  \Pr\bigg(\bigg\| \XB_{>k}\SigmaB_{>k}\XB_{>k}^\top - \sum_{i>k} \lambda_i^2\IB_n \bigg\| > t \bigg) < \exp(-n) \\ 
  \Rightarrow\  & \Pr\bigg(\big\|\XB_{>k}\SigmaB_{>k}\XB_{>k}^\top\big\|  > 11c n \lambda_{k+1}^2 + 2\sum_{i>k}\lambda_i^2 \bigg) < \exp(-n),
\end{align*}
which verifies the second claim.

\emph{The third claim.}~
Note that entries of $\XB_{>k} \wB^*_{>k}$ are independent and $\sigma_x^2 \| \wB^*_{>k}  \|^2_{\SigmaB_{>k}}$-subgaussian by \Cref{assum:upper-bound:item:x}, and there are $n$ such entries.
So the third claim is a standard bound on the norm of a high-dimensional subgaussian random vector \citep[see, e.g.,][Theorem 3.1.1]{vershynin2026high}. 

\emph{The fourth claim.}~
We write 
\[
\XB_{>k} = \begin{bmatrix}
    \zB_1^\top \\
    \vdots \\
    \zB_n^\top 
\end{bmatrix}
\SigmaB_{>k}^{\frac{1}{2}},
\]
where $\zB_i$ for $i=1,\dots,n$ are independent random vectors with independent $\sigma_x^2$-subgaussian entries with unit variance by \Cref{assum:upper-bound:item:x}.
Then the $i$-th diagonal entry is 
\[
\big(\XB_{>k}\SigmaB_{>k}\XB_{>k}^\top\big)_{ii} = \zB_i^\top \SigmaB^2_{>k} \zB_i,
\]
which has expectation $\sum_{i>k}\lambda_i^2$ and is $\sigma_x^2 \sum_{i>k}\lambda_i^2$-subexponential after centering.
Therefore,
\begin{align*}
   \tr\big( \XB_{>k}\SigmaB_{>k}\XB_{>k}^\top \big) = \sum_{i=1}^n \zB_i^\top \SigmaB^2_{>k} \zB_i
\end{align*}
is the sum of $n$ independent $\sigma_x^2 \sum_{i>k}\lambda_i^2$-subexponential random variables (after centering). 
So the fourth claim is a standard application of Bernstein's inequality \citep[Theorem 2.9.1]{vershynin2026high}.

Finally, we complete the proof by applying a union bound over the four events.
\end{proof}

The following lemma is adapted from \citep[Lemma 9]{bartlett2020benign} or \citep[Lemma 3]{tsigler2023benign}.
\begin{lemma}[tail concentration]\label[lemma]{lemma:tail-regularization}
Under \Cref{assum:upper-bound:item:x}, there are $c_0, c_1, c_2 >1$ that only depend on $\sigma_x^2$, for which the following holds. 
For each $k$ such that 
\[
\frac{1}{\eta t} + \frac{ \sum_{i>k}\lambda_i}{n} \ge c_2 \lambda_{k+1},
\]
with probability at least $1-\exp(- n / c_0)$ it holds that
\begin{align*}
\frac{1}{c_1}   \bigg(\frac{n}{\eta t } + \sum_{i>k}\lambda_i\bigg)\IB \preceq  \AB_{k} \preceq c_1 \bigg(\frac{n}{\eta t } + \sum_{i>k}\lambda_i\bigg)\IB.
\end{align*}
\end{lemma}
\begin{proof}[Proof of \Cref{lemma:tail-regularization}]
Recall \Cref{assum:upper-bound:item:x} and our definition that 
\[
\AB_k = \XB_{>k}\XB_{>k}^\top + \frac{n}{\eta t}\IB.
\]
Then \citep[Lemma 9]{bartlett2020benign} implies that: for each $k$,
\[
\frac{1}{c}\sum_{i>k}\lambda_i - c n \lambda_{k+1} + \frac{n}{\eta t} \le \lambda_{\min}(\AB_k) \le \lambda_{\max}(\AB_k) \le c \bigg( \sum_{i>k}\lambda_i + n \lambda_{k+1} \bigg) + \frac{n}{\eta t}
\]
holds with probability at least $1-\exp(-n/c_0)$ for some $c_0, c_1>1$ that only depend on $\sigma_x^2$.
We finish the proof by imposing the condition on $k$ for $c_2 = 2c^2$ and setting $c_1 = 2c$.
\end{proof}

\subsection{Variance error}\label[appendix]{append:sec:gd:variance}

The GD variance analysis first appears in \citet[Lemma B.1]{zou2022risk}, which ultimately reduces to the variance error of ridge regression and is then a direct consequence of \citep{bartlett2020benign,tsigler2023benign}.
We include a formal statement and its proof as the following \Cref{lemma:gd:variance} for completeness.

\begin{lemma}[variance error]\label[lemma]{lemma:gd:variance}
Let $0< \eta \le n/\|\AB\|$.
Let $k$ be an index satisfying the conditions in \Cref{lemma:basic-concentration,lemma:tail-regularization}.
Then under the joint events of \Cref{lemma:basic-concentration,lemma:tail-regularization}, the matrix $\CB$ defined in \Cref{eq:bias-variance} is bounded by 
\begin{align*}
     \CB   \preceq \dfrac{64 c_1^3 }{n}\begin{bmatrix}[1.2]
\SigmaB_{\le k}^{-1} & \\
        & \dfrac{1}{\tilde\lambda^2}\dfrac{1}{n}\XB^\top_{>k}\XB_{>k}
    \end{bmatrix},\quad \text{where}\ \ \tilde\lambda:= \frac{1}{\eta t} + \frac{\sum_{i>k}\lambda_i}{n}.
\end{align*}
As a consequence, we have 
\begin{align*}
    \la \CB,\, \SigmaB\ra \le \frac{64 c_1^4}{n}\bigg( k   + \dfrac{1}{\tilde\lambda^2} \sum_{i>k}\lambda_i^2\bigg).
\end{align*}
\end{lemma}
\begin{proof}[Proof of \Cref{lemma:gd:variance}]
By \Cref{lemma:basic-algebra,lemma:GD-projector-decomp}, we have 
\[
\CB = \XB^\top \tilde\AB^{-2} \XB \preceq 4 \XB^\top \tilde\AB_{\ridge}^{-2}\XB,\quad \text{where}\ \ \tilde\AB_{\ridge}:= \AB + \frac{n}{\eta t} \IB = \XB_{\le k}\XB_{\le k}^\top + \AB_{k}.
\]
Using the matrix splitting notation, we have 
\begin{align*}
 \XB^\top \tilde\AB_{\ridge}^{-2} \XB 
 = \begin{bmatrix}[1.2]
     \XB_{\le k}^\top \\
     \XB_{>k}^\top 
 \end{bmatrix}
 \tilde\AB_{\ridge}^{-2}
 \begin{bmatrix}
     \XB_{\le k} & \XB_{>k}
 \end{bmatrix}
    =:  \begin{bmatrix}[1.2]
        \CB_{11} & \CB_{12} \\
        \CB_{12}^\top & \CB_{22}
    \end{bmatrix} 
    \preceq 2 \begin{bmatrix}[1.2]
        \CB_{11} & \\
        & \CB_{22}
    \end{bmatrix},
\end{align*}
where 
\begin{align*}
   \CB_{11} :=  \XB^\top_{\le k}\tilde\AB_{\ridge}^{-2}\XB_{\le k},\quad 
\CB_{22} :=\XB^\top_{>k}\tilde\AB_{\ridge}^{-2}\XB_{>k}.
\end{align*}
For the variance head $\CB_{11}$, we first use Woodbury's identity to show that 
\begin{align}
    \XB_{\le k}^\top \tilde\AB^{-1}_{\ridge}
    &= \XB_{\le k}^\top \big(\XB_{\le k}\XB_{\le k}^\top + \AB_k\big)^{-1}\notag  \\
    &= \XB_{\le k}^\top \big( \AB_k^{-1} - \AB_k^{-1} \XB_{\le k}(\IB + \XB^\top_{\le k} \AB_k^{-1} \XB_{\le k} )^{-1} \XB_{\le k}^\top \AB_k^{-1}\big) \notag \\
    &= \big( \IB - \XB_{\le k}^\top \AB_k^{-1} \XB_{\le k}(\IB + \XB^\top_{\le k} \AB_k^{-1} \XB_{\le k} )^{-1} \big)  \XB_{\le k}^\top \AB_k^{-1} \notag \\
    &= (\IB + \XB^\top_{\le k} \AB_k^{-1} \XB_{\le k} )^{-1}\XB_{\le k}^\top \AB_k^{-1}.\notag
\end{align}
This implies that
\begin{align*}
    \CB_{11} & =  (\IB+\XB_{\le k}^\top \AB_{k}^{-1}\XB_{\le k})^{-1} \XB_{\le k}^\top\AB_k^{-2} \XB_{\le k} (\IB+\XB_{\le k}^\top \AB_{k}^{-1}\XB_{\le k})^{-1} \\
    &\preceq \frac{1}{\lambda^2_{\min}(\AB_k)} (\IB+\XB_{\le k}^\top \AB_{k}^{-1}\XB_{\le k})^{-1} \XB_{\le k}^\top\XB_{\le k} (\IB+\XB_{\le k}^\top \AB_{k}^{-1}\XB_{\le k})^{-1} \\
    &= \frac{1}{\lambda^2_{\min}(\AB_k)} \SigmaB^{-\frac{1}{2}}_{\le k}(\SigmaB_{\le k}^{-1}+\ZB_{\le k}^\top \AB_{k}^{-1}\ZB_{\le k})^{-1} \ZB_{\le k}^\top\ZB_{\le k} (\SigmaB_{\le k}^{-1}+\ZB_{\le k}^\top \AB_{k}^{-1}\ZB_{\le k})^{-1} \SigmaB^{-\frac{1}{2}}_{\le k} \\
    &\preceq  \frac{2 n }{\lambda^2_{\min}(\AB_k)} \SigmaB^{-\frac{1}{2}}_{\le k}(\SigmaB_{\le k}^{-1}+\ZB_{\le k}^\top \AB_{k}^{-1}\ZB_{\le k})^{-2}  \SigmaB^{-\frac{1}{2}}_{\le k} \\
    &\preceq \frac{2n }{\lambda^2_{\min}(\AB_k)} \frac{1}{\lambda_{\min}^2(\SigmaB_{\le k}^{-1}+\ZB_{\le k}^\top \AB_{k}^{-1}\ZB_{\le k})}  \SigmaB_{\le k}^{-1} \\
     &\preceq \frac{2 n }{\lambda^2_{\min}(\AB_k)} \frac{1}{\lambda_{\min}^2(\ZB_{\le k}^\top \AB_{k}^{-1}\ZB_{\le k})}  \SigmaB_{\le k}^{-1} \\
      &\preceq \frac{2 n }{\lambda^2_{\min}(\AB_k)} \frac{\lambda^2_{\max}(\AB_k)}{\lambda_{\min}^2(\ZB_{\le k}^\top \ZB_{\le k})}  \SigmaB_{\le k}^{-1} \\
      &\preceq \frac{8}{n} \frac{\lambda^2_{\max}(\AB_k)}{\lambda_{\min}^2(\AB_{k})}  \SigmaB_{\le k}^{-1},
\end{align*}
where the second and the last inequalities are because 
\[\frac{n}{2} \le \lambda_{\min}(\ZB_{\le k}^\top \ZB_{\le k}) \le \lambda_{\max}(\ZB_{\le k}^\top \ZB_{\le k}) \le 2n \] 
by \Cref{lemma:basic-concentration}.
For the variance tail $\CB_{22}$, we have
\begin{align*}
    \CB_{22} &:=\XB^\top_{>k}\tilde\AB_{\ridge}^{-2}\XB_{>k} \\
    &\preceq \frac{1}{\lambda_{\min}^2(\tilde\AB_{\ridge})} \XB^\top_{>k}\XB_{>k} \\
    &\preceq  \frac{1}{\lambda_{\min}^2(\AB_k)} \XB^\top_{>k}\XB_{>k}  \\
    &= \frac{n}{\lambda_{\min}^2(\AB_k)}\frac{1}{n}\XB^\top_{>k}\XB_{>k}.
\end{align*}
Using \Cref{lemma:tail-regularization}, we have 
\begin{align*} 
\frac{\lambda_{\max}(\AB_k)}{\lambda_{\min}(\AB_k)} \le c_1^2,\quad 
    \lambda_{\min}(\AB_k) \ge \frac{1}{c_1} \bigg(\frac{n}{\eta t} + \sum_{i>k}\lambda_i\bigg)=\frac{n\tilde\lambda}{c_1}.
\end{align*}
Then we have
\begin{align*}
    \CB_{11} \preceq \frac{8c_1^3 }{n}\SigmaB_{\le k}^{-1},\quad 
    \CB_{22}
    \preceq \frac{c_1^2}{n\tilde\lambda^2}\frac{1}{n}\XB^\top_{>k}\XB_{>k}.
\end{align*}
Putting everything together, we have
\begin{align*}
    \CB\preceq 8 \begin{bmatrix}[1.2]
        \CB_{11} & \\
        & \CB_{22}
    \end{bmatrix}
    \preceq
    8\begin{bmatrix}[1.2]
\dfrac{8c_1^3 }{n}\SigmaB_{\le k}^{-1} & \\
        & \dfrac{c_1^2}{n\tilde\lambda^2}\dfrac{1}{n}\XB^\top_{>k}\XB_{>k}
    \end{bmatrix}
    \preceq \dfrac{64c_1^3 }{n}\begin{bmatrix}[1.2]
\SigmaB_{\le k}^{-1} & \\
        & \dfrac{1}{\tilde\lambda^2}\dfrac{1}{n}\XB^\top_{>k}\XB_{>k}
    \end{bmatrix}.
\end{align*}
As a consequence, we have
\begin{align*}
\la \CB, \SigmaB\ra 
  &\le  \frac{64 c_1^3}{n} \bigg(\la \SigmaB_{\le k}^{-1},\, \SigmaB_{\le k}\ra  + \dfrac{1}{\tilde\lambda^2}\bigg\la \dfrac{1}{n}\XB^\top_{>k}\XB_{>k},\, \SigmaB_{>k}\bigg\ra\bigg) \\
  &= \frac{64c_1^3}{n}\bigg( k   + \dfrac{1}{\tilde\lambda^2} \dfrac{\tr\big(\XB_{>k}\SigmaB_{>k}\XB^\top_{>k}\big)}{n}\bigg) \\
    &\le \frac{64 c_1^3}{n}\bigg( k   + \dfrac{c_1}{\tilde\lambda^2} \sum_{i>k}\lambda_i^2\bigg) \\
    &\le \frac{64c_1^4}{n}\bigg( k   + \dfrac{1}{\tilde\lambda^2} \sum_{i>k}\lambda_i^2\bigg),
\end{align*}
where the second inequality is by \Cref{lemma:basic-concentration}.
This completes our proof.
\end{proof}

\subsection{Bias error}\label[appendix]{append:sec:gd:bias}
Our bias analysis for GD is motivated by the bias analysis for ridge regression by \citet{tsigler2023benign}.
The core novelty is that, in suitable places in the arguments for obtaining the bias upper bound, the GD shrinkage matrix $\tilde\AB$ can be replaced by the ridge shrinkage matrix $\tilde\AB_{\ridge}$ using \Cref{lemma:GD-projector-decomp}, and this maintains the correct directions in the chains of inequalities. 
However, this approach only works for obtaining an upper bound for the bias error. Indeed, there are examples where GD is strictly better than ridge regression, so one cannot expect to prove a lower bound for GD using the same approach.

We compute $\BB$ using its equivalent formula in \Cref{lemma:basic-algebra}. 
Using the matrix splitting notation, we have
\begin{align*}
    \IB - \XB^\top \tilde\AB^{-1}\XB 
= \IB - \begin{bmatrix}[1.5]
        \XB_{\le k}^\top \\ \XB_{>k}^\top
    \end{bmatrix} \tilde\AB^{-1} \begin{bmatrix}
        \XB_{\le k} & \XB_{>k}
    \end{bmatrix} 
    = \begin{bmatrix}[1.5]
        \IB - \XB_{\le k}^\top \tilde\AB^{-1}\XB_{\le k} & -  \XB_{\le k}^\top \tilde\AB^{-1}\XB_{>k} \\
        - \XB_{>k}^\top\tilde\AB^{-1}  \XB_{\le k} &  \IB - \XB_{>k}^\top\tilde\AB^{-1} \XB_{>k}
    \end{bmatrix}.
\end{align*}
So we have
\begin{align*}
 \BB = \big(\IB-\XB^\top \tilde\AB^{-1} \XB\big)\SigmaB\big(\IB-\XB^\top \tilde\AB^{-1} \XB\big) =: \begin{bmatrix}[1.2]
        \BB_{11} & \BB_{12} \\
        \BB_{12}^\top & \BB_{22}
    \end{bmatrix}
    \preceq 2 \begin{bmatrix}[1.2]
        \BB_{11} &  \\
         & \BB_{22}
    \end{bmatrix},
\end{align*}
where the principal submatrices are
\begin{align}
    \BB_{11} &:=  \underbrace{\big(\IB - \XB_{\le k}^\top \tilde\AB^{-1}\XB_{\le k}\big) \SigmaB_{\le k}\big(\IB - \XB_{\le k}^\top \tilde\AB^{-1}\XB_{\le k}\big)}_{ \BB_{11}^{(1)}} +  \underbrace{\XB_{\le k}^\top \tilde\AB^{-1}\XB_{>k}\SigmaB_{>k}\XB_{>k}^\top\tilde\AB^{-1}  \XB_{\le k}}_{ \BB_{11}^{(2)}}, \label{eq:gd:bias:head-matrix} \\
    \BB_{22} &:= \underbrace{\XB_{>k}^\top\tilde\AB^{-1}  \XB_{\le k} \SigmaB_{\le k}  \XB_{\le k}^\top \tilde\AB^{-1}\XB_{>k} }_{ \BB_{22}^{(1)}} + \underbrace{\big( \IB - \XB_{>k}^\top\tilde\AB^{-1} \XB_{>k}\big) \SigmaB_{>k} \big(\IB - \XB_{>k}^\top\tilde\AB^{-1} \XB_{>k}\big)}_{ \BB_{22}^{(2)}}.\label{eq:gd:bias:tail-matrix}
\end{align}
By Cauchy–Schwarz inequality, we have
\begin{align*}
\la \BB, \, \wB^{* \otimes 2}\ra 
\le \bigg\la 2\begin{bmatrix}[1.2]
    \BB_{11} & \\
    & \BB_{22}
\end{bmatrix},\, 
2 \begin{bmatrix}[1.2]
    \wB_{\le k}^{*\otimes 2} & \\
        & \wB_{>k}^{* \otimes 2}
\end{bmatrix}
\bigg\ra 
= 4 \la \BB_{11},\, \wB_{\le k}^{*\otimes 2}\ra  + 4\la \BB_{22},\, \wB_{>k}^{* \otimes 2}\ra .
\end{align*}
The following \Cref{lemma:gd:bias-head,lemma:gd:bias-tail} provide upper bounds on the head and tail parts of the bias error, respectively, under the joint events of \Cref{lemma:basic-concentration,lemma:tail-regularization}.

\begin{lemma}[bias head]\label[lemma]{lemma:gd:bias-head}
Let $0< \eta \le n/\|\AB\|$.
Let $k$ be an index satisfying the conditions in \Cref{lemma:basic-concentration,lemma:tail-regularization}.
Then under the joint events of \Cref{lemma:basic-concentration,lemma:tail-regularization}, the matrix $\BB_{11}$ defined in \Cref{eq:gd:bias:head-matrix} is bounded by 
\begin{align*}
\BB_{11} \preceq \bigg(16c_1^2 + 32 c_1^5 \frac{1+c_2}{c_2^2}\bigg) \tilde\lambda^2\SigmaB^{-1}_{\le k},\quad \text{where}\ \ \tilde\lambda:=\frac{1}{\eta t} + \frac{\sum_{i>k}\lambda_i}{n}.
\end{align*}
As a consequence, we have 
\begin{align*}
\la \BB_{11},\, \wB^{*\otimes 2}_{\le k}\ra \le \bigg(16c_1^2 + 32 c_1^5 \frac{1+c_2}{c_2^2}\bigg) \tilde\lambda^2\big\|\wB^*_{\le k}\big\|^2_{\SigmaB^{-1}_{\le k}}.
\end{align*}
\end{lemma}
\begin{proof}[Proof of \Cref{lemma:gd:bias-head}]
We compute the first term in $\BB_{11}$ defined in \Cref{eq:gd:bias:head-matrix}:
\begin{align*}
    \BB_{11}^{(1)} := \big(\IB - \XB_{\le k}^\top \tilde\AB^{-1}\XB_{\le k}\big) \SigmaB_{\le k}\big(\IB - \XB_{\le k}^\top \tilde\AB^{-1}\XB_{\le k}\big).
\end{align*}
Recall that 
\(\tilde\AB = \XB_{\le k}\XB_{\le k}^\top + \tilde\AB_k.\)
By Woodbury's identity, we have 
\begin{align}
    \XB_{\le k}^\top \tilde\AB^{-1}
    &= \XB_{\le k}^\top \big(\XB_{\le k}\XB_{\le k}^\top + \tilde\AB_k\big)^{-1}\notag  \\
    &= \XB_{\le k}^\top \big( \tilde\AB_k^{-1} - \tilde\AB_k^{-1} \XB_{\le k}(\IB + \XB^\top_{\le k} \tilde\AB_k^{-1} \XB_{\le k} )^{-1} \XB_{\le k}^\top \tilde\AB_k^{-1}\big) \notag \\
    &= \big( \IB - \XB_{\le k}^\top \tilde\AB_k^{-1} \XB_{\le k}(\IB + \XB^\top_{\le k} \tilde\AB_k^{-1} \XB_{\le k} )^{-1} \big)  \XB_{\le k}^\top \tilde\AB_k^{-1} \notag \\
    &= (\IB + \XB^\top_{\le k} \tilde\AB_k^{-1} \XB_{\le k} )^{-1}\XB_{\le k}^\top \tilde\AB_k^{-1}. \notag %
\end{align}
This implies that 
\begin{align*}
    \IB -  \XB_{\le k}^\top \tilde\AB^{-1}\XB_{\le k} &= \IB -  (\IB + \XB^\top_{\le k} \tilde\AB_k^{-1} \XB_{\le k} )^{-1}\XB_{\le k}^\top \tilde\AB_k^{-1}\XB_{\le k} \\
    &= (\IB + \XB^\top_{\le k} \tilde\AB_k^{-1} \XB_{\le k} )^{-1}.
\end{align*}
So we have 
\begin{align*}
    \BB_{11}^{(1)} 
    &:= \big(\IB - \XB_{\le k}^\top \tilde\AB^{-1}\XB_{\le k}\big) \SigmaB_{\le k}\big(\IB - \XB_{\le k}^\top \tilde\AB^{-1}\XB_{\le k}\big) \\
    &= (\IB + \XB^\top_{\le k} \tilde\AB_k^{-1} \XB_{\le k} )^{-1} \SigmaB_{\le k} (\IB + \XB^\top_{\le k} \tilde\AB_k^{-1} \XB_{\le k} )^{-1} \\
    &= \SigmaB^{-\frac{1}{2}}_{\le k} \big(\SigmaB^{-1}_{\le k} + \ZB^\top_{\le k} \tilde\AB_k^{-1} \ZB_{\le k} \big)^{-2}\SigmaB^{-\frac{1}{2}}_{\le k} ,
\end{align*}
where the last equality follows from the convention of $\XB = \ZB\SigmaB^{\frac{1}{2}}$.
Consider the middle matrix. By Weyl's inequality, we have
\begin{align*}
    \lambda_{\min}\big( \SigmaB^{-1}_{\le k} + \ZB^\top_{\le k} \tilde\AB_k^{-1} \ZB_{\le k}\big)
    &\ge \lambda_{\min}\big(\ZB^\top_{\le k} \tilde\AB_k^{-1} \ZB_{\le k}\big) \\
    &\ge\frac{\lambda_{\min}\big(\ZB^\top_{\le k}  \ZB_{\le k}\big)}{\lambda_{\max}(\tilde\AB_k)} \\
    &\ge \frac{\lambda_{\min}(\ZB^\top_{\le k}\ZB_{\le k})}{ 2\lambda_{\max}(\AB_k) } \\
    &\ge \frac{ n}{ 4 \lambda_{\max}(\AB_k) },
\end{align*}
where the third inequality is because $\tilde\AB_k \preceq 2\AB_k$ by \Cref{lemma:GD-projector-decomp} and the last inequality is by \Cref{lemma:basic-concentration}. 
Then we get
\begin{align*}
    \BB_{11}^{(1)} 
    = \SigmaB^{-\frac{1}{2}}_{\le k} \big(\SigmaB^{-1}_{\le k} + \ZB^\top_{\le k} \tilde\AB_k^{-1} \ZB_{\le k} \big)^{-2}\SigmaB^{-\frac{1}{2}}_{\le k}
    \preceq \bigg(  \frac{4 \lambda_{\max}(\AB_k)  }{ n} \bigg)^2 \SigmaB^{-1}_{\le k}.
\end{align*}
Next, we compute the second term in $\BB_{11}$ defined in \Cref{eq:gd:bias:head-matrix}:
\begin{align*}
\BB_{11}^{(2)}
&:=\XB_{\le k}^\top \tilde\AB^{-1}\XB_{>k}\SigmaB_{>k}\XB_{>k}^\top\tilde\AB^{-1}  \XB_{\le k} \\
    &\preceq c_1 \bigg(n \lambda_{k+1}^2 + \sum_{i>k}\lambda_i^2\bigg) \XB_{\le k}^\top \tilde\AB^{-2}\XB_{\le k} \\
    &\preceq 4 c_1 \bigg(n \lambda_{k+1}^2 + \sum_{i>k}\lambda_i^2\bigg) \XB_{\le k}^\top \big(\XB_{\le k}\XB_{\le k}^\top + \AB_k \big)^{-2}\XB_{\le k},
\end{align*}
where the first inequality is by \Cref{lemma:basic-concentration} and
the second inequality is by \Cref{lemma:GD-projector-decomp}.
Note that the matrix in the right-hand side is exactly $\CB_{11}$ in the proof of \Cref{lemma:gd:variance}, from where we have
\begin{align*}
\XB_{\le k}^\top \big(\XB_{\le k}\XB_{\le k}^\top + \AB_k \big)^{-2}\XB_{\le k}
\preceq \frac{8}{n} \frac{\lambda_{\max}^2(\AB_k)}{\lambda_{\min}^2(\AB_k)}  \SigmaB^{-1}_{\le k}.
\end{align*}
So we have
\begin{align*}
    \BB_{11}^{(2)} 
    \preceq 32c_1 \frac{n \lambda_{k+1}^2 + \sum_{i>k}\lambda_i^2}{n}\frac{ \lambda_{\max}^2(\AB_k)}{ \lambda_{\min}^2(\AB_k)}\SigmaB^{-1}_{\le k}.
\end{align*}
Putting things together, we have 
\begin{align*}
    \BB_{11} &= \BB_{11}^{(1)} + \BB_{11}^{(2)}  \\
    &\preceq 16 \bigg(  \frac{\lambda_{\max}(\AB_k)  }{n} \bigg)^2 \SigmaB^{-1}_{\le k} + 32 c_1 \frac{n \lambda_{k+1}^2 + \sum_{i>k}\lambda_i^2}{n}\frac{ \lambda_{\max}^2(\AB_k)}{ \lambda_{\min}^2(\AB_k)}\SigmaB^{-1}_{\le k}. 
\end{align*}
Now let us consider an index $k$ that satisfies the condition in \Cref{lemma:tail-regularization}, then we have 
\begin{equation}\label{eq:tail-regu:Ak}
\frac{ \lambda_{\max}(\AB_{k})}{ \lambda_{\min}(\AB_{k})} \le c_1^2,\quad 
\lambda_{\max}(\AB_{k}) \le c_1\bigg(\frac{n}{\eta t} + \sum_{i>k}\lambda_i \bigg) = c_1 n \tilde\lambda.
\end{equation}
Furthermore, such an index $k$ satisfies
\begin{align}
    \frac{n \lambda_{k+1}^2 + \sum_{i>k}\lambda_i^2 }{n}
    &\le \lambda_{k+1} \bigg(\lambda_{k+1} + \frac{\sum_{i>k}\lambda_i}{n}\bigg) \notag \\
    &\le \frac{1}{c_2}\bigg(\frac{1}{c_2} + 1\bigg) \bigg(\frac{1}{\eta t} + \frac{\sum_{i>k}\lambda_i}{n}\bigg)^2 \notag \\
    &=\frac{1}{c_2}\bigg(\frac{1}{c_2} + 1\bigg) \tilde\lambda^2 , \label{eq:tail-regu:tail}
\end{align}
where the second inequality is because $k$ satisfies \Cref{lemma:tail-regularization}, which requires
\[
\frac{1}{\eta t} + \frac{\sum_{i>k}\lambda_i}{n}
\ge c_2 \lambda_{k+1}.
\]
Bringing this back, we get
\begin{align*}
& \BB_{11} \preceq \bigg(16c_1^2 + 32 c_1^5 \frac{1+c_2}{c_2^2}\bigg) \tilde\lambda^2\SigmaB^{-1}_{\le k}.
\end{align*}
This completes our proof.
\end{proof}

\begin{lemma}[bias tail]\label[lemma]{lemma:gd:bias-tail}
Let $0< \eta \le n/\|\AB\|$.
Let $k$ be an index satisfying the conditions in \Cref{lemma:basic-concentration,lemma:tail-regularization}.
Then under the joint events of \Cref{lemma:basic-concentration,lemma:tail-regularization}, the matrix $\BB_{22}$ defined in \Cref{eq:gd:bias:tail-matrix} is bounded by 
\begin{align*}
 \BB_{22} \preceq2\SigmaB_{>k} + \bigg( 32c_1^4 + 8c_1^3 \frac{1+c_2}{c_2^2} \bigg) \frac{1}{n}\XB_{>k}^\top \XB_{>k}.
\end{align*}
As a consequence, we have 
\begin{align*}
    \big\la \BB_{22},\, \wB^{*\otimes 2}_{>k} \big\ra \le \bigg(2+ 32c_1^5 + 8c_1^4 \frac{1+c_2}{c_2^2} \bigg) \big\| \wB^*_{>k}  \big\|^2_{\SigmaB_{>k}} .
\end{align*}
\end{lemma}
\begin{proof}[Proof of \Cref{lemma:gd:bias-tail}]
For the first term of $\BB_{22}$ defined in \Cref{eq:gd:bias:tail-matrix}, we have
\begin{align*}
\BB_{22}^{(1)}
    &:= \XB_{>k}^\top\tilde\AB^{-1}  \XB_{\le k} \SigmaB_{\le k}  \XB_{\le k}^\top \tilde\AB^{-1}\XB_{>k} \\
    &\preceq  \lambda_{\max}\big( \tilde\AB^{-1}  \XB_{\le k} \SigmaB_{\le k}  \XB_{\le k}^\top \tilde\AB^{-1}\big) \XB_{>k} ^\top \XB_{>k}.
\end{align*}
For the scalar factor, we have 
\begin{align*}
    &\lefteqn{ \lambda_{\max}\big( \tilde\AB^{-1}  \XB_{\le k} \SigmaB_{\le k}  \XB_{\le k}^\top \tilde\AB^{-1} \big)  } \\
    &=  \lambda_{\max} \big( \SigmaB_{\le k}^{\frac{1}{2}}  \XB_{\le k}^\top \tilde\AB^{-2} \XB_{\le k} \SigmaB_{\le k}^{\frac{1}{2}}  \big) \\
    &\preceq  4\lambda_{\max}\Big( \SigmaB_{\le k}^{\frac{1}{2}}  \XB_{\le k}^\top \big(\XB_{\le k}\XB_{\le k}^\top + \AB_k\big)^{-2} \XB_{\le k} \SigmaB_{\le k}^{\frac{1}{2}}  \Big),
\end{align*}
where the inequality is by \Cref{lemma:GD-projector-decomp} and Weyl's inequality.
Note that the middle matrix is exactly $\CB_{11}$ in the proof of \Cref{lemma:gd:variance}, from where we have 
\begin{align*}
\XB_{\le k}^\top \big(\XB_{\le k}\XB_{\le k}^\top + \AB_k \big)^{-2}\XB_{\le k} 
\preceq \frac{8}{n} \frac{\lambda_{\max}^2(\AB_k)}{\lambda_{\min}^2(\AB_k)}  \SigmaB^{-1}_{\le k}.
\end{align*}
Applying Weyl's inequality again, we get
\begin{align*}
    \lambda_{\max}\big( \tilde\AB^{-1}  \XB_{\le k} \SigmaB_{\le k}  \XB_{\le k}^\top \tilde\AB^{-1} \big) \le \frac{32}{n} \frac{\lambda_{\max}^2(\AB_k)}{\lambda_{\min}^2(\AB_k)} .
\end{align*}
Bringing this back, we have
\begin{align*}
  \BB_{22}^{(1)} \preceq \frac{32 \lambda^2_{\max}(\AB_k) }{\lambda^2_{\min}(\AB_k)} \frac{1}{n} \XB_{>k}^\top \XB_{>k} .
\end{align*}
For the second term of $\BB_{22}$ defined in \Cref{eq:gd:bias:tail-matrix}, we have
\begin{align*}
\BB^{(2)}_{22}
&:=   \big(\IB - \XB_{>k}^\top\tilde\AB^{-1} \XB_{>k}\big) \SigmaB_{>k} \big(\IB - \XB_{>k}^\top\tilde\AB^{-1} \XB_{>k} \big) \\
&\preceq 2 \SigmaB_{>k} + 2 \XB_{>k}^\top\tilde\AB^{-1} \XB_{>k} \SigmaB_{>k} \XB_{>k}^\top\tilde\AB^{-1} \XB_{>k}  \\
&\preceq 2 \SigmaB_{>k} + 2\lambda_{\max} \big( \tilde\AB^{-1} \XB_{>k} \SigmaB_{>k} \XB_{>k}^\top\tilde\AB^{-1}\big) \XB_{>k}^\top \XB_{>k}\\
&\preceq 2 \SigmaB_{>k} + 2\frac{\lambda_{\max} \big( \XB_{>k} \SigmaB_{>k} \XB_{>k}^\top\big)}{\lambda_{\min}^2(\tilde\AB )} \XB_{>k}^\top \XB_{>k},
\end{align*}
where the first inequality is by the Cauchy–Schwarz inequality.
By \Cref{lemma:basic-concentration}, we have
\begin{align*}
    \XB_{>k} \SigmaB_{>k} \XB_{>k}^\top 
    \preceq c_1 \bigg(n\lambda_{k+1}^2  +\sum_{i>k}\lambda_i^2\bigg)\IB_n.
\end{align*}
By \Cref{lemma:GD-projector}, we have
\begin{align*}
    & \tilde\AB 
    \succeq \frac{1}{2} \big(\XB_{\le k}\XB_{\le k}^\top  +\AB_{k} \big) \succeq \frac{1}{2}\lambda_{\min}(\AB_{k}) \IB_n.
\end{align*}
So we have 
\begin{align*}
     \BB^{(2)}_{22}
     &\preceq 2\SigmaB_{>k} +  8 c_1 \frac{n\lambda_{k+1}^2  +\sum_{i>k}\lambda_i^2}{\lambda^2_{\min}(\AB_{k}) } \XB_{>k}^\top \XB_{>k}\\
     &= 2\SigmaB_{>k} +  8 c_1 \frac{n^2\lambda_{k+1}^2  +n\sum_{i>k}\lambda_i^2}{\lambda^2_{\min}(\AB_{k}) } \frac{1}{n}\XB_{>k}^\top \XB_{>k}.
\end{align*}
Putting things together, we have shown that 
\begin{align*}
    \BB_{22}
    &= \BB_{22}^{(1)}+\BB_{22}^{(2)} \\
    &\preceq 2\SigmaB_{>k} +\bigg( \frac{32 \lambda^2_{\max}(\AB_k) }{ \lambda^2_{\min}(\AB_k)}
    +  8 c_1 \frac{n^2 \lambda_{k+1}^2  + n\sum_{i>k}\lambda_i^2}{\lambda^2_{\min}(\AB_{k}) } \bigg) \frac{1}{n}\XB_{>k}^\top \XB_{>k}.
\end{align*}
Similarly to the proof of \Cref{lemma:gd:bias-head}, we consider an index $k$ that satisfies the condition in \Cref{lemma:tail-regularization}, then we have \Cref{eq:tail-regu:Ak,eq:tail-regu:tail}, which implies
\begin{align*}
 \BB_{22} \preceq2\SigmaB_{>k} + \bigg( 32c_1^4 + 8c_1^3 \frac{1+c_2}{c_2^2} \bigg) \frac{1}{n}\XB_{>k}^\top \XB_{>k}.
\end{align*}
Finally, applying the above and \Cref{lemma:basic-concentration}, we have
\begin{align*}
    \big\la \BB_{22},\, \wB^{*\otimes 2}_{>k} \big\ra \le \bigg(2+ 32c_1^5 + 8c_1^4 \frac{1+c_2}{c_2^2} \bigg) \big\| \wB^*_{>k}  \big\|^2_{\SigmaB_{>k}} .
\end{align*}
This completes the proof.
\end{proof}

\subsection{\texorpdfstring{Proof of \Cref{thm:gd:ridge}}{Proof of the ridge-type bound}}

\begin{proof}[Proof of \Cref{thm:gd:ridge}]
We have decomposed the excess risk into bias and variance errors. 
Under the joint events in \Cref{lemma:basic-concentration,lemma:tail-regularization}, 
the variance error bound is given by \Cref{lemma:gd:variance} in \Cref{append:sec:gd:variance}, and the bias error bound is given by  \Cref{lemma:gd:bias-head,lemma:gd:bias-tail} in \Cref{append:sec:gd:bias}.
Finally, \Cref{lemma:basic-concentration,lemma:tail-regularization} in \Cref{append:sec:gd:lemmas} show that the joint events hold with probability at least $1-2\exp(-n/c_0)$ under \Cref{assum:upper-bound}.
Rescaling the constants properly, we have completed the proof for the promised high probability bound. 
\end{proof}

\section{\texorpdfstring{Proof of \Cref{thm:gd:lower-bound}}{Lower bound for GD}}\label[appendix]{append:sec:gd:lower-bound}

Following the notation setup in \Cref{append:sec:gd:ridge}, we have the GD formula \Cref{eq:gd:decompose}, from where we compute the excess risk as
\begin{align*}
 \Ebb_{\epsilonB}  \excessRisk(\wB_t) &= \Ebb_{\epsilonB} \|\wB_t - \wB^*\|^2_{\SigmaB} \\ 
    &= \big\|(\IB-\eta \hat\SigmaB)^t \wB^* \big\|^2_{\SigmaB} + \frac{1}{n^2} \Ebb_{\epsilonB} \big\|\big(\IB- (\IB - \eta \hat\SigmaB )^{t}\big) \hat\SigmaB^{-1}\XB^\top \epsilonB\big\|^2_{\SigmaB} \\
    &\ge  \big\|(\IB-\eta \hat\SigmaB)^t\wB^* \big\|^2_{\SigmaB} + \frac{1}{n^2} \big\la \XB\hat\SigmaB^{-1}\big(\IB- (\IB - \eta \hat\SigmaB )^{t}\big) \SigmaB\big(\IB- (\IB - \eta \hat\SigmaB )^{t}\big) \hat\SigmaB^{-1}\XB^\top,\, \sigma^2\IB \big\ra \\
    &= \big\la (\IB-\eta \hat\SigmaB)^t\SigmaB(\IB-\eta \hat\SigmaB)^t,\, \wB^{*\otimes 2} \big\ra + \sigma^2 \bigg\la  \frac{1}{n}\big(\IB- (\IB - \eta \hat\SigmaB )^{t}\big)\hat\SigmaB^{-1}\big(\IB- (\IB - \eta \hat\SigmaB )^{t}\big)  ,\, \SigmaB \bigg\ra, 
\end{align*}
where the second and third lines are because the conditional noise $\epsilonB | \XB$ is mean zero and has a variance bounded from below by \Cref{assum:lower-bound:item:noise}.
Recall the definitions of bias and variance matrices in \Cref{eq:bias-variance}. We then have
\begin{align*}
    & \Ebb_{\epsilonB}  \excessRisk(\wB_t) \ge \la \BB,\, \wB^{*\otimes 2}\ra + \sigma^2 \la \CB,\, \SigmaB\ra \\ 
  \Rightarrow\quad   &\Ebb  \excessRisk(\wB_t) \ge \la \Ebb \BB,\, \wB^{*\otimes 2}\ra + \sigma^2 \la \Ebb \CB,\, \SigmaB\ra .
\end{align*}
Also recall the equivalent formulas for the bias and variance matrices in \Cref{lemma:basic-algebra}.

\begin{lemma}[a variance lower bound for GD]\label[lemma]{lemma:gd:lower-bound:variance}
Let $0< \eta \le n/\|\AB\|$.
Under \Cref{assum:lower-bound:item:x}, there exist $c_1, c_2, c_3 >1$ that only depend on $\sigma_x^2$ for which the following holds. 
Define
\begin{align*}
k^* := \min\bigg\{k: \frac{1}{\eta t}+ \frac{\sum_{i>k}\lambda_i}{n} \ge  c_2 \lambda_{k+1}\bigg\},\quad 
    D:= k^* + \frac{1}{\tilde\lambda^2}\sum_{i>k^*}\lambda_i^2,\quad \tilde\lambda:=\frac{1}{\eta t}+ \frac{\sum_{i>k^*}\lambda_i}{n}.
\end{align*}
Then we have 
\begin{align*}
     \Ebb \la \CB,\, \SigmaB\ra  \ge \frac{1}{c_1} \min\bigg\{\frac{D}{n},\, 1\bigg\}.
\end{align*}
\end{lemma}
\begin{proof}[Proof of \Cref{lemma:gd:lower-bound:variance}]
The assumption on $\eta$ enables \Cref{lemma:GD-projector} in \Cref{append:sec:gd:lemmas} with high probability, which implies 
\[
\tilde\AB \preceq \AB + \frac{2n}{\eta t}\IB \quad 
\Rightarrow \quad 
\tilde\AB^2 \preceq \bigg(\AB + \frac{2n}{\eta t} \IB\bigg)^2.
\]
So for $\CB$ in \Cref{eq:bias-variance}, with its equivalent formula in \Cref{lemma:basic-algebra}, we have
\begin{align*}
   & \CB = \XB^\top \tilde\AB^{-2}\XB \succeq \XB^\top \bigg(\AB + \frac{2n}{\eta t} \IB\bigg)^{-2} \XB  \\
   \Rightarrow \quad & \la \CB,\, \SigmaB\ra 
   \ge \bigg\la \XB^\top \bigg(\AB + \frac{2n}{\eta t} \IB\bigg)^{-2}\XB ,\, \SigmaB\bigg\ra, 
\end{align*}
where the right-hand side is the variance error of ridge regression \Cref{eq:ridge} with $\lambda = 2/(\eta t)$.
So the lower bound is then a consequence of \citet[Theorem 4]{bartlett2020benign} or \citet[Theorem 2]{tsigler2023benign}
(see also \Cref{thm:ridge}).
Finally, a high probability lower bound implies an expectation lower bound, as the variance error is a non-negative random variable.
\end{proof}

\begin{lemma}[a bias lower bound for GD]\label[lemma]{lemma:gd:lower-bound:bias}
Let $0< \eta \le n /\|\AB\|$.
Under \Cref{assum:lower-bound:item:x,assum:lower-bound:item:x-symmetry}, there exist $c_1, c_2, c_3>1$ that only depend on $\sigma_x^2$ for which the following holds.
Define
\begin{align*}
    \ell^* := \min\bigg\{k: \frac{\sum_{i>k}\lambda_i}{n} \ge c_2 \lambda_{k+1}\bigg\}.
\end{align*}
Assume that $n\ge c_3$.
Then we have
\begin{align*}
    \Ebb \la \BB, \, \wB^{*\otimes 2}\ra 
    \ge 
\frac{1}{c_1} \bigg(   \bigg(\frac{\sum_{i>\ell^*}\lambda_i}{n}\bigg)^2 \|\wB^*\|^2_{\SigmaB_{0:\ell^*}^{-1}} + \|\wB^*\|^2_{\SigmaB_{\ell^*:\infty}} \bigg) .
\end{align*}
\end{lemma}
\begin{proof}[Proof of \Cref{lemma:gd:lower-bound:bias}]
Our proof idea is adapted from the proof of the ridge regression bias lower bound in \citet{tsigler2023benign}. 
However, the reduction idea (which we used in our proof of \Cref{thm:gd:ridge}) no longer works, as there exist examples such that GD achieves a polynomially better bound than ridge. 
Instead, we seek to show that the GD bias error is lower bounded by the OLS (that corresponds to ridge regression with $\lambda=0$) bias error.

Recall the formula of $\BB$ in \Cref{lemma:basic-algebra}.
First, for $i\ne j$, we have 
\begin{align*}
    \Ebb\BB_{i j} = 0.
\end{align*}
This is because the distribution of $\xB$ is symmetric by \Cref{assum:lower-bound:item:x-symmetry} (see Lemma C.6 in \citet{zou2021benefits}).
We next compute $\BB_{ii}$ following the arguments in \citet[Lemma C.7]{zou2021benefits}, which is ultimately adapted from \citet{tsigler2023benign}.
Let $d$ be the dimension of $\Hbb$ (where we allow $d=\infty$).
We write 
\[
\XB = \begin{bmatrix}
    \zB_1 & \dots & \zB_d
\end{bmatrix}
\SigmaB^{\frac{1}{2}}
= \begin{bmatrix}
    \zB_1 & \dots & \zB_d
\end{bmatrix}
\begin{bmatrix}
    \lambda_1^{\frac{1}{2}} & & \\
    & \ddots & \\
    & & \lambda_p^{\frac{1}{2}}
\end{bmatrix},
\]
where $\zB_i$ for $i=1,\dots,d$ are independent random vectors with independent $\sigma_x^2$-subgaussian entries with unit variance by \Cref{assum:lower-bound:item:x}.
Let $(\eB_i)_{i=1}^d$ be the canonical basis for $\Hbb$, in which we assume $\SigmaB$ is diagonal, without loss of generality.
Then we have 
\begin{align*}
    \XB \eB_i = \lambda_i^{\frac{1}{2}} \zB_i,\quad \XB\SigmaB\eB_i = \lambda_i^{\frac{3}{2}} \zB_i,\quad 
    \XB\SigmaB\XB^\top = \sum_{j\ge 1} \lambda_j^2 \zB_j\zB_j^\top.
\end{align*}
We thus have
\begin{align*}
    \BB_{ii} 
    &= \eB_i^\top \BB \eB_i \\ 
    &= \eB_i^\top (\IB-\XB^\top \tilde\AB^{-1}\XB)\SigmaB(\IB-\XB^\top \tilde\AB^{-1}\XB) \eB_i \\
    &= \lambda_i - 2\lambda_i^{2} \zB_i^\top \tilde\AB^{-1} \zB_i + \lambda_i \zB_i^\top \tilde\AB^{-1} \bigg(\sum_{j\ge 1} \lambda_j^2 \zB_j\zB_j^\top \bigg)\tilde\AB^{-1} \zB_i \\
    &\ge  \lambda_i - 2\lambda_i^{2} \zB_i^\top \tilde\AB^{-1} \zB_i + \lambda_i \zB_i^\top \tilde\AB^{-1} \bigg( \lambda_i^2 \zB_i\zB_i^\top \bigg)\tilde\AB^{-1} \zB_i \\
    &= \lambda_i \big(1 - \lambda_i \zB_i^\top \tilde\AB^{-1} \zB_i \big)^2.
\end{align*}
Define
\begin{align*}
    \tilde\AB_{-i} := \tilde\AB - \lambda_i \zB_i \zB_i^\top.
\end{align*}
By Woodbury's identity, we have 
\begin{align*}
    \zB_i^\top \tilde\AB^{-1} 
    &= \zB_i^\top (\lambda_i \zB_i\zB_i^\top + \tilde\AB_{-i})^{-1} \\
    &= \zB_i^\top \big( \tilde\AB_{-i}^{-1} - \lambda_i\tilde\AB_{-i}^{-1}\zB_i(1+\lambda_i \zB_i^\top \tilde\AB_{-i}^{-1} \zB_i)^{-1}\zB_i^\top \tilde\AB_{-i}^{-1}  \big) \\ 
    &= \big(1- \lambda_i\zB_i^\top \tilde\AB_{-i}^{-1}\zB_i(1+\lambda_i \zB_i^\top \tilde\AB_{-i}^{-1} \zB_i)^{-1} \big)\zB_i^\top \tilde\AB_{-i}^{-1}   \\
    &= \frac{\zB_i^\top\tilde\AB_{-i}^{-1} }{1+\lambda_i \zB_i^\top \tilde\AB_{-i}^{-1} \zB_i}.
\end{align*}
Then we have
\begin{align*}
    \BB_{ii} &\ge \lambda_i \big(1 - \lambda_i \zB_i^\top \tilde\AB^{-1} \zB_i \big)^2 
    = \lambda_i \bigg( 1 - \frac{\lambda_i \zB_i^\top\tilde\AB_{-i}^{-1} \zB_i }{1+\lambda_i \zB_i^\top \tilde\AB_{-i}^{-1} \zB_i}\bigg)^2 \\
    &= \frac{\lambda_i}{\big(1+\lambda_i \zB_i^\top \tilde\AB_{-i}^{-1} \zB_i\big)^2}
    \ge \frac{\lambda_i}{\big(1+\lambda_i \|\zB_i\|^2 / \lambda_{\min}(\tilde\AB_{-i})\big)^2}.
\end{align*}
By \Cref{lemma:GD-projector} in \Cref{append:sec:gd:lemmas}, we have $\tilde\AB\succeq \AB$, which implies
\begin{align*}
    \tilde\AB_{-i} \succeq \AB - \lambda_i \zB_i \zB_i^\top = \sum_{j\ne i} \lambda_j \zB_j \zB_j^\top
    \succeq 
    \begin{dcases}
        \sum_{j>\ell^*}\lambda_j\zB_j\zB_j^\top & i \le \ell^*;\\
    \lambda_i \zB_1\zB_1^\top + \sum_{j>\ell^*, j\ne i}\lambda_j\zB_j\zB_j^\top & i > \ell^*.
    \end{dcases}  
\end{align*}
Notice that the right-hand side follows the same distribution as 
\(
\XB_{\ell^*:\infty}\XB_{\ell^*:\infty}^\top
\)
since $\zB_1$ follows the same distribution as $\zB_i$ (this argument first appears in the proof of Lemma 17 in \citet{tsigler2023benign}).
Then we can apply \Cref{lemma:tail-regularization} in \Cref{append:sec:gd:lemmas} with $t=\infty$ and $k=\ell^*$ to obtain that:
for each $i$, with probability at least $1-\exp(-n/c_0)$,
\begin{align*}
    \tilde\AB_{-i}
    \succeq \bigg(\frac{1}{c_1} \sum_{j>\ell^*}\lambda_j \bigg)\IB_n.
\end{align*}
Moreover, by Hoeffding's inequality, we have: 
for each $i$, with probability at least $1-\exp(-n/c_0)$,
\[
\|\zB_i\|^2 \le c_1 n.
\]
For each $i$, under the joint of these two events, we have
\begin{align*}
    \BB_{ii} \ge \frac{\lambda_i}{\big(1+\lambda_i \|\zB_i\|^2 / \lambda_{\min}(\tilde\AB_{-i})\big)^2}
    \ge \frac{\lambda_i}{\bigg(1+\frac{c_1^2\lambda_i n}{ \sum_{j>\ell^*}\lambda_j }\bigg)^2}
    \ge \frac{1}{c_1'} \min\bigg\{ \bigg(\frac{\sum_{j>\ell^*}\lambda_j}{n}\bigg)^2 \lambda_i^{-1},\, \lambda_i \bigg\},
\end{align*}
for some $c_1'$.
Let $n \ge \ln(2) c_0$ so that the joint of the two events happens with probability at least $1/2$, then the high probability lower bound implies an expectation lower bound as the random variable is non-negative, 
\begin{align*}
    \Ebb \BB_{ii} \ge \frac{1}{2c_1'} \min\bigg\{ \bigg(\frac{\sum_{j>\ell^*}\lambda_j}{n}\bigg)^2 \lambda_i^{-1},\, \lambda_i \bigg\}.
\end{align*}
Putting things together, we have 
\begin{align*}
    \Ebb\la \BB, \, \wB^{*\otimes 2}\ra 
    &= \sum_i ( \Ebb\BB_{ii})  (\wB^*_i)^2 \\
    &\ge \frac{1}{2c_1'}\sum_i \min\bigg\{ \bigg(\frac{\sum_{j>\ell^*}\lambda_j}{n}\bigg)^2 \lambda_i^{-1},\, \lambda_i \bigg\}(\wB^*_i)^2 \\
    &= \frac{1}{2c_1'}\bigg(\sum_{i\le \ell^*} \bigg(\frac{\sum_{j>\ell^*}\lambda_j}{n}\bigg)^2 \lambda_i^{-1}(\wB^*_i)^2 + \sum_{i>\ell^*}\lambda_i (\wB^*_i)^2 \bigg) \\
    &= \frac{1}{2c_1'}\bigg(\bigg(\frac{\sum_{j>\ell^*}\lambda_j}{n}\bigg)^2\|\wB^*_{0:\ell^*}\|^2_{\SigmaB^{-1}_{0:\ell^*}} + \|\wB^*_{\ell^*:\infty}\|^2_{\SigmaB_{\ell^*:\infty}}\bigg),
\end{align*}
where the first inequality is because the off-diagonal entries are zero, and the second inequality is because of the lower bound on the diagonal entries we have established.
This completes our proof.
\end{proof}

\begin{proof}[Proof of \Cref{thm:gd:lower-bound}]
It follows from \Cref{lemma:gd:lower-bound:variance,lemma:gd:lower-bound:bias}.
\end{proof}

\section{\texorpdfstring{Proof of \Cref{thm:gd-vs-sgd:heavy-tail}}{GD does not dominate SGD}}\label[appendix]{append:sec:gd-sgd}
\begin{proof}[Proof of \Cref{thm:gd-vs-sgd:heavy-tail}]
We first compute a lower bound for GD using \Cref{thm:gd:lower-bound}. It is clear that $\ell^*=1$ and that 
\begin{align*}
    \frac{\sum_{i>\ell^*}\lambda_i }{n} \eqsim \frac{1}{n} > \lambda_{\ell^*+1}=\frac{1}{d}.
\end{align*}
Moreover, the expected excess risk for GD is lower bounded by its bias error, which is 
\begin{align*}
    \bigg(\frac{\sum_{i>\ell^*}\lambda_i}{n}\bigg)^2 \|\wB^*_{0:\ell^*}\|^2_{\SigmaB_{0:\ell^*}^{-1}} + \|\wB^*_{\ell^*:\infty}\|^2_{\SigmaB_{\ell^*:\infty}}
    &= \bigg(\frac{\sum_{i>\ell^*}\lambda_i}{n}\bigg)^2 \|\wB^*_{0:\ell^*}\|^2_{\SigmaB_{0:\ell^*}^{-1}} \\
    &\eqsim \frac{1}{n^2} \|\wB^*\|^2_{\SigmaB_{0:1}^{-1}}
    \eqsim \frac{1}{n^2} n^{1.8} \eqsim n^{-0.2}.
\end{align*}
We then compute an upper bound for SGD using \Cref{thm:sgd}.
We choose $\eta=1/(4\tr(\SigmaB))\eqsim 1$, then $k^*=1$. So the effective dimension is 
\begin{align*}
    D := k^* + (\eta N)^2 \sum_{i>k^*}\lambda_i^2 
    \eqsim 1+ \bigg(\frac{n}{\log n}\bigg)^2 \frac{1}{d}\eqsim 1,
\end{align*}
where the last equality is because $d\ge n^2$.
Therefore, the sum of the effective variance and variance errors is 
\begin{align*}
    \big(\sigma^2 + \|\wB^*\|_{\SigmaB}^2\big)\frac{D}{N} \eqsim \frac{1}{N} \eqsim \frac{\log n}{n}.
\end{align*}
Moreover, the effective bias error is 
\begin{align*}
\bigg\|\prod_{t=1}^{n}\big(\IB-\eta_t\SigmaB\big)\wB^*\bigg\|^2_{\SigmaB}
    &\le \|(\IB-\eta \SigmaB)^N \wB^*\|^2_{\SigmaB}\\
    &\le  \exp(-2\eta N \lambda_1)\lambda_1(\wB^*_1)^2\\
    &=  \exp(-2\eta n^{0.1}/\log(n))\\
    &=\Ocal(1/n).
\end{align*}
So the expected excess risk for SGD is upper bounded by $\Ocal(\log (n) / n)$. 
We complete the proof.
\end{proof}

\section{\texorpdfstring{Proof of \Cref{thm:gd:exp}}{SGD-type bound for GD}}\label[appendix]{append:sec:gd:exp}
Following the notation in \Cref{append:sec:gd:ridge}, we have 
\begin{align*}
    \Ebb_{\epsilonB}\excessRisk(\wB_t) \le 2\underbrace{\la \BB,\, \wB^{*\otimes 2}\ra}_{\bias} + 2\underbrace{\sigma^2\la \CB,\, \SigmaB\ra}_{\variance},
\end{align*}
where $\BB$ and $\CB$ are defined in \Cref{eq:bias-variance}.
For the variance error, we use the bound derived in \Cref{lemma:gd:variance} in \Cref{append:sec:gd:variance}.
The main results in this part are to derive an alternative bound for the bias error for stopping time $t=\bigO( n)$.

Without loss of generality, assume that $\SigmaB$ is diagonal. 
Recall the notation that 
\[
\SigmaB_{0:k} = \begin{bmatrix}
    \SigmaB_{\le k} & \\
    & 0
\end{bmatrix},\quad 
\SigmaB_{k:\infty} = \begin{bmatrix}
    0 & \\
    & \SigmaB_{>k}
\end{bmatrix},\quad 
\hat\SigmaB = \frac{1}{n}\XB^\top \XB.
\]
Define 
\[
\QB := (\IB-\eta\hat \SigmaB)^{t/2} - (\IB - \eta \SigmaB_{0:k})^{t/2}.
\]
By basic linear algebra, we know that $\QB$ is symmetric and that
\[
\QB = \eta \sum_{j=0}^{t/2-1}(\IB - \eta\hat \SigmaB)^{t/2-1-j}(\SigmaB_{0: k}-\hat \SigmaB)(\IB-\eta \SigmaB_{0: k})^j. %
\]
By the definition of $\QB$, we have
\begin{align*}
 (\IB-\eta \hat\SigmaB)^t = (\IB-\eta\hat\SigmaB)^{t/2}(\IB-\eta \SigmaB_{0: k})^{t/2} + (\IB-\eta\hat\SigmaB)^{t/2}\QB,
\end{align*}
By Cauchy–Schwarz inequality, we have 
\begin{align}
\la \BB,\, \wB^{*\otimes 2}\ra
&=    \big\|\SigmaB^{1/2}(\IB - \eta \hat \SigmaB)^t \wB^*\big\|^2 \notag  \\ 
&\le  2\underbrace{\big\|\SigmaB^{1/2}(\IB - \eta  \hat \SigmaB)^{t/2}(\IB - \eta  \SigmaB_{0: k})^{t/2} \wB^* \big\|^2}_{\effBias} +  2\underbrace{\big\| \SigmaB^{1/2}(\IB - \eta \hat \SigmaB)^{t/2}\QB \wB^* \big\|^2}_{\effVar}. \label{eq:gd:exp:bias-decomposition}
\end{align}
We derive bounds on the effective bias and effective variance errors in \Cref{append:sec:gd:exp:effective-bias,append:sec:gd:exp:effective-variance}, respectively. Before that, we summarize a couple of basic concentration lemmas in the next subsection.

\subsection{Additional concentration lemmas}\label[appendix]{append:sec:gd:exp:lemmas}

In this part, we provide two additional concentration lemmas that will be useful for our analysis.
\begin{lemma}[basic concentration]\label[lemma]{lemma:head-and-stepsize-concentration}
Under \Cref{assum:upper-bound:item:x}, there exist $c_0, c_1\ge 1$ only depending on $\sigma_x^2$ such that the joint of the following events holds with probability at least $1-\exp(-n/c_0)$:
\begin{gather*}
   \frac{1}{n} \big\|\XB\XB^\top\big\| \le 2\tr(\SigmaB), \\ 
    \bigg\|\frac{1}{n} \SigmaB_{0:k}^{-\frac{1}{2}}\XB_{\le k}^\top \XB_{\le k}\SigmaB^{-\frac{1}{2}}_{0:k}- \IB_k \bigg\| \le c_1 \sqrt{\frac{k}{n}}.
\end{gather*}
\end{lemma}
\begin{proof}[Proof of \Cref{lemma:head-and-stepsize-concentration}]
For the first claim, notice that
\begin{align*}
  \frac{1}{n} \|\XB\XB^\top\| 
  \le \frac{1}{n} \tr(\XB\XB^\top)
  =  \frac{1}{n}\sum_{i=1}^n \|\xB_i\|^2,
\end{align*}
which is the sum of $n$ independent $\sigma_x^2\tr(\SigmaB)$-subexponential random variables after centering at $\tr(\SigmaB)$ by \Cref{assum:upper-bound:item:x}.
Then by Bernstein's inequality \citep[Corollary 2.9.2]{vershynin2026high}, we have
\begin{align*}
    \Pr\bigg( \frac{1}{n}\sum_{i=1}^n \|\xB_i\|^2 - \tr(\SigmaB) > t \bigg) \le \exp\bigg(-\frac{n}{c_0}\min\bigg\{\frac{t^2}{\tr(\SigmaB)^2}, \, \frac{t}{\tr(\SigmaB)}\bigg\} \bigg)
\end{align*}
for some $c_0\ge 1$ that only depends on $\sigma_x^2$.
Setting $t=\tr(\SigmaB)$
verifies the first claim. 

The second claim is a classical matrix concentration bound \citep[Theorem 4.4.3]{vershynin2026high}.
\end{proof}

The following lemma is easy to prove when $\XB$ is exactly Gaussian, using rotation invariance. 
Below, we prove the results for subgaussian cases (under \Cref{assum:upper-bound:item:x}) through a more involved approach.
\begin{lemma}[Fourth-moment concentration]\label[lemma]{lemma:fourth-moment-concentration}
For any fixed unit vector $\uB\in \Rbb^n$, under \Cref{assum:upper-bound:item:x}, there exists a constant $c_1\ge 1$ that only depends on $\sigma_x^2$ such that
\begin{align*}
\Pr \Bigg(    \uB^\top (\XB \XB^\top)^2 \uB > c_1 (n+t) \bigg(t\lambda_{1}^2 + \sum_{i}\lambda_i^2\bigg) + c_1 \bigg(t\lambda_1 + \sum_i \lambda_i\bigg)^2 \Bigg) \le \exp( - t ),\quad t\ge 4.
\end{align*}
\end{lemma}
\begin{proof}[Proof of \Cref{lemma:fourth-moment-concentration}]
    We write 
    \begin{align*}
        \uB = \begin{bmatrix}
            u_1 \\
            \vdots \\
            u_n
        \end{bmatrix},\quad 
        \XB  = \begin{bmatrix}
            \zB_1 & \zB_2 & \dots
        \end{bmatrix}\SigmaB^{\frac{1}{2}},
    \end{align*}
where $(\zB_i)_{i\ge 1}$ are independent, $n$-dimensional, $\sigma_x^2$-subgaussian random vectors with unit variance and independent entries.
Then 
\begin{align*}
    \XB\XB^\top \uB = \begin{bmatrix}
            \vdots \\
            \xB_i^\top \sum_{j = 1}^n \xB_j u_j \\
            \vdots
        \end{bmatrix}_{i=1,\dots, n} 
        = \underbrace{\begin{bmatrix}
            \vdots \\
            \xB_i^\top \xB_i u_i \\
            \vdots
        \end{bmatrix}_{i=1,\dots, n}}_{=: \dB}  + 
        \underbrace{\begin{bmatrix}
            \vdots \\
            \xB_i^\top \sum_{j \ne i} \xB_j u_j \\
            \vdots
        \end{bmatrix}_{i=1,\dots, n}}_{=:\oB}.
\end{align*}
We control the $p$-th moment of $\dB$ and $\oB$ separately.
We allow constants (e.g., $c_1$) to depend on $\sigma_x^2$.

\emph{The diagonal component.}~
For each coordinate of $\dB$, we can write
\[\xB_i^\top \xB_i = \sum_{j}\lambda_j z_{ij}^2,\]
where $(z_{ij})_{j\ge 1}$ are independent, $\sigma_x^2$-subgaussian random variables with unit variance. 
Then by the subexponential Khintchine inequality \citep[Exercise 2.46]{vershynin2026high}, for all $p\ge 2$, we have 
\begin{align*}
 \bigg\|\xB_i^\top \xB_i - \sum_{j}\lambda_j \bigg\|_{L^p} \le c_1 \bigg(\sqrt{p \sum_{j}\lambda_j^2} + p\lambda_1 \bigg) \quad 
   \Rightarrow\quad  \big\|\xB_i^\top \xB_i \big\|_{L^p} \le c_2 \bigg(\sum_{j}\lambda_j + p\lambda_1 \bigg),
\end{align*}
where $c_1, c_2>1$ are constants.
Since $\uB$ is a unit vector and $\|\cdot \|_{L^{p/2}}$ is a norm for $p\ge 2$, we have 
\begin{align*}
   \|\dB\|_{L^p}^2 &= \big(\Ebb (\|\dB\|^2)^{p/2}\big)^{2/p} 
   = \big\| \|\dB\|^2 \big\|_{L^{p/2}} 
   = \big\| \sum_{i} u_i^2 (\xB_i^\top \xB_i)^2 \big\|_{L^{p/2}} \\ 
   &\le \sum_{i} u_i^2  \big\| (\xB_i^\top \xB_i)^2 \big\|_{L^{p/2}}  
   = \sum_{i} u_i^2  \|\xB_i^\top \xB_i\|^2_{L^{p}} \\ 
   & \le c_2^2 \bigg(\sum_{j}\lambda_j + p\lambda_1 \bigg)^2.
\end{align*}

\emph{The off-diagonal component.}~
We use the decoupling method \citep[Section 6.1]{vershynin2026high} to control $\oB$. 
Specifically, for $i,j=1,\dots,n$, define vector-valued bilinear form $h_{ij}$ as 
\begin{align*}
    h_{ij}(\aB, \bB):= 
        \aB^\top \bB u_j\eB_i .
\end{align*}
Then we can write 
\[\oB = \sum_{i\ne j} h_{ij}(\xB_i, \xB_j).
\]
By vector decoupling \citep[Exercise 6.3]{vershynin2026high}, for all convex function $F$, we have 
\[
\Ebb F(\oB) \le  \Ebb F (4 \oB'),\quad  \oB':= \sum_{i, j} h_{ij}(\xB_i', \xB_j),
\]
where $(\xB'_i)_{i=1}^n$ are independent copies of $(\xB_i)_{i=1}^n$. 
Since $\|\cdot\|_{L^p}$ is a convex function, we have
\[
\|\oB\|_{L^p} \le  \|4\oB'\|_{L^p} =  4 \|\oB'\|_{L^p}.
\]
To bound the right-hand side, notice that 
\begin{align*}
    \oB' = \begin{bmatrix}
    \vdots \\
    \xB_i^{'\top} \sum_{j=1}^n\xB_j u_j \\
    \vdots 
\end{bmatrix}_{i=1,\dots,n}= \begin{bmatrix}
    \vdots \\
    \xB_i^{'\top}\XB^\top \uB \\
    \vdots 
\end{bmatrix}_{i=1,\dots,n}.
\end{align*}
Conditional on $\XB$, coordinates of $\oB'$ are independent, $\big\|\SigmaB^{1/2}\XB^\top \uB\big\|^2$-subgaussian, so $\|\oB'\|^2$ is the sum of n independent $\sigma_x^2 \big\|\SigmaB^{1/2}\XB^\top \uB\big\|^2$-subexponential random variables, for which by Subexponential Khintchine inequality we have, for $p\ge 4$,
\begin{align*}
   \big\| \|\oB'\|^2 \big\|_{L^{p/2}(\XB')} 
   \le c_1 (n+p)\big\|\SigmaB^{1/2}\XB^\top \uB\big\|^2,
\end{align*}
where $c_1 > 1$ is a constant.
We write the right-hand side as
\begin{align*}
    \big\|\SigmaB^{1/2}\XB^\top \uB\big\|^2 = \uB^\top \XB\SigmaB\XB^\top \uB = \uB^\top \ZB\SigmaB^2\ZB \uB = \sum_{i}\lambda_i^2 (\zB_i^\top \uB)^2,
\end{align*}
which is the sum of independent $\sigma_x^2 \lambda_i^2$-subexponential random variables (with expectation $\lambda_i^2$ respectively). 
Using Subexponential Khintchine inequality again, we have 
\begin{align*}
    \Big\| \big\|\SigmaB^{1/2}\XB^\top \uB\big\|^2\Big\|_{L^{p/2}(\XB)} \le c_1 \bigg(p\lambda_1^2 + \sum_{i}\lambda_i^2 \bigg)\ \  \Rightarrow\  \ 
    \big\| \|\oB'\|^2 \big\|_{L^{p/2}}\le c_1^2 (n+p) \bigg(p\lambda_1^2 + \sum_{i}\lambda_i^2 \bigg).
\end{align*}
Together, we have for $p\ge 4$,
\begin{align*}
& \quad \|\oB'\|_{L^p}^2 =  \big\| \|\oB'\|^2 \big\|_{L^{p/2}} \le c_1^2 (n+p) \bigg(p\lambda_1^2 + \sum_{i}\lambda_i^2 \bigg)\\
 \Rightarrow &\quad  \|\oB\|_{L^p} \le 4 \|\oB'\|_{L^p} \le 4 c_1 \sqrt{(n+p) \bigg(p\lambda_1^2 + \sum_{i}\lambda_i^2 \bigg)}.
\end{align*}

\emph{Final bound.}~
Combining the $p$-th moment bound, we have 
\begin{align*}
    \|\XB\XB^\top \uB \|_{L^p} \le \|\dB\|_{L^p} + \|\oB\|_{L^p} \le c_2 \bigg( p \lambda_1 + \sum_{i}\lambda_i\bigg) + 4 c_1 \sqrt{(n+p) \bigg(p\lambda_1^2 + \sum_{i}\lambda_i^2 \bigg)},\quad p\ge 4.
\end{align*}
Applying the above with Markov inequality (with $p=t$), we have with probability at least $1-\exp(-t)$ for $t\ge 4$,
\begin{align*}
  \|\XB\XB^\top \uB \|\le e c_2 \bigg( t \lambda_1 + \sum_{i}\lambda_i\bigg) + 4e c_1 \sqrt{(n+ t) \bigg(t\lambda_1^2 + \sum_{i}\lambda_i^2 \bigg)}.
\end{align*}
Squaring both sides completes our proof.
\end{proof}

\begin{lemma}[tail concentration]\label[lemma]{lemma:tail-concentration}
For any $k\le n$, fix $k$ orthogonal unit vectors $\uB_1,\dots,\uB_k\in\Rbb^n$. 
Under \Cref{assum:upper-bound:item:x}, there exist $c_0, c_1\ge 1$ that only depend on $\sigma_x^2$ for which the following holds.
With probability at least $1-\exp(-n/c_0)$, 
\[
\big\|\XB_{>k} \XB_{>k}^\top \big\| \le c_1 \bigg(n\lambda_{k+1} + \sum_{i>k}\lambda_i\bigg).
\]
Additionally, with probability at least $1-\exp(-k/c_0)$, for all unit vector $\uB \in \spn \{\uB_1,\dots,\uB_k\}$,
\begin{align*}
\uB^\top      \big( \XB_{>k} \SigmaB_{>k}\XB_{>k}^\top  \big) \uB 
&\le c_1 \bigg(k \lambda_{k+1}^2  + \sum_{i>k}\lambda_i^2 \bigg), \\
\uB^\top     \big(  \XB_{>k} \XB_{>k} ^\top \big)^2  \uB 
&\le  c_1 n \bigg(k\lambda_{k+1}^2+ \sum_{i>k}\lambda_i^2 \bigg) + c_1 \bigg(k\lambda_{k+1}+\sum_{i>k}\lambda_i\bigg)^2.
\end{align*}
\end{lemma}
\begin{proof}[Proof of \Cref{lemma:tail-concentration}]
Denote the intersection of the $k$-dimensional subspace $\spn\{\uB_1,\dots,\uB_k\}$ and the unit sphere by
\[
\Ubb := \{\uB \in \Rbb^n: \|\uB\| = 1,\ \uB \in \spn\{\uB_1,\dots,\uB_k\}\}.
\]
Let $(a_i)_{i\ge 1}$ be a sequence of fixed non-negative scalars in non-increasing order.
Let $(\zB_i)_{i\ge 1}$ be a sequence of independent random vectors in $\Rbb^n$ with independent $\sigma_x^2$-subgaussian entries with unit variance. 
Then by Bernstein's inequality and a net argument on $\Ubb$ \citep[see, e.g.,][Proof of Lemma 9 in Appendix C]{bartlett2020benign}, we have 
\begin{equation}\label{eq:bernstein-union-bound:subspace}
    \Pr\bigg(\sup_{\uB\in\Ubb} \uB^\top \bigg(\sum_{i}a_i \zB_i\zB_i^\top\bigg) \uB - \sum_{i}a_i > t \bigg)  < \exp\bigg( -\frac{1}{c} \min\bigg\{\frac{t^2}{\sum_i a_i^2},\, \frac{t}{a_1} \bigg\} + 10k \bigg),
\end{equation}
where $c\ge 1$ only depends on $\sigma_x^2$.

\emph{The first claim.}~
Note that 
\[
\XB_{>k}\XB_{>k}^\top = \sum_{i>k} \lambda_i \zB_i \zB_i^\top,
\]
where $\zB_i\in\Rbb^n$ for $i>k$ are independent random vectors with independent $\sigma_x^2$-subgaussian entries with unit variance by \Cref{assum:upper-bound:item:x}. 
So the first claim follows by applying \Cref{eq:bernstein-union-bound:subspace} with $n$ as the dimension of the subspace, $(\lambda_i)_{i>k}$ as the weights, and 
\[ t = \max \bigg\{11cn \lambda_{k+1},\, \sqrt{11cn \sum_{i>k}\lambda_i^2}  \bigg\} \le 11cn \lambda_{k+1} + \sum_{i>k}\lambda_i\]
with rescaled constants.

\emph{The second claim.}~
Note that 
\[
\XB_{>k}\SigmaB_{>k}\XB_{>k}^\top = \sum_{i>k} \lambda_i^2 \zB_i \zB_i^\top,
\]
where $\zB_i\in\Rbb^n$ for $i>k$ are independent random vectors with independent $\sigma_x^2$-subgaussian entries with unit variance by \Cref{assum:upper-bound:item:x}. 
So the second claim follows by applying \Cref{eq:bernstein-union-bound:subspace} with $(\lambda_i^2)_{i> k}$ as the weights and 
\[ t = \max \bigg\{11ck \lambda_{k+1}^2,\, \sqrt{11ck \sum_{i>k}\lambda_i^4}  \bigg\} \le 11ck \lambda_{k+1}^2 + \sum_{i>k}\lambda_i^2\]
with rescaled constants.

\emph{The third claim.}~
This is by applying \Cref{lemma:fourth-moment-concentration} to a net over the $k$-dimensional set $\Ubb$ with $t = 10 k$, then applying a union bound.
\end{proof}

\subsection{Variance and effective bias errors}\label[appendix]{append:sec:gd:exp:effective-bias}
\begin{lemma}[variance and effective bias errors]\label[lemma]{lemma:gd:exp:effective-bias}
Assume that 
\[
 \eta \le \frac{1}{2\tr(\SigmaB)},\quad t\le b n,
\]
for a positive constant $b > 0$.
Under \Cref{assum:upper-bound:item:x}, there exist $c_0, c_2, c_3 \ge 1$ only depending on $\sigma_x^2$, and $c_1\ge 1$ only depending on $\sigma_x^2$ and $b$, for which the following holds.  
For every $k$ such that
\[
\frac{1}{\eta t} \ge c_2 \lambda_{k+1},\quad k\le \frac{n}{c_3},
\]
with probability at least $1-\exp(-n/c_0)$, the variance error and the effective bias error are respectively bounded by
\begin{align*}
\variance &:= \sigma^2 \la \CB, \SigmaB\ra \\
&\le c_1 \sigma^2 \frac{k+(\eta t)^2 \sum_{i>k}\lambda_i^2}{n}, \\ 
\effBias &:=\big\|(\IB - \eta  \hat \SigmaB)^{t/2}(\IB - \eta  \SigmaB_{0: k})^{t/2} \wB^* \big\|^2_{\SigmaB}  \\
&\le c_1 \bigg( \frac{1}{(\eta t)^2}\big\|(\IB - \eta  \SigmaB)^{t/2} \wB^* \big\|^2_{\SigmaB_{0:k}} + \| \wB^* \|^2_{\SigmaB_{k:\infty}} \bigg).
\end{align*}
\end{lemma}
\begin{proof}[Proof of \Cref{lemma:gd:exp:effective-bias}]
The variance error has been analyzed in \Cref{lemma:gd:variance} in \Cref{append:sec:gd:variance}.
The effective bias error is exactly the bias error of GD with $t/2$ steps for a linear regression problem with optimal parameter 
\[(\IB - \eta  \SigmaB_{0: k})^{t/2} \wB^* = \begin{bmatrix}[1.2]
    (\IB_k - \eta  \SigmaB_{\le k})^{t/2} \wB^*_{\le k} \\
    \wB^*_{>k}
\end{bmatrix},\] 
which has been analyzed in \Cref{lemma:gd:bias-head,lemma:gd:bias-tail} in \Cref{append:sec:gd:bias}.
Applying \Cref{lemma:gd:variance} in \Cref{append:sec:gd:variance}, and  applying \Cref{lemma:gd:bias-head,lemma:gd:bias-tail} in \Cref{append:sec:gd:bias} with $t\leftarrow t/2$ and $\wB^* \leftarrow (\IB - \eta  \SigmaB_{0: k})^{t/2} \wB^*$, we obtain the following bound:
there exist $c_0, c_1, c_2, c_3 \ge 1$ that only depend on $\sigma_x^2$ such that, for all $\eta \le n/\|\XB\XB^\top\|$, and for every $k$ such that 
\[
k\le \frac{n}{c_3},\quad 
\frac{1}{\eta t} + \frac{\sum_{i>k}\lambda_i }{n} \ge c_2 \lambda_{k+1},
\]
it holds with probability at least $1-\exp(-n/c_0)$ that
\begin{align*}
\variance &:= \sigma^2 \la \CB, \SigmaB\ra  
\le c_1 \sigma_2^2 \frac{k+ \Big(\frac{1}{\eta t} + \frac{\sum_{i>k}\lambda_i}{n}\Big)^{-2} \sum_{i>k}\lambda_i^2}{n}, \\ 
   \effBias & := \big\|(\IB - \eta  \hat \SigmaB)^{t/2}(\IB - \eta  \SigmaB_{0: k})^{t/2} \wB^* \big\|^2_{\SigmaB} \\
    &\le c_1 \bigg( \bigg(\frac{1}{\eta t} + \frac{\sum_{i>k}\lambda_i}{n}\bigg)^2\big\|(\IB - \eta  \SigmaB_{0: k})^{t/2} \wB^* \big\|^2_{\SigmaB^{-1}_{0: k }} + \| \wB^* \|^2_{\SigmaB_{k:\infty}} \bigg).
\end{align*}
These give the promised bound with the following minor tweaks. 
By \Cref{lemma:head-and-stepsize-concentration}, with probability at least $1-\exp(-n/c_0)$, the stepsize condition is satisfied:
\begin{align*}
    \eta \le \frac{1}{2\tr(\SigmaB)} \quad \Rightarrow \quad \eta \le \frac{n}{\|\XB\XB^\top\|}.
\end{align*}
Moreover, the index condition is also satisfied because
\begin{align*}
    \frac{1}{\eta t} \ge  c_2 \lambda_{k+1}\quad \Rightarrow \quad 
    \frac{1}{\eta t} +  \frac{\sum_{i>k}\lambda_i }{n} \ge c_2 \lambda_{k+1}.
\end{align*}
Finally, since $t\le bn$ and $\eta \le 1/(2\tr(\SigmaB))$, we have 
\begin{align*}
   \frac{\sum_{i>k}\lambda_i }{n}  \le 
    \frac{\tr(\SigmaB)}{n} \le \frac{b/2}{\eta t}  \quad 
    \Rightarrow \quad \frac{1}{\eta t}\le \frac{1}{\eta t} + \frac{\sum_{i>k}\lambda_i}{n} \le \frac{1+b/2}{\eta t}.
\end{align*}
Plugging this into the bounds, applying a union bound over all required events, and rescaling the constants, we complete the proof.
\end{proof}

\subsection{Effective variance error}\label[appendix]{append:sec:gd:exp:effective-variance}

\begin{lemma}[effect of empirical covariance]\label[lemma]{lemma:gd:exp:empirical-covariance}
Assume that 
\[
 \eta \le \frac{1}{2\tr(\SigmaB)},\quad t\le b n,
\]
for a positive constant $b > 0$.
Let each of $c_0,c_1,c_2,c_3$ be the maximum of those appearing in \Cref{lemma:basic-concentration,lemma:tail-regularization,lemma:head-and-stepsize-concentration,lemma:tail-concentration}.
Let $k$ be an index such that 
\[
\frac{1}{\eta t} \ge c_2 \lambda_{k+1},\quad 
k\le \frac{n}{c_3},
\]
which satisfies the conditions in \Cref{lemma:basic-concentration,lemma:tail-regularization}.
Let 
\begin{align*}
    D:= k + (\eta t)^2\sum_{i>k}\lambda_i^2,\quad D_1 := k+ \eta t \sum_{i>k}\lambda_i.
\end{align*}
Then under the joint events of \Cref{lemma:basic-concentration,lemma:tail-regularization,lemma:head-and-stepsize-concentration,lemma:tail-concentration},
we have
\begin{align*}
     (\hat\SigmaB- \SigmaB_{0:k})\begin{bmatrix}[1.5]
        \dfrac{1}{\eta^2 t^2}\SigmaB^{-1}_{\le k} & \\
        & \dfrac{1}{n} \XB_{>k}^\top \XB_{>k} + \SigmaB_{>k}
    \end{bmatrix}(\hat\SigmaB- \SigmaB_{0:k}) 
     \preceq \dfrac{10c_1^2(1+b)^2}{\eta^2 t^2} \begin{bmatrix}[1.5]
         \bigg(\dfrac{D}{n} + \dfrac{D_1^2}{n^2}\bigg)\SigmaB_{\le k } & \\
         & \dfrac{1}{n }\XB_{>k }^\top \XB_{>k } 
     \end{bmatrix}.
\end{align*}
\end{lemma}
\begin{proof}[Proof of \Cref{lemma:gd:exp:empirical-covariance}]
By the matrix splitting notation, we have
\begin{align*}
  \hat\SigmaB- \SigmaB_{0: k} 
  =  \begin{bmatrix}[2]
        \dfrac{1}{n}\XB_{\le k}^\top \XB_{\le k} - \SigmaB_{\le k} & \dfrac{1}{n}\XB_{\le k}^\top \XB_{>k}  \\
        \dfrac{1}{n}\XB_{>k}^\top \XB_{\le k} & \dfrac{1}{n}\XB_{>k}^\top \XB_{>k} 
    \end{bmatrix}.
\end{align*}
We then have 
\begin{align*}
 \big( \hat\SigmaB- \SigmaB_{0: k} \big)\begin{bmatrix}[1.5]
        \dfrac{1}{\eta^2 t^2}\SigmaB^{-1}_{\le k} &  \\
        & \dfrac{1}{n} \XB_{>k}^\top \XB_{>k} +\SigmaB_{k:\infty}
    \end{bmatrix} \big( \hat\SigmaB- \SigmaB_{0: k} \big) 
=  \begin{bmatrix}
        \MB_{11} & * \\
        * & \MB_{22}
    \end{bmatrix}  
\preceq 2 \begin{bmatrix}
        \MB_{11} & 0 \\
        0 & \MB_{22}
\end{bmatrix},
\end{align*}
where the principal submatrices are given by
\begin{align*}
    \MB_{11} &:= \underbrace{\frac{1}{\eta^2 t^2}\bigg(\frac{1}{n}\XB_{\le k}^\top \XB_{\le k} - \SigmaB_{\le k}\bigg)\SigmaB_{\le k}^{-1}\bigg(\frac{1}{n}\XB_{\le k}^\top \XB_{\le k} - \SigmaB_{\le k}\bigg)}_{ \MB_{11}^{(1)}}  
     + \underbrace{\frac{1}{n^2 }\XB_{\le k}^\top\XB_{>k} \bigg(\frac{1}{n}\XB^\top_{> k}\XB_{>k} + \SigmaB_{>k}\bigg)\XB_{>k}^\top \XB_{\le k} }_{\MB_{11}^{(2)}}, \\
    \MB_{22} &:= \underbrace{\frac{1}{\eta^2 t^2 n^2 }\XB_{>k}^\top \XB_{\le k}\SigmaB^{-1}_{\le k}\XB_{\le k}^\top\XB_{>k}}_{ \MB_{22}^{(1)}} 
     + \underbrace{ \frac{1}{n^2}\XB_{>k}^\top \XB_{>k} \bigg( \frac{1}{n}\XB_{>k}^\top \XB_{>k}+\SigmaB_{> k} \bigg) \XB_{>k}^\top \XB_{>k} }_{\MB_{22}^{(2)}}.
\end{align*}
We analyze each component in what follows.

\emph{The head component.}~
For $\MB_{11}^{(1)}$, we have 
\begin{align*}
    \MB_{11}^{(1)} 
    &:= \frac{1}{\eta^2 t^2}\bigg(\frac{1}{n}\XB_{\le k}^\top \XB_{\le k} - \SigmaB_{\le k}\bigg)\SigmaB_{\le k}^{-1}\bigg(\frac{1}{n}\XB_{\le k}^\top \XB_{\le k} - \SigmaB_{\le k}\bigg) \\
    &= \frac{1}{\eta^2 t^2}\SigmaB^{\frac{1}{2}}_{\le k}\bigg(\frac{1}{n}\ZB_{\le k}^\top \ZB_{\le k} - \IB_k\bigg)^2 \SigmaB_{\le k}^{\frac{1}{2}}  \\
    &\preceq  \frac{c_1^2 }{\eta^2 t^2} \frac{k}{n} \SigmaB_{\le k}. && \explain{by \Cref{lemma:head-and-stepsize-concentration}}
\end{align*}
For $\MB_{11}^{(2)}$, we have 
\begin{align*}
\MB_{11}^{(2)}
&:= \frac{1}{n^2 }\XB_{\le k}^\top\XB_{>k} \bigg(\frac{1}{n}\XB^\top_{>k}\XB_{>k} + \SigmaB_{>k}\bigg)\XB_{>k}^\top \XB_{\le k} \\
&= \frac{1}{n} \XB_{\le k}^\top \bigg( \frac{1}{n} \XB_{>k} \SigmaB_{>k}\XB_{>k}^\top +\frac{1}{n^2} \big(  \XB_{>k} \XB_{>k}^\top \big)^2\bigg) \XB_{\le k}.
\end{align*}
Since $\XB_{\le k}$ is independent of $\XB_{>k}$, we can apply  \Cref{lemma:tail-concentration} to obtain
\begin{align*}
 \MB_{11}^{(2)} 
&\preceq 
    c_1 \bigg( 2\frac{k\lambda_{k+1}^2 + \sum_{i>k}\lambda_i^2}{n} + \frac{\big(k\lambda_{k+1} + \sum_{i>k}\lambda_i\big)^2}{n^2}\bigg) \frac{1}{n} \XB_{\le k}^\top \XB_{\le k} \\
&\preceq 
    2c_1 \bigg( 2\frac{k\lambda_{k+1}^2 + \sum_{i>k}\lambda_i^2}{n} + \frac{\big(k\lambda_{k+1} + \sum_{i>k}\lambda_i\big)^2}{n^2}\bigg)\SigmaB_{\le k} ,
\end{align*}
where the last inequality is by \Cref{lemma:basic-concentration}.
Using the conditions that 
\[
\frac{1}{\eta t} \ge c_2 \lambda_{k+1},\quad t\le b n,\quad \eta \le \frac{1}{2\tr(\SigmaB)},
\]
we have 
\begin{equation}\label{eq:gd:exp:k-condition}
    \lambda_{k+1}\le \frac{1}{c_2 \eta t} \le \frac{1}{\eta t},\quad \frac{\sum_{i>k}\lambda_i}{n} \le \frac{\tr(\SigmaB)}{n}\le \frac{b/2}{\eta t} .
\end{equation}
Plugging these into the above bound for $\MB_{11}^{(2)}$, we get
\begin{align*}
    \MB_{11}^{(2)}
    &\preceq 
    \frac{2c_1 }{\eta^2 t^2} \bigg( 2\frac{k + \eta^2 t^2\sum_{i>k}\lambda_i^2}{n} + \frac{\big(k + \eta t\sum_{i>k}\lambda_i\big)^2}{n^2}\bigg)\SigmaB_{\le k}  \\
 &\preceq 
   \frac{4 c_1 }{\eta^2 t^2}\bigg(\frac{D}{n} + \frac{D_1^2}{n^2}\bigg)\SigmaB_{\le k} .
\end{align*}
Putting these together, the head component is bounded by 
\begin{align*}
    \MB_{11} &:= \MB_{11}^{(1)} + \MB_{11}^{(2)} \\ 
    &\preceq \frac{c_1^2 }{\eta^2 t^2} \frac{k}{n} \SigmaB_{\le k} + 
 \frac{4 c_1 }{\eta^2 t^2}  \bigg(\frac{D}{n} + \frac{D_1^2}{n^2}\bigg)\SigmaB_{\le k}  \\
&\preceq 
   \frac{5 c_1^2 }{\eta^2 t^2}\bigg(\frac{D}{n} + \frac{D_1^2}{n^2}\bigg)\SigmaB_{\le k} ,
\end{align*}
where the last inequality is because $k\le D$.

\emph{The tail component.}~
For $\MB_{22}^{(1)}$, we have 
\begin{align*}
    \MB_{22}^{(1)} 
    &:= \frac{1}{\eta^2 t^2 n^2 }\XB_{>k}^\top \XB_{\le k}\SigmaB^{-1}_{\le k}\XB_{\le k}^\top\XB_{>k} \\
    &\preceq  \frac{1}{\eta^2 t^2 }\frac{\big\|\XB_{\le k}\SigmaB^{-1}_{\le k}\XB_{\le k}^\top\big\|}{n} \frac{1}{n}\XB_{>k}^\top \XB_{>k} \\
     &\preceq \frac{2}{\eta^2 t^2 } \frac{1}{n}\XB_{>k}^\top \XB_{>k}. && \explain{by \Cref{lemma:basic-concentration}}
\end{align*}
For $\MB_{22}^{(2)}$, we have 
\begin{align*}
    \MB_{22}^{(2)} &:=  \frac{1}{n^2}\XB_{>k}^\top \XB_{>k}\bigg(\frac{1}{n}\XB_{>k}^\top \XB_{>k}+  \SigmaB_{>k}\bigg) \XB_{>k}^\top \XB_{>k}  \\
    &\preceq \bigg(\frac{\big\|\XB_{>k} \XB_{>k}^\top \big\|^2}{n^2} + \frac{\big\|\XB_{>k}\SigmaB_{> k} \XB_{>k}^\top\big\|}{n}\bigg) \frac{1}{n}\XB_{>k}^\top \XB_{>k}  \\
    &\preceq \bigg(c_1^2 \frac{\big(n\lambda_{k+1}+\sum_{i>k}\lambda_i\big)^2}{n^2}+ c_1\frac{n\lambda_{k+1}^2 + \sum_{i>k}\lambda_i^2}{n} \bigg)\frac{1}{n} \XB_{>k}^\top \XB_{>k} && \explain{by \Cref{lemma:tail-concentration,lemma:basic-concentration}} \\
    &\preceq  \frac{c_1^2(1+b)^2+c_1(1+b)}{\eta^2 t^2 } \frac{1}{n}\XB_{>k}^\top \XB_{>k}. && \explain{by \Cref{eq:gd:exp:k-condition}}
\end{align*}
Putting things together, we have 
\begin{align*}
    \MB_{22} &: = \MB_{22}^{(1)} + \MB_{22}^{(2)} 
    \preceq \frac{2+c_1^2(1+b)^2+c_1(1+b)}{\eta^2 t^2} \frac{1}{n}\XB_{>k}^\top \XB_{>k}
    \preceq \frac{4c_1^2(1+b)^2}{\eta^2 t^2} \frac{1}{n}\XB_{>k}^\top \XB_{>k}.
\end{align*}
Combining everything completes the proof.
\end{proof}

\begin{lemma}[effective variance error]\label[lemma]{lemma:gd:exp:effective-var}
Assume that 
\[
 \eta \le \frac{1}{2\tr(\SigmaB)},\quad t\le b n,
\]
for a positive constant $b > 0$.
Under \Cref{assum:upper-bound:item:x}, there exist $c_2, c_3 \ge 1$ only depending on $\sigma_x^2$, and $c_0, c_1\ge 1$ only depending on $\sigma_x^2$ and $b$, for which the following holds.  
For every $k$ such that
\[
\frac{1}{\eta t} \ge c_2 \lambda_{k+1},\quad k\le \frac{n}{c_3},
\]
with probability at least $1-\exp(-k/c_0)$, the effective variance error in \Cref{eq:gd:exp:bias-decomposition} is bounded by
\[
\effVar \le c_1 \Bigg( \bigg(\frac{D}{n}+ \frac{D_1^2}{n^2} \bigg)\frac{1}{\eta^2 t^2}\|\wB^* \|^2_{\SigmaB_{0:k}^{-1}} + \| \wB^* \|^2_{\SigmaB_{k:\infty}} \Bigg).
\]
\end{lemma}
\begin{proof}[Proof of \Cref{lemma:gd:exp:effective-var}]
Recall the definition of the effective variance error in \Cref{eq:gd:exp:bias-decomposition}.
We have
\begin{align*}
&\ \sqrt{\effVar}
:=  \|\SigmaB^{1/2}(\IB-\eta \hat\SigmaB)^{t/2}\QB \wB^*\| \\
 &= \eta \bigg\|\sum_{s=0}^{t/2-1} \SigmaB^{\frac{1}{2}}(\IB-\eta\hat\SigmaB)^{t-1-s}(\SigmaB_{0: k} -\hat\SigmaB)(\IB-\eta\SigmaB_{0: k})^{s}\wB^* \bigg\| \qquad \explain{definition of $\QB$}\\
&\le \eta \sum_{s=0}^{t/2-1}\big\| \SigmaB^{\frac{1}{2}}(\IB-\eta\hat\SigmaB)^{t-1-s}(\SigmaB_{0: k} -\hat\SigmaB)(\IB-\eta\SigmaB_{0: k})^{s}\wB^* \big\| \qquad \explain{triangle inequality}\\
&= \eta \sum_{s=0}^{t/2-1}\sqrt{\big\la (\IB-\eta\hat\SigmaB)^{t-1-s}\SigmaB(\IB-\eta\hat\SigmaB)^{t-1-s},\, \big( (\SigmaB_{0: k} -\hat\SigmaB)(\IB-\eta\SigmaB_{0: k})^{s}\wB^*\big)^{\otimes 2} \big\ra  }.
\end{align*}
Here, the inner matrix
\(
\big( \IB-\eta\hat\SigmaB \big)^{t-1-s}\SigmaB\big(\IB-\eta\hat\SigmaB\big)^{t-1-s} 
\)
is exactly the bias matrix for GD in $t-1-s \ge t/2$ steps (see the definition of $\BB$ in \Cref{eq:bias-variance}). 
Due to the same argument as in \Cref{lemma:gd:exp:effective-bias} (by applying \Cref{lemma:gd:bias-head,lemma:gd:bias-tail}), we have the following:
there exist $c_0, c_1, c_2, c_3 \ge 1$ only depending on $\sigma_x^2$ such that for each $s\le t/2-1$ and each $k$ such that 
\[
\frac{1}{\eta t} \ge c_2 \lambda_{k+1},\quad k\le \frac{n}{c_3},
\]
with probability at least $1-\exp(-n/c_0)$, 
\begin{align*}
   &\  (\IB-\eta\hat\SigmaB)^{t-1-s}\SigmaB(\IB-\eta\hat\SigmaB)^{t-1-s} \\
    &\preceq c_1 \begin{bmatrix}[1.5]
        \dfrac{1}{\eta^2 (t-1-s)^2}\SigmaB^{-1}_{\le k} & \\
        & \dfrac{1}{n} \XB_{>k}^\top \XB_{>k} + \SigmaB_{> k}
    \end{bmatrix} \\
    &\preceq 4c_1 \begin{bmatrix}[1.5]
        \dfrac{1}{\eta^2 t^2}\SigmaB^{-1}_{\le k} & \\
        & \dfrac{1}{n} \XB_{>k}^\top \XB_{>k} + \SigmaB_{> k}
    \end{bmatrix}. && \explain{since $s\le t/2-1$}
\end{align*}
Applying a union bound, the above events hold for all $0\le s\le t/2-1\le bn$ simultaneously with probability 
\begin{align*}
    1- bn \exp(-n/c_0) \ge 1-\exp(-n/b_0),
\end{align*}
for some $b_0\ge 1$ depending on $c_0$ and $b$.
Under this joint event, we apply \Cref{lemma:gd:exp:empirical-covariance} to obtain that: for all $s\le t/2-1$,
\begin{align*}
     (\SigmaB_{0: k} -\hat\SigmaB)(\IB-\eta\hat\SigmaB)^{t-1-s}\SigmaB(\IB-\eta\hat\SigmaB)^{t-1-s}(\SigmaB_{0: k} -\hat\SigmaB) 
    \preceq 
     \frac{b_1}{\eta^2 t^2} 
     \begin{bmatrix}[1.5]
         \bigg(\dfrac{D}{n}+\dfrac{D^2_1}{n^2}\bigg)\SigmaB_{\le k } &  \\
         & \dfrac{1}{n }\XB_{>k }^\top \XB_{>k } & 
     \end{bmatrix},
\end{align*}
for some $b_1$ only depending on $\sigma_x^2$ and $b$.
Plugging the above into the bound on the effective variance error, and noticing that 
\[
(\IB-\eta\SigmaB_{0:k})^s \wB^* = \begin{bmatrix}[1.5]
    (\IB-\eta\SigmaB_{\le k})^s\wB^*_{\le k} \\
    \wB^*_{>k}
\end{bmatrix},
\]
we get 
\begin{align*}
 &\ \sqrt{\effVar }  \\
&\le
    \frac{b_1}{t} \sum_{s=0}^{t/2-1}  \sqrt{\bigg(\frac{D}{n} + \frac{D_1^2}{n^2} \bigg)\big\|(\IB-\eta\SigmaB_{\le k})^{s}\wB^*_{\le k} \big\|^2_{\SigmaB_{\le k}} +  \frac{1}{n}\wB_{>k }^{*\top} \XB_{>k }^\top \XB_{>k } \wB^*_{>k } } \\ 
    &\le \frac{b_1}{t} \sum_{s=0}^{t/2-1}  \sqrt{\bigg(\frac{D}{n} + \frac{D_1^2}{n^2} \bigg) \big\|(\IB-\eta\SigmaB_{\le k})^{s}\wB^*_{\le k} \big\|^2_{\SigmaB_{\le k}} +  c_1  \|\wB^*_{>k }\|^2_{\SigmaB_{>k}} } \qquad \explain{by \Cref{lemma:basic-concentration}} \\ 
    &\le \frac{b_1}{t} \sum_{s=0}^{t/2-1}  \bigg(\sqrt{\bigg(\frac{D}{n} + \frac{D_1^2}{n^2} \bigg)} \big\|(\IB-\eta\SigmaB_{\le k})^{s}\wB^*_{\le k} \big\|_{\SigmaB_{\le k}} +  \sqrt{c_1}  \|\wB^*_{>k }\|_{\SigmaB_{>k}} \bigg)\\
     &\le \frac{b_1}{t} \sqrt{\bigg(\frac{D}{n} + \frac{D_1^2}{n^2} \bigg)}\big\|\SigmaB_{\le k}^{-1/2}\wB^*_{\le k} \big\| \cdot \sum_{s=0}^{t/2-1}  \big\|\SigmaB_{\le k}^{1/2}(\IB-\eta\SigmaB_{\le k})^{s} \SigmaB_{\le k}^{1/2}\big\| +  \frac{b_1\sqrt{c_1}}{2}  \|\wB^*_{>k }\|_{\SigmaB_{>k}} \\
     &\le \frac{b_1}{t} \sqrt{\bigg(\frac{D}{n} + \frac{D_1^2}{n^2} \bigg)}\big\|\SigmaB_{\le k}^{-1/2}\wB^*_{\le k} \big\| \cdot \frac{1}{\eta}  +  \frac{b_1\sqrt{c_1}}{2}  \|\wB^*_{>k }\|_{\SigmaB_{>k}}. \qquad \explain{limit of geometric series}
\end{align*}
Squaring both sides completes our proof.
\end{proof}

\subsection{\texorpdfstring{Proof of \Cref{thm:gd:exp}}{Proof of SGD-type upper bound}}
\begin{proof}[Proof of \Cref{thm:gd:exp}]
We decompose the excess risk into the sum of variance, effective bias, and effective variance errors. 
Bounds on the variance and effective bias errors are due to  \Cref{lemma:gd:exp:effective-bias}, and bounds on the effective variance are due to \Cref{lemma:gd:exp:effective-var}.
By a union bound, the joint of all required events holds with probability at least $1-\exp(-k/c_0)$ for some $c_0\ge 1$ depending on $\sigma_x^2$ and $b$.
This completes the proof. Finally, to simplify the presentation of the final bound, we absorb the tail component of the effective variance error into the tail component of the effective bias error and rescale the constants.
\end{proof}

\section{Analysis for power law class}\label[appendix]{append:sec:power-law}

\begin{lemma}[head and tail rates]\label[lemma]{lemma:power-law:effective-regu}
Assume that $\lambda_i \lesssim i^{-a}$ for some $a>1$.
Then for all $\mu>0$, we have 
\begin{gather*}
    \#\{i: \lambda_i > \mu \} \lesssim \mu^{-\frac{1}{a}}, \\
    \sum_{\lambda_i \le \mu }\lambda_i \lesssim \mu^{\frac{a-1}{a}}, \\ 
    \sum_{\lambda_i\le \mu} \lambda_i^2 \lesssim \mu^{\frac{2a-1}{a}}.
\end{gather*}
\end{lemma}
\begin{proof}[Proof of \Cref{lemma:power-law:effective-regu}]
For any $i$ such that $\lambda_i>\mu$, the power law upper bound gives
\[
\mu < \lambda_{i} \lesssim i^{-a}\quad \Rightarrow \quad 
i \lesssim \mu^{-1/a},
\]
which gives the first claim.
The second and third claims are corollaries of the following claim:
\begin{align*}
    \text{for all $p>1/a$},\quad \sum_{\lambda_i\le \mu}\lambda_i^p \lesssim \mu^{p-\frac{1}{a}}.
\end{align*}
Indeed, decomposing the set $\{i:\lambda_i\le \mu\}$ into dyadic slices gives
\begin{align*}
    \sum_{\lambda_i\le \mu}\lambda_i^p 
    &= \sum_{\ell=0}^{\infty}\sum_{ \mu/2^{\ell+1}<\lambda_i \le \mu/2^{\ell} } \lambda_i^p \\ 
    &\le \sum_{\ell=0}^{\infty} \bigg(\frac{\mu}{2^{\ell}}\bigg)^p
    \#\big\{i: \mu/2^{\ell+1}<\lambda_i \le \mu/2^{\ell}\big\}  \\
    &\le \sum_{\ell=0}^{\infty} \bigg(\frac{\mu}{2^{\ell}}\bigg)^p
    \#\big\{i: \lambda_i > \mu/2^{\ell+1}\big\}  \\
    &\lesssim \sum_{\ell=0}^{\infty} \bigg(\frac{\mu}{2^{\ell}}\bigg)^p
    \bigg(\frac{\mu}{2^{\ell+1}}\bigg)^{-\frac{1}{a}} && \explain{by the first claim} \\
    &\lesssim \mu^{p-\frac{1}{a}} \sum_{\ell=0}^{\infty} \big(2^{\frac{1}{a}-p}\big)^{\ell} \\ 
    &\lesssim \mu^{p-\frac{1}{a}}, && \explain{since $p>1/a$}
\end{align*}
which gives the promised claim.
Taking $p=1$ and $p=2$ completes the proof.
\end{proof}

\subsection{\texorpdfstring{Proof of \Cref{thm:power-law:ridge}}{Ridge for power law class}}\label[appendix]{append:sec:power-law:ridge}

\begin{proof}[Proof of \Cref{thm:power-law:ridge}]
Recall the notation in \Cref{thm:ridge} and \Cref{eq:power-law}. Without loss of generality, assume that $\SigmaB$ is diagonal.
By the definitions of $k^*$ and $\tilde\lambda$ in \Cref{thm:ridge}, we have 
\begin{align*}
    \lambda_{k^*+1} \le \tilde\lambda / c_2 < \lambda_{k^*},\quad k^* = \#\{i: \lambda_i > \tilde\lambda / c_2\}.
\end{align*}

\emph{Upper bound.}~
Let us first consider $0\le r\le 1$ and prove an upper bound for ridge regression. 
By \Cref{lemma:power-law:effective-regu}, the effective dimension $D$ in \Cref{thm:ridge} is 
\begin{align*}
    D &:= k^* + \tilde\lambda^{-2}\sum_{i>k^*}\lambda_i^2 \\
    &\lesssim  \tilde\lambda^{-\frac{1}{a}} + \tilde\lambda^{-2} \tilde\lambda^{2-\frac{1}{a}} \\
    & \eqsim \tilde\lambda^{-\frac{1}{a}} .
\end{align*}
So the variance error is 
\begin{align*}
    \variance \lesssim \frac{D}{n} \lesssim \frac{\tilde\lambda^{-\frac{1}{a}}}{n}.
\end{align*}
For the bias error, we have
\begin{align*}
    \bias 
    &\eqsim \tilde\lambda^2\|\wB^*\|^2_{\SigmaB^{-1}_{0:k^*}} + \|\wB^*\|^2_{\SigmaB_{k^*:\infty}} \\
    &\eqsim
    \tilde\lambda^2 \sum_{i\le k^*} \lambda_i^{-1}\wB_i^{*2} + \sum_{i>k^*}\lambda_i \wB_i^{*2} \\
&\eqsim \tilde\lambda^2 \sum_{i\le k^*}\lambda_i^{2(r-1)} \lambda_i^{1-2r}\wB_i^{*2} + \sum_{i>k^*}\lambda_i^{2r} \lambda_i^{1-2r} \wB_i^{*2} .
\end{align*}
Recall that $0\le r \le 1$. Then by the definitions of $k^*$ and $\tilde\lambda$, we have 
\begin{align*}
  &  \text{for $i>k^*$},\ \lambda_i^{2r} \le \lambda_{k^*+1}^{2r} \lesssim \tilde\lambda^{2r}, \\ 
   & \text{for $i\le k^*$},\ \lambda_i^{2(r-1)} \le \lambda_{k^*}^{2(r-1)} \lesssim \tilde\lambda^{2(r-1)}.
\end{align*}
Bringing this back, we have 
\begin{align*}
    \bias & \lesssim  \tilde\lambda^{2r} \sum_{i} \lambda_i^{1-2r}\wB_i^{*2} \\
    &\lesssim \tilde\lambda^{2r} && \explain{by \Cref{eq:power-law}}. 
\end{align*}
Lastly, we compute $\tilde\lambda$.
Clearly, $\tilde\lambda\ge \lambda$;
additionally, we have 
\begin{align*}
\tilde\lambda 
    &:=  \lambda+ \frac{\sum_{i>k^*\lambda_i}}{n} \\ 
    &\lesssim \lambda + \frac{\tilde\lambda^{1-\frac{1}{a}}}{n} && \explain{by \Cref{lemma:power-law:effective-regu}} \\ 
    &\eqsim \max\bigg\{\lambda,\, \frac{\tilde\lambda^{1-\frac{1}{a}}}{n} \bigg\},
\end{align*}
which implies that $\tilde\lambda \lesssim \max\{\lambda, n^{-a}\}$ by discussing the two branches.
So we have 
\[
\lambda \lesssim \tilde\lambda \lesssim \max\{\lambda, n^{-a}\}.
\]
To proceed, we choose
\[
\lambda \eqsim n^{-\frac{a}{1+2a r}}\quad \Rightarrow\quad \lambda \gtrsim n^{-a}\quad \Rightarrow\quad 
    \tilde{\lambda} \eqsim n^{-\frac{a}{1+2a r}}.\]
So the sum of the bias and variance errors is
\begin{align*}
  \bias + \variance \lesssim \tilde\lambda^{2r}  + \frac{\tilde\lambda^{-\frac{1}{a}}}{n}  
  \eqsim n^{-\frac{2a r}{1+2a r}},
\end{align*}
which gives the promised rate.

\emph{Lower bound.}~
We next consider $r>1$ and prove a lower bound for ridge regression with any $\lambda\ge 0$.
We consider the following hard problem from $\Pbb_{a,r}$:
\[
\sigma^2\eqsim 1,\quad 
\lambda_i\eqsim i^{-a},\quad 
\lambda_i\wB_i^{*2} \eqsim  i^{-b} \quad \text{for any}\  b> 1+2ar.
\]
Clearly, $b>1+2a$ since $r>1$.
Note that $\lambda$ decides $k^*$; without loss of generality, we prove the lower bound for all $k^*$.
If $k^* = 0$, then \Cref{thm:ridge} implies a $\Theta(1)$ risk lower bound since
\begin{align*}
    \bias \ge \|\wB^*\|^2_{\SigmaB} \eqsim  \sum_{i=1}^{\infty} i^{-b} \gtrsim 1.
\end{align*}
Otherwise, we have $k^* \ge 1$; by the definition, we have 
\begin{align*}
    \tilde{\lambda} \gtrsim \lambda_{k^*+1} \eqsim (k^*)^{-a} \quad \Rightarrow \quad k^* \gtrsim \tilde\lambda^{-1/a}.
\end{align*}
Then \Cref{thm:ridge} implies a risk lower bound via the head components of the variance and bias errors,
\begin{align*}
    \variance &\gtrsim \frac{k^*}{n} \gtrsim \frac{\tilde\lambda^{-\frac{1}{a}}}{n}, \\ 
    \bias &\ge \tilde\lambda^2 \|\wB^*\|^2_{\SigmaB^{-1}_{0:k^*}} 
    \eqsim \tilde\lambda^2 \sum_{i=1}^{k^*} i^{-(b-2a)} \eqsim \tilde\lambda^2,
\end{align*}
that is, the risk is lower bounded by 
\[
\Omega\bigg( \frac{\tilde\lambda^{-\frac{1}{a}}}{n}  + \tilde\lambda^2 \bigg) = \Omega\bigg( n^{-\frac{2a}{1+2a}} \bigg),\quad \text{for all $\tilde\lambda\ge 0$}.
\]
We have completed the proof.
\end{proof}

\subsection{\texorpdfstring{Proof of \Cref{thm:power-law:sgd}}{SGD for power-law class}}\label[appendix]{append:sec:power-law:sgd}
\begin{proof}[Proof of \Cref{thm:power-law:sgd}]
Recall the notation in \Cref{thm:sgd} and \Cref{eq:power-law}.
Without loss of generality, assume that $\SigmaB$ is diagonal. 
By the definitions of $k^*$ in \Cref{thm:sgd}, we have 
\begin{align*}
    \lambda_{k^*+1} \le \frac{1}{c \eta N} < \lambda_{k^*},\quad k^* = \#\bigg\{i: \lambda_i > \frac{1}{c \eta N}\bigg\}.
\end{align*}

\emph{Upper bound.}~
Let us first consider $a \le 1+2ar$ and prove an upper bound for SGD. 
By \Cref{lemma:power-law:effective-regu}, the effective dimension $D$ in \Cref{thm:sgd} is 
\begin{align*}
    D &:= k^* + (\eta N)^2\sum_{i>k^*}\lambda_i^2 \\
    &\lesssim  (\eta N)^{\frac{1}{a}} + (\eta N)^2 (\eta N)^{\frac{1}{a} - 2} \\
    & \eqsim (\eta N)^{\frac{1}{a}}.
\end{align*}
Notice that under \Cref{eq:power-law}, we have $\|\wB^*\|_{\SigmaB}^2 \lesssim 1$ and $\sigma^2 \lesssim 1$, then the variance error is 
\begin{align*}
    \variance \lesssim \frac{D}{n} \lesssim \frac{(\eta N)^{\frac{1}{a}}}{n}.
\end{align*}
For the bias error, we have 
\begin{align*}
   \bias &= \bigg\|\prod_{t=1}^{n}\big(\IB-\eta_t\SigmaB\big)\wB^*\bigg\|^2_{\SigmaB} \\ 
    &\le \bigg\|\big(\IB-\eta\SigmaB\big)^N\wB^*\bigg\|^2_{\SigmaB} && \explain{by the stepsize scheduler in \Cref{eq:sgd}} \\
    &\le \sum_{i} \exp( - 2\eta N \lambda_i)\lambda_i \wB^{*2}_i \\
    &\lesssim \sum_{i} (\eta N \lambda_i)^{-2r} \lambda_i \wB^{*2}_i && \explain{by $e^{-t} \le (c/t)^c$ for every $c, t>0$} \\
    &\eqsim (\eta N)^{-2r} \sum_{i} \lambda_i^{1-2r} \wB^{*2}_i  \\
    &\lesssim (\eta N)^{-2r} && \explain{by \Cref{eq:power-law}} .
\end{align*}
Finally, we choose 
\[
\eta =  N^{- \frac{1+2ar - a}{1+2ar}}  / (4\tr(\SigmaB)) \eqsim  N^{- \frac{1+2ar - a}{1+2ar}} .
\]
Since $1<a\le 1+2ar$, we have $\eta \le 1/(4\tr(\SigmaB))$, which enables \Cref{thm:sgd}.
Note that
\[
\eta N =  N^{\frac{a}{1+2ar}},
\]
plugging which into the bounds on the bias and variance errors gives the promised rate.

\emph{Lower bound.}~
We next consider $a>1+2ar$ and prove a lower bound for SGD with any stepsize $0\le \eta \le 1/(4\tr(\SigmaB))$. 
For each $n$ (hence for each $N$), consider the following sequence of hard problems from $\Pbb_{a,r}$:
\[
\sigma^2\eqsim 1,\quad 
\lambda_i\eqsim i^{-a},\quad 
\wB^*_i \eqsim  \begin{dcases}
   N^{\frac{1}{2} - r} & i = N^{\frac{1}{a}}; \\
   0 & i\ne N^{\frac{1}{a}}.
\end{dcases}
\]
This instance belongs to $\Pbb_{a,r}$ since 
\[
\|\SigmaB^{-r}\wB^*\|_{\SigmaB}^2 = \sum_{i}\lambda_i^{1-2r}\wB^{*2}_i \eqsim \big( (N^{1/a})^{-a}  \big)^{1-2r} \big( N^{1/2-r}\big)^2 \eqsim 1.
\]
Then \Cref{thm:sgd} implies a risk lower bound by its bias error,
\begin{align*}
    \bias &= \bigg\|\prod_{t=1}^{n}\big(\IB-\eta_t\SigmaB\big)\wB^*\bigg\|^2_{\SigmaB} \\ 
    &\ge \bigg\|\big(\IB-\eta\SigmaB\big)^{2N}\wB^*\bigg\|^2_{\SigmaB} && \explain{by the stepsize scheduler in \Cref{eq:sgd}} \\ 
    &= \sum_{i} (1-\eta \lambda_i)^{4N} \lambda_i\wB_i^{* 2} \\
    &= (1-\eta\lambda_{N^{1/a}} )^{4N} \lambda_{N^{1/a}}\wB_{N^{1/a}}^{* 2} \\
    &\gtrsim \big(1- \Theta(\eta / N) \big)^{4N} \frac{1}{N} \big(N^{\frac{1}{2} - r}\big)^2  \\
    &\eqsim N^{-2r}. && \explain{since $\eta \lesssim 1$}
\end{align*}
This completes our proof.
\end{proof}

\subsection{\texorpdfstring{Proof of \Cref{thm:power-law:gd}}{GD for power law class}}\label[appendix]{append:sec:power-law:gd}

\begin{proof}[Proof of \Cref{thm:power-law:gd}]
Under \Cref{eq:power-law}, by standard concentration we have
$\|\XB\XB^\top\|/ n \le 2\tr(\SigmaB)$ holds  
with probability at least $1-\exp(-n/c_0)$.
In what follows, we work under the joint of this event and the events of \Cref{thm:gd:ridge,thm:gd:exp}, which holds with probability at least $1-3\exp(-k^*/c_0)$.

For $0\le r \le 1$, we apply \Cref{thm:gd:ridge} and use \Cref{thm:power-law:ridge} to get our upper bound for GD.

For $r>1$, we must have $a< 1+2ar$, and thus our choice of stepsize 
\[
\eta \ge n^{-\frac{1+2ar - a}{1+2ar}}
\]
and the condition that 
\[
\eta t = n^{\frac{a}{1+2ar}}
\]
guarantees that
\(
t \le n
\), which enables \Cref{thm:gd:exp}.
By the formula of $\eta t$ and \Cref{eq:power-law}, we have 
\[
\lambda_{k^*+1} \le \frac{1}{c_2 \eta t} < \lambda_{k^*},\quad 
k^* = \bigg\{i: \lambda_i > \frac{1}{c_2 \eta t} \bigg\}.
\]
By \Cref{lemma:power-law:effective-regu}, $D$ and $D_1$ in \Cref{thm:gd:exp} are respectively bounded by
\begin{align*}
 D &:= k^* + (\eta t )^2 \sum_{i>k^*}\lambda_i^2 
    \lesssim  (\eta t)^{\frac{1}{a}} + (\eta t)^2 (\eta t)^{\frac{1}{a}-2}  
    \eqsim  (\eta t)^{\frac{1}{a}}, \\ 
 D_1 &:= k^* + \eta t \sum_{i>k^*}\lambda_i 
    \lesssim   (\eta t)^{\frac{1}{a}}+ \eta t (\eta t)^{\frac{1}{a}-1}
    \eqsim (\eta t)^{\frac{1}{a}}.
\end{align*}
Notice that $\|\wB^*\|^2_{\SigmaB} \lesssim 1$ and $\sigma^2\lesssim 1$ by \Cref{eq:power-law}, then the sum of the variance and effective variance errors in \Cref{thm:gd:exp} is 
\begin{align*}
    \variance + \effVar \lesssim \frac{D}{n} + \bigg(\frac{D_1}{n}\bigg)^2 \lesssim \frac{(\eta t)^{\frac{1}{a}}}{n}.
\end{align*}
For the effective bias error in \Cref{thm:gd:exp}, we have
\begin{align*}
   \effBias &:=  \frac{1}{\eta^2 t^2}\big\|(\IB-\eta\SigmaB_{0:k^*})^{t/2}\wB^*\big\|^2_{\SigmaB_{0:k^*}^{-1}} + \|\wB^*_{k^*:\infty}\|^2_{\SigmaB_{k^*:\infty}} \\
   &\le (\eta t)^{-2}\sum_{i\le k^*} \exp( - \eta t \lambda_i)\lambda_i^{-1} \wB^{*2}_i+\sum_{i> k^*} \lambda_i \wB^{*2}_i \\
    &\lesssim (\eta t)^{-2}\sum_{i\le k^*} (  \eta t \lambda_i)^{-2(r-1)}\lambda_i^{-1} \wB^{*2}_i+\sum_{i> k^*} \lambda_i \wB^{*2}_i && \explain{by $r> 1$ and $e^{-t} \le (c/t)^c$ for every $c, t>0$} \\
    &\lesssim (\eta t)^{-2r} \sum_{i\le k^*} \lambda_i^{1-2r} \wB^{*2}_i + \lambda_{k^*+1}^{2r} \sum_{i>k^*} \lambda_i^{1-2r} \wB^{*2}_i \\
    &\lesssim (\eta t)^{-2r} + \lambda_{k^*+1}^{2r} && \explain{by \Cref{eq:power-law}} \\ 
    &\eqsim (\eta t)^{-2r}.
\end{align*}
So the total error is 
\begin{align*}
    \bigO\bigg(\frac{(\eta t)^{\frac{1}{a}}}{n} +   (\eta t)^{-2r} \bigg),
\end{align*}
bringing $\eta t = n^{\frac{a}{1+2ar}}$ into which gives the promised rate.
\end{proof}

\end{document}